\documentclass[twoside,11pt]{article}

\usepackage{jrnl}

\usepackage{geometry}

\usepackage[utf8]{inputenc}
\usepackage[T1]{fontenc}
\usepackage{url}
\usepackage{nicefrac}
\usepackage{microtype}

\usepackage{graphicx}
\usepackage[font=small,labelfont=bf]{caption}
\usepackage{subcaption}
\usepackage{adjustbox}
\usepackage{wrapfig}
\graphicspath{{figures/pdf/}}
\DeclareGraphicsExtensions{.pdf}

\usepackage{tikz}
\usetikzlibrary{bayesnet}
\usetikzlibrary{calc}
\usetikzlibrary{shapes}

\usepackage{comment}

\usepackage{amsfonts}
\usepackage{amsmath}
\usepackage{amssymb}
\usepackage{mathtools}

\usepackage{algorithm}
\usepackage{algorithmic}
\usepackage{rotating}

\usepackage{booktabs}
\usepackage{makecell}
\usepackage{multirow}
\usepackage{tabularx}

\usepackage{enumitem}
\usepackage{xspace}

\usepackage{macro}

\jmlrheading{}{}{}{}{}{}{}
\ShortHeadings{Contextual Explanation Networks}{Al-Shedivat, Dubey, Xing}
\firstpageno{1}

\begin{document}

\title{Contextual Explanation Networks}

\author{%
  \name{Maruan~Al-Shedivat} \email{alshedivat@cs.cmu.edu} \\
  \addr{Carnegie Mellon University}
  \AND
  \name{Avinava~Dubey} \email{avinava.dubey@google.com} \\
  \addr{Google Research}
  \AND
  \name{Eric~P.~Xing} \email{epxing@cs.cmu.edu} \\
  \addr{Carnegie Mellon University \& Petuum Inc.}%
}

\editor{Edoardo M. Airoldi}

\maketitle

\begin{abstract}%
Modern learning algorithms excel at producing accurate but complex models of the data.
However, deploying such models in the real-world requires extra care: we must ensure their reliability, robustness, and absence of undesired biases.
This motivates the development of models that are equally accurate but can be also easily inspected and assessed beyond their predictive performance.
To this end, we introduce \emph{contextual explanation networks} ({\CENs})---a class of architectures that learn to predict by generating and utilizing intermediate, simplified probabilistic models.
Specifically, {\CENs} generate parameters for intermediate graphical models which are further used for prediction and play the role of explanations.
Contrary to the existing \emph{post-hoc} model-explanation tools, {\CENs} learn to predict and to explain simultaneously.
Our approach offers two major advantages: (i) for each prediction, valid, instance-specific explanation is generated with no computational overhead and (ii) prediction via explanation acts as a regularizer and boosts performance in data-scarce settings.
We analyze the proposed framework theoretically and experimentally.
Our results on image and text classification and survival analysis tasks demonstrate that {\CENs} are not only competitive with the state-of-the-art methods but also offer additional insights behind each prediction, that can be valuable for decision support.
We also show that while post-hoc methods may produce misleading explanations in certain cases, {\CENs} are consistent and allow to detect such cases systematically.
\end{abstract}

\section{Introduction}\label{sec:introduction}

Model interpretability is a long-standing problem in machine learning that has become quite acute with the accelerating pace of the widespread adoption of complex predictive algorithms.
While high performance often supports our belief in the predictive capabilities of a system, perturbation analysis reveals that black-box models can be easily broken in an unintuitive and unexpected manner \citep{szegedy2013intriguing,nguyen2015deep}.
Therefore, for a machine learning system to be used in a social context (\eg, in healthcare) it is imperative to provide sound reasoning for each prediction or decision it makes.

To design such systems, we may restrict the class of models to only \emph{human-intelligible} \citep{caruana2015intelligible}.
However, such an approach is often limiting in modern practical settings.
Alternatively, we may fit a complex model and explain its predictions \emph{post-hoc}, \eg, by searching for linear local approximations of the decision boundary~\citep{ribeiro2016trust}.
While such methods achieve their goal, explanations are generated \emph{a posteriori} require additional computation per data instance and, most importantly, are never the basis for the predictions made in the first place, which may lead to erroneous interpretations, as we show in this paper, or even be exploited~\citep{dombrowski2019explanations, lakkaraju2019fool}.

\begin{figure}[t]
    \centering
    \includegraphics[width=0.60\textwidth]{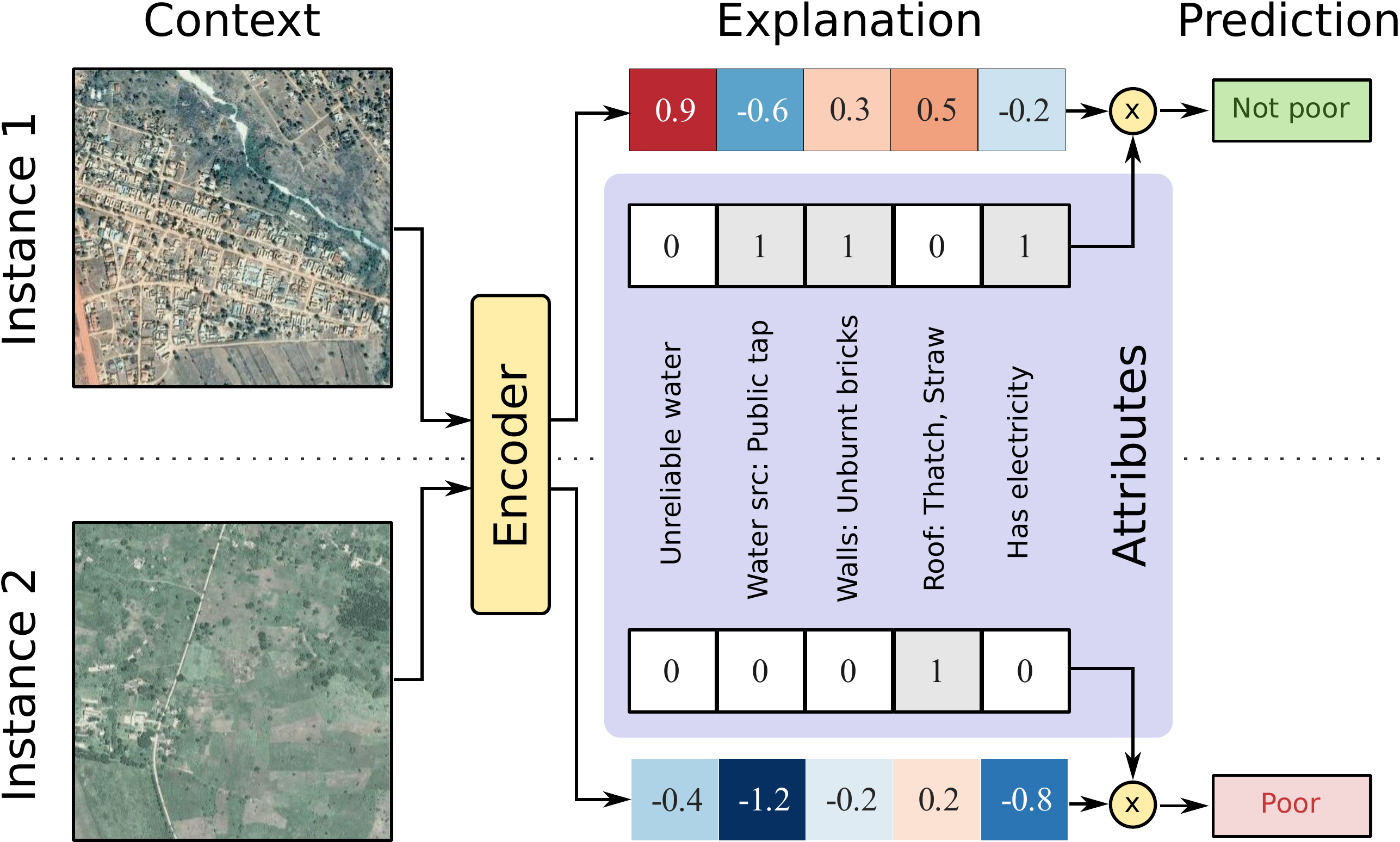}%
    \caption{%
        High-level functionality of {\CENs}:
        The context is represented by satellite imagery and used to generate instance-specific linear models (explanations).
        The latter act on a set of interpretable attributes from regional survey data and produce predictions.%
    }
    \label{fig:illustration}
\end{figure}

Explanation is a fundamental part of the human learning and decision process~\citep{lombrozo2006structure}.
Inspired by this fact, we introduce \emph{contextual explanation networks} ({\CENs})---a class of architectures that learn to predict and to explain jointly, alleviating the drawbacks of the post-hoc methods.
To make a prediction, {\CENs} operate as follows (Figure~\ref{fig:illustration}).
First, they process a subset of inputs and generate parameters for a simple probabilistic model (\eg, sparse linear model) which is regarded interpretable by a domain expert.
Then, the generated model is applied to another subset of inputs and produces a prediction.
To motivate such an architecture, we consider the following example.

\paragraph{A motivating illustration.}
One of the tasks we consider in this paper is classification of households into poor and not poor having access to satellite imagery and categorical data from surveys~\citep{jean2016combining}.
If a human were to solve this task, to make predictions, they might assign weights to features in the categorical data and explain their predictions in terms of the most relevant variables (\ie, come up with a linear model).
Moreover, depending on the type of the area (as seen from the imagery), they might select slightly different weights for different areas (\emph{e.g.}, when features indicative of poverty are different for urban, rural, and other types of areas).

The {\CEN} architecture given in Figure~\ref{fig:illustration} imitates this process by making predictions using sparse linear models applied to interpretable categorical features.
The weights of the linear models are contextual, generated by a learned encoder that maps images (the context) to the weight vectors.
The learned encoder is sensitive to the infrastructure presented in the input images and generates different linear models for urban and rural areas.
The generated models not only are used for prediction but also play the role of explanations and can encode arbitrary prior knowledge.
{\CENs} can represent complex model classes by using powerful encoders.
At the same time, by offsetting complexity into the encoding process, we achieve simplicity of explanations and can interpret predictions in terms the variables of interest.

The proposed architecture opens a number of questions:
What are the fundamental advantages and limitations of {\CEN}?
How much of the performance should be attributed to the context encoder and how much to the explanations?
Are there any degenerate cases and do they happen in practice?
Finally, how do \CEN-generated explanations compare to alternatives, \eg, produced with \LIME~\citep{ribeiro2016trust}?
In the rest of this paper, we formalize our intuitions and answer these questions theoretically and experimentally.

\subsection{Contributions}\label{sec:contributions}

The main four contributions of this paper are as follows:
\begin{itemize}[noitemsep,topsep=6pt,parsep=2pt,leftmargin=2em]
    \item[(i)] We formally define {\CENs} as a class of probabilistic models, consider special cases, and derive learning and inference algorithms for scalar and structured outputs.

    \item[(ii)] We design {\CENs} in the form of new deep learning architectures trainable end-to-end for prediction and survival analysis tasks.

    \item[(iii)] Empirically, we demonstrate the value of learning with explanations for both prediction and model diagnostics.
    Moreover, we find that explanations can act as a regularizer and result in improved sample efficiency.

    \item[(iv)] We also show that noisy features can render post-hoc explanations inconsistent and misleading, and how {\CENs} can help to detect and avoid such situations.

\end{itemize}

\noindent Our code is available at \url{https://github.com/alshedivat/cen}.

\subsection{Organization}
\label{sec:organization}

The paper is organized as follows.
Section~\ref{sec:background} presents the notation and some background on post-hoc interpretability methods.
In Sections \ref{sec:CEN}, we introduce the general {\CEN} framework, describe specific implementations, learning, and inference.
In Section~\ref{sec:related_work}, we overview broadly related work.
In Section~\ref{sec:analysis}, we discuss and analyze properties of {\CEN} theoretically.
Section~\ref{sec:case-studies} presents a number of case studies:
experimental results for scalar prediction tasks (Section~\ref{sec:applications-classification}),
an empirical analysis of consistency of linear explanations generated by {\CEN} vs. alternatives (Section~\ref{sec:applications-properties}),
and finally how {\CENs} with structured explanations can efficiently solve survival analysis tasks (Section~\ref{sec:applications-survival-analysis}).

\section{Background}
\label{sec:background}

We start by introducing the notation and reviewing post-hoc model explanations, with a focus on {\LIME}~\citep{ribeiro2016trust} as one of the most popular frameworks to date.

Given a collection of data where each instance is represented by inputs, $\cv \in \Cc$, and targets, $\yv \in \Yc$, our goal is to learn an accurate predictive model, $f: \Cc \mapsto \Yc$.
To explain predictions, we can assume that each data point has another set of features, $\xv \in \Xc$.
We construct explanations in the form of simpler models, $g_{\cv}: \Xc \mapsto \Yc$, so that they are consistent with the original model in the neighborhood of the corresponding data instance, \ie, $g_\cv(\xv) = f(\cv)$.
While the original inputs, $\cv$, can be of complex, low-level, unstructured data types (\eg, text, image pixels, sensory inputs), we assume that $\xv$ are high-level, meaningful variables (\eg, categorical features).
In the post-hoc explanation literature, it is assumed that $\xv$ are \emph{derived} from $\cv$ and are often binary~\citep{lundberg2017shap} (\eg, $\cv$ can be images, while $\xv$ can be vectors of binary indicators over the corresponding super-pixels).
We consider a more general setting where $\cv$ and $\xv$ are not necessarily derived from each other.
Throughout the paper, we call $\cv$ the \emph{context} and $\xv$ the \emph{attributes} or \emph{variables of interest}.

\subsection*{Locally Interpretable Model-agnostic Explanations (LIME)}
\label{sec:lime}

Given a trained model, $f$, and a data instance with features $(\cv, \xv)$, LIME constructs an explanation, $g_{\cv}$, as follows:
\begin{equation}
    \label{eq:LIME-general}
    g_{\cv} = \argmin_{g \in \Gc} \Lc(f, g, \pi_{\cv}) + \Omega(g)
\end{equation}
where $\Lc(f, g, \pi_{\cv})$ is the loss that measures how well $g$ approximates $f$ in the neighborhood defined by the similarity kernel, $\pi_{\cv} : \Xc \mapsto \Rb_+$, in the space of attributes, $\Xc$, and $\Omega(g)$ is the penalty on the complexity of explanation.\footnote{\citet{ribeiro2016trust} argue that only simple models of low complexity (\eg, sufficiently sparse linear models) are human-interpretable and support that by human studies.}
Now more specifically, \citet{ribeiro2016trust} assume that $\Gc$ is the class of linear models, $g_{\cv}(\xv) := b_{\cv} + \wv_{\cv} \cdot \xv$, and define the loss and the similarity kernel as follows:
\begin{equation}
    \label{eq:LIME-specific}
    \Lc(f, g, \pi_{\cv}) := \sum_{\xv^\prime \in \Xc} \pi_{\cv}(\xv^\prime)\left(f(\cv^\prime) - g(\xv^\prime)\right)^2, \quad
    \pi_{\cv}(\xv^\prime) := \exp\left\{-D(\xv, \xv^\prime)^2 / \sigma^2\right\}
\end{equation}
where the data instance of interest is represented by $(\cv, \xv)$, $\xv^\prime$ and the corresponding $\cv^\prime$ are the perturbed features, $D(\xv, \xv^\prime)$ is some distance function, and $\sigma$ is the scale parameter of the kernel.
The regularizer, $\Omega(g)$, is often chosen to favor sparse explanations.

The model-agnostic property is the key advantage of LIME (and variations)---we can solve \eqref{eq:LIME-general} for any trained model, $f$, any class of explanations, $\Gc$, at any point of interest, $(\cv, \xv)$.
While elegant, predictive and explanatory models in this framework are learned independently and hence never affect each other.
In the next section, we propose a class of models that ties prediction and explanation together in a joint probabilistic framework.

\section{Contextual Explanation Networks}
\label{sec:CEN}

\begin{figure}[t]
\centering
\begin{subfigure}[b]{0.16\textwidth}
    \begin{tikzpicture}[
        inner/.style={draw, fill=yellow!20, thin, inner sep=2pt},
    ]
        \tikzstyle{latent} = [circle,fill=white,draw=black,inner sep=1pt,
        minimum size=17pt, font=\fontsize{9}{9}\selectfont, node distance=1]

        \node (C) [obs] {$\Cv$};
        \node (theta) [latent, above=20pt of C] {$\thetav$};
        \node (X_) [obs, transparent, right=15pt of C, yshift=6pt] {};
        \node (Y_) [obs, transparent, above=15pt of X_, xshift=2pt] {};
        \node (theta_)  [latent, transparent, above=25pt of C] {};

        \node (w) [const, above=5pt of C, xshift=-7pt] {$\wv$};

        \plate[inner sep=8pt, thick] {CEN} {(C)(theta_)(Y_)} {N};
        \plate[inner] {CRF} {(X_)(Y_)} { };

        \node (X1) [obs, right=16pt of C]  {$\Xv$};
        \node (Y1) [latent, above=20pt of X1] {$\Yv$};

        \edge {C} {theta};

        \factor[above=5pt of X1]{X1-Y1} {} {} {};
        \factoredge {X1} {X1-Y1} {Y1};
        \edge[bend right=25, dashed] {theta} {X1-Y1};
    \end{tikzpicture}
    \caption{}\label{fig:CEN-scalar}
\end{subfigure}
\hfil
\begin{subfigure}[b]{0.29\textwidth}
    \begin{tikzpicture}[
        inner/.style={draw, fill=yellow!20, thin, inner sep=2pt},
    ]
        \tikzstyle{latent} = [circle,fill=white,draw=black,inner sep=1pt,
        minimum size=17pt, font=\fontsize{9}{9}\selectfont, node distance=1]

        \node (C) [obs] {$\Cv$};
        \node (theta) [latent, above=20pt of C] {$\thetav$};
        \node (X_) [obs, transparent, right=15pt of C, yshift=6pt] {};
        \node (Y_) [obs, transparent, above=15pt of X_, xshift=57pt] {};
        \node (theta_)  [latent, transparent, above=25pt of C] {};

        \node (w) [const, above=5pt of C, xshift=-7pt] {$\wv$};

        \plate[inner sep=8pt, thick] {CEN} {(C)(theta_)(Y_)} {N};
        \plate[inner] {CRF} {(X_)(Y_)} { };

        \node (X1) [obs, right=16pt of C]  {$\Xv_1$};
        \node (Y1) [latent, above=20pt of X1] {$\Yv_1$};
        \node (X2) [obs, right=10pt of X1] {$\Xv_2$};
        \node (Y2) [latent, right=10pt of Y1] {$\Yv_2$};
        \node (X3) [obs, right=10pt of X2] {$\Xv_3$};
        \node (Y3) [latent, right=10pt of Y2] {$\Yv_3$};

        \edge {C} {theta};

        \factor[above=5pt of X1]{X1-Y1} {} {} {};
        \factoredge {X1} {X1-Y1} {Y1};
        \edge[bend right=25, dashed] {theta} {X1-Y1};

        \factor[above=5pt of X2]{X2-Y2} {} {} {};
        \factoredge {X2} {X2-Y2} {Y2};
        \edge[bend right=18, dashed] {theta} {X2-Y2};

        \factor[above=5pt of X3]{X3-Y3} {} {} {};
        \factoredge {X3} {X3-Y3} {Y3};
        \edge[bend right=12, dashed] {theta} {X3-Y3};

        \factor[right=2pt of Y1]{Y1-Y2} {} {} {};
        \factor[right=2pt of Y2]{Y2-Y3} {} {} {};
        \edge[-] {Y1} {Y2};
        \edge[-] {Y2} {Y3};
        \edge[bend left=50, dashed] {theta} {Y1-Y2};
        \edge[bend left=50, dashed] {theta} {Y2-Y3};
    \end{tikzpicture}
    \caption{}\label{fig:CEN-structured}
\end{subfigure}
\hfil
\begin{subfigure}[b]{0.29\textwidth}
    \begin{tikzpicture}[
        inner/.style={draw, fill=yellow!20, thin, inner sep=2pt},
    ]
        \tikzstyle{latent} = [circle,fill=white,draw=black,inner sep=1pt,
        minimum size=17pt, font=\fontsize{9}{9}\selectfont, node distance=1]

        \node (C) [obs] {$\Cv$};
        \node (theta) [latent, above=20pt of C] {$\thetav$};
        \node (C_) [obs, transparent, xshift=-2pt] {};
        \node (X_) [obs, transparent, right=15pt of C, yshift=6pt] {};
        \node (Y_) [obs, transparent, above=15pt of X_, xshift=57pt] {};
        \node (theta_)  [latent, transparent, above=25pt of C] {};

        \node (p) [const, above=8pt of C, xshift=5pt] {$p$};
        \node (q) [const, above=8pt of C, xshift=-13pt] {$q$};

        \plate[inner sep=8pt, thick] {CEN} {(C_)(theta_)(Y_)} {N};
        \plate[inner] {CRF} {(X_)(Y_)} { };

        \node (X1) [obs, right=16pt of C]  {$\Xv_1$};
        \node (Y1) [latent, above=20pt of X1] {$\Yv_1$};
        \node (X2) [obs, right=10pt of X1] {$\Xv_2$};
        \node (Y2) [latent, right=10pt of Y1] {$\Yv_2$};
        \node (X3) [obs, right=10pt of X2] {$\Xv_3$};
        \node (Y3) [latent, right=10pt of Y2] {$\Yv_3$};

        \edge {theta} {C};
        \edge[bend left=30, dashed] {C} {theta};

        \factor[above=5pt of X1]{X1-Y1} {} {} {};
        \factoredge {X1} {X1-Y1} {Y1};
        \edge[bend right=25, dashed] {theta} {X1-Y1};

        \factor[above=5pt of X2]{X2-Y2} {} {} {};
        \factoredge {X2} {X2-Y2} {Y2};
        \edge[bend right=18, dashed] {theta} {X2-Y2};

        \factor[above=5pt of X3]{X3-Y3} {} {} {};
        \factoredge {X3} {X3-Y3} {Y3};
        \edge[bend right=12, dashed] {theta} {X3-Y3};

        \factor[right=2pt of Y1]{Y1-Y2} {} {} {};
        \factor[right=2pt of Y2]{Y2-Y3} {} {} {};
        \edge[-] {Y1} {Y2};
        \edge[-] {Y2} {Y3};
        \edge[bend left=50, dashed] {theta} {Y1-Y2};
        \edge[bend left=50, dashed] {theta} {Y2-Y3};
    \end{tikzpicture}
    \caption{}\label{fig:VCEN}
\end{subfigure}%
\caption{%
(a) A graphical model for {\CEN} with a context encoder parameterized by $\wv$ and linear explanations.
(b) A graphical model for {\CEN} with context encoder and CRF-based explanations.
The model is parameterized by $\wv$.
(c) A graphical model for {\CEN} with context autoencoding via the inference, $q$, and generator, $p$, networks and CRF-based explanations.%
}
\label{fig:PGM}
\end{figure}

We consider the same problem of learning from a collection of data represented by context variables, $\cv \in \Cc$, attributes, $\xv \in \Xc$, and targets, $\yv \in \Yc$.
We denote the corresponding random variables by capital letters, $\Cv$, $\Xv$, and $\Yv$, respectively.
Our goal is to learn a model, $\prob[\wv]{\Yv \mid \xv, \cv}$, parametrized by $\wv$ that can predict $\yv$ from $\xv$ and $\cv$.
We define contextual explanation networks as probabilistic models that assume the following form (Figure~\ref{fig:PGM}):\footnote{%
While we focus on predictive modeling, {\CENs} are applicable beyond that.
For example, instead of learning a predictive distribution, $\prob[\wv]{\Yv \mid \xv, \cv}$, we may want to learn a contextual marginal distribution, $\prob[\wv]{\Xv \mid \cv}$, over a set random variables $\Xv$, where $\prob{\Xv \mid \thetav}$ is defined by an arbitrary graphical model.}
\begin{equation}
    \label{eq:CEN-general}
    \yv \sim \prob{\Yv \mid \xv, \thetav}, \quad
    \thetav \sim \prob[\wv]{\thetav \mid \cv}, \quad
    \prob[\wv]{\Yv \mid \xv, \cv} = \int \prob{\Yv \mid \xv, \thetav} \prob[\wv]{\thetav \mid \cv} d\thetav
\end{equation}
where $\prob{\Yv \mid \xv, \thetav}$ is a predictor parametrized by $\thetav$.
We call such predictors \emph{explanations}, since they explicitly relate interpretable attributes, $\xv$, to the targets, $\yv$.
For example, when the targets are scalar and binary, explanations may take the form of linear logistic models;
when the targets are more complex, dependencies between the components of $\yv$ can be represented by a graphical model, \eg, \emph{conditional random field}~\citep{lafferty2001crf}.

{\CENs} assume that each explanation is context-specific: $\prob[\wv]{\thetav \mid \cv}$ defines a conditional probability of an explanation $\thetav$ being valid in the context $\cv$.
To make a prediction, we marginalize out $\thetav$.
To interpret a prediction, $\hat \yv$, for a given data instance, $(\xv, \cv)$, we infer the posterior, $\prob[\wv]{\thetav \mid \hat \yv, \xv, \cv}$.
The main advantage of this approach is to allow modeling conditional probabilities, $\prob[\wv]{\thetav \mid \cv}$, in a black-box fashion while keeping the class of explanations, $\prob{\Yv \mid \xv, \thetav}$, simple and interpretable.
For instance, when the context is given as raw text, we may choose $\prob[\wv]{\thetav \mid \cv}$ to be represented with a recurrent neural network, while $\prob{\Yv \mid \xv, \thetav}$ be in the class of linear models.

Implications of these assumptions are discussed in Section~\ref{sec:analysis}.
Here, we continue with a discussion of a number of practical choices for $\prob[\wv]{\thetav \mid \cv}$ and $\prob{\Yv \mid \xv, \thetav}$ (Table~\ref{tab:cen-components}).

\begin{table}[t!]
\caption{\small%
Different types of encoders and explanations used in {\CEN}.}
\label{tab:cen-components}
\smallskip
\centering
\scriptsize
\def\arraystretch{1.2}
\begin{tabular}[t]{@{}l|r@{}}
    \toprule
    \textbf{Encoder}    & \textbf{Parameter distribution, $\prob{\thetav \mid \cv}$}  \\
    \midrule
    Deterministic       & $\delta\left(\phi(\cv), \thetav\right)$ where $\phi(\cv)$ is arbitrary \\
    Constrained         & $\delta\left(\phi(\cv), \thetav\right)$ where $\phi(\cv) := \alphav(\cv)^\top \Dv$ \\
    MoE                 & $\sum_{k=1}^K \prob{k \mid \cv} \delta(\thetav, \thetav_k)$ \\
    \bottomrule
\end{tabular}
~
\begin{tabular}[t]{@{}l|>{\raggedleft\arraybackslash}p{5.1cm}@{}}
    \toprule
    \textbf{Explanation} & \textbf{Predictive distribution, $\prob{\yv \mid \xv, \thetav}$}  \\
    \midrule
    Linear                & $\softmax\bb{\thetav^\top \xv}$ \\
    Structured            & $\propto \exp\cbb{-E_{\thetav}(\xv, \yv)}$ where $E_{\thetav}(\cdot, \cdot)$ is some energy function, linear in $\thetav$ \\[0.575ex]
    \bottomrule
\end{tabular}
\end{table}

\subsection{Context Encoders}\label{sec:encoders}

In practice, we represent $\prob[\wv]{\thetav \mid \cv}$ with a neural network that encodes the context into the parameter space of the explanation models.
There are two simple ways to construct an encoder, which we consider below.

\subsubsection{Deterministic encoding}
\label{sec:det-enc}

Let $\prob[\wv]{\thetav \mid \cv} := \delta\left(\phi_{\wv}(\cv), \thetav\right)$, where $\delta(\cdot, \cdot)$ is a delta-function and $\phi_{\wv}(\cdot)$ is the network that maps $\cv$ to $\thetav$.
Collapsing the conditional distribution to a delta-function makes $\thetav$ depend deterministically on $\cv$ and results into the following conditional likelihood:
\begin{equation}
    \label{eq:CEN-det}
    \prob{\yv \mid \xv, \cv; \wv}
    = \int \prob{\yv \mid \xv, \thetav} \delta\left(\phi_{\wv}(\cv), \thetav\right)d\thetav = \prob{\yv \mid \xv, \thetav = \phi_{\wv}(\cv)}
\end{equation}
Modeling $\prob[\wv]{\thetav \mid \cv}$ with a delta-function is convenient since the posterior, $\prob[\wv]{\thetav \mid \yv, \xv, \cv} \propto \prob{\yv \mid \xv, \thetav} \delta\left(\phi_{\wv}(\cv), \thetav\right)$ also collapses to $\thetav^\star = \phi_{\wv}(\cv)$, hence the inference is done via a single forward pass and the posterior can be regularized by imposing $L_1$ or $L_2$ losses on $\phi_{\wv}(\cv)$.

\subsubsection{Constrained deterministic encoding}
\label{sec:constrained-det-enc}

The downside of deterministic encoding is the lack of constraints on the generated explanations.
There are multiple reasons why this might be an issue:
(i) when the context encoder is unrestricted, it might generate unstable, overfitted local models,
(ii) when we want to reason about the patterns in the data as a whole, local explanations are not enough.
To address these issues, we constrain the space of explanations by introducing a context-independent, \emph{global dictionary}, $\Dv := \{\thetav_k\}_{k=1}^K$, where each atom, $\thetav_k$, is sparse.
The encoder generates context-specific explanations using \emph{soft attention} over the dictionary
(Figure~\ref{fig:explanation-net}):
\begin{equation}
    \label{eq:CEN-dict-enc}
    \begin{aligned}
        \phi_{\wv, \Dv}(\cv) := \sum_{k=1}^K \prob[\wv]{k \mid \cv} \thetav_k = \alphav_\wv(\cv)^\top \Dv, \quad
        \sum_{k=1}^K \alphav_\wv^{(k)}(\cv) = 1, \quad
        \forall k: \alphav_\wv^{(k)}(\cv) \geq 0,
    \end{aligned}
\end{equation}
where $\alphav_\wv(\cv)$ is the attention over the dictionary produced by the encoder.
Attention-based construction of explanations using a global dictionary (i) forces the encoder to produce models shared across different contexts, (ii) allows us to interpret the learned dictionary atoms as global ``explanation modes.''
Again, since $\prob[\wv]{\thetav \mid \cv}$ is a delta-distribution, the likelihood is the same as given in~\eqref{eq:CEN-det} and inference is conveniently done via a forward pass.

\noindent
The two proposed context encoders represent $\prob{\thetav \mid \cv}$ with delta-functions, which simplifies learning, inference, and interpretation of the model, and are used in our experiments.
Other ways to represent $\prob{\thetav \mid \cv}$ include: (i) using a mixture of delta-functions (which makes {\CEN} function similar to a mixture-of-experts model and further discussed in Section~\ref{sec:CEN-special-calses}), or (ii) using variational autoencoding.
We leave more complex approaches to future research.

\begin{figure}[t!]
\centering
\includegraphics[width=\textwidth]{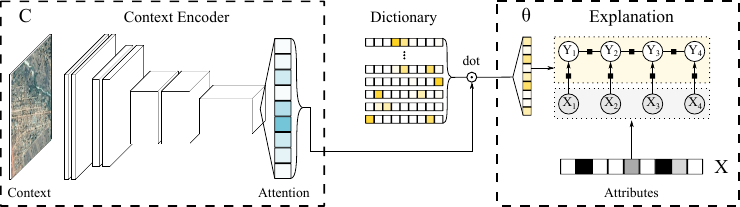}
\caption{%
An example of a CEN architecture.
In this example, the context is represented by an image and transformed by a convnet encoder into an attention vector, which is used to softly select parameters for a contextual linear probabilistic model.}
\label{fig:explanation-net}
\end{figure}

\subsection{Explanations}
\label{sec:explanations}

In this paper, we consider two types of explanations: \emph{linear} that can be used for regression or classification and \emph{structured} that are suitable for structured prediction.

\subsubsection{Linear Explanations}
\label{sec:linear-explanations}

In case of classification, {\CENs} with linear explanations assume the following $\prob{\Yv \mid \xv, \thetav}$:
\begin{equation}
    \label{eq:linear-explanation}
    \prob{\Yv = i \mid \xv, \thetav} := \frac{\exp\cbb{(\Wv \xv + \bv)_i}}{\sum_{j \in \Yc}\exp\cbb{(\Wv \xv + \bv)_j}},
\end{equation}
where $\thetav := (\Wv, \bv)$ and $i, j$ index classes in $\Yc$.
If $\xv$ is $d$-dimensional and we are given $m$-class classification problem, then $\Wv \in \Rb^{m \times d}$ and $\bv \in \Rb^{m}$.
The case of regression is similar.

In Section~\ref{sec:CEN-vs-LIME}, we show that if we apply {\LIME} to interpret {\CEN} with linear explanations, the local linear models inferred by {\LIME} are guaranteed to recover the original {\CEN}-generated explanations.
In other words, linear explanations generated by {\CEN} have similar properties, \eg, local faithfulness  \citep{ribeiro2016trust}.
However, we emphasize the key difference between {\LIME} and {\CEN}: the former regards explanation as a post-processing step (done after training) while the latter integrates explanation into the learning process.

\subsubsection{Structured Explanations}
\label{sec:structured-explanations}

While post-hoc methods, such as {\LIME}, can easily generate local linear explanations for scalar outputs, using such methods for structured outputs is non-trivial.
At the same time, {\CENs} let us represent $\prob{\Yv \mid \xv, \thetav}$ using arbitrary graphical models.
To be concrete, we consider the case where the targets are binary vectors, $\yv \in \{0, 1\}^{m}$, and explanations are represented by  CRFs~\citep{lafferty2001crf} with linear potential functions.

The predictive distribution $\prob{\Yv \mid \xv, \thetav}$ represented by a CRF takes the following form:
\begin{equation}
    \label{eq:generic-CRF}
    \prob{\Yv \mid \xv, \thetav} := \frac{1}{Z_{\thetav}(\xv)} \prod_{a \in \Ac} \Psi_a (\yv_a, \xv_a; \thetav)
\end{equation}
where $Z_{\thetav}(\xv)$ is the normalizing constant and $a \in \Ac$ indexes subsets of variables in $\xv$ and $\yv$ that correspond to the factors:
\begin{equation}
    \label{eq:generic-CRF-linear-factor}
    \Psi_a (\yv_a, \xv_a; \thetav) := \exp\cbb{\sum_{k=1}^K \thetav_{ak} f_{ak}(\xv_a, \yv_a)},
\end{equation}
where $\{f_{ak}(\xv_a, \yv_a)\}_{k=1}^K$ is a collection of feature vectors associated with factor $\Psi_a (\yv_a, \xv_a; \thetav)$.
For interpretability purposes, we are interested in CRFs with feature vectors that are linear or bi-linear in $\xv$ and $\yv$.
There is a variety of application-specific CRF models developed in the literature \citep[\eg, see][]{sutton2012crftutorial}.
While in the following section, we discuss learning and inference more generally, in Section~\ref{sec:applications-survival-analysis} we develop a {\CEN} model with linear chain CRF explanations for solving survival analysis tasks.

\subsection{Inference and Learning}
\label{sec:inference-and-learning}

{\CENs} with deterministic encoders are convenient since the posterior, $\prob{\thetav \mid \yv, \xv, \cv}$, collapses to a point $\thetav^\star = \phi(\cv)$.
Inference in such models is done in two steps: (1) first, compute $\theta^\star$, then (2) using $\theta^\star$ as parameters, compute the predictive distribution, $\prob{\yv \mid \xv, \thetav^\star}$.
To train the model, we can optimize its log likelihood on the training data.
To make a prediction using a trained {\CEN} model, we infer $\hat\yv = \argmax_{\yv} \prob{\yv \mid \xv, \thetav^\star}$.
For classification (and regression) computing predictions is straightforward.
Below, we show how to compute predictions for {\CEN} with CRF-based explanations.

\subsubsection{Inference for {\CEN} with Structured Explanations}
\label{sec:inference-CRF}

Given a CRF model \eqref{eq:generic-CRF}, we can make a prediction $\hat\yv$ for inputs $(\cv, \xv)$ by performing inference:
\begin{equation}
    \label{eq:crf-prediction}
    \hat\yv (\thetav^\star) = \argmax_{\yv \in \Yc} \prob{\yv \mid \xv, \thetav^\star} = \argmax_{\yv \in \Yc} \sum_{a=1}^A \sum_{k=1}^K \thetav^\star_{ak} f_{ak}(\xv_a, \yv_a)
\end{equation}
Depending on the structure of the CRF model (\eg, linear chain, tree-structured model, etc.), we could use different inference algorithms, such the Viterbi algorithm or variational inference, in order to solve \eqref{eq:crf-prediction} \citep[see Ch. 4,][for an overview and examples]{sutton2012crftutorial}.
The key point here is that having $\prob{\yv \mid \xv, \thetav^\star}$ or $\hat\yv(\thetav^\star)$ computable in an (approximate) functional form, lets us construct different objective functions, \eg, $\Lc(\{\yv_i, \xv_i, \cv_i\}_{i=1}^N, \wv)$, and learn parameters of the {\CEN} model end-to-end using gradient methods, which are standard in deep learning.
In Section \ref{sec:applications-survival-analysis}, we construct a specific objective function for survival analysis.

\subsubsection{Learning via Likelihood Maximization and Posterior Regularization}
\label{sec:learning-mle-posterior-reg}

In this paper, we use the negative log likelihood (NLL) objective for learning {\CEN} models:
\begin{equation}
    \label{eq:CEN-loglik}
    \Lc(\{\yv_i, \xv_i, \cv_i\}_{i=1}^N, \wv) := \frac{1}{N} \sum_{i=1}^N \log \prob{\yv_i \mid \xv_i, \thetav = \phi_{\wv}(\cv_i)}
\end{equation}
$L_1$, $L_2$, and other types of regularization imposed on $\thetav$ can be added to the objective \eqref{eq:CEN-loglik}.
Such regularizers, as well as the dictionary constraint introduced in Section \ref{sec:constrained-det-enc}, can be seen as a form of \emph{posterior regularization} \citep{ganchev2010posterior} and are important for achieving the best performance and interpretability.

\section{Related work}\label{sec:related_work}

Contextual explanation networks combine multiple threads of research that we discuss below.

\subsection{Deep graphical models}
The idea of combining deep networks with graphical models has been explored extensively.
Notable threads of recent work include: replacing task-specific feature engineering with task-agnostic general representations (or embeddings) discovered by deep networks \citep{collobert2011natural,rudolph2016efe,rudolph2017structured}, representing energy functions~\citep{belanger2016structured} and potential functions~\citep{jaderberg2014deep} with neural networks, encoding learnable structure into Gaussian processes with deep and recurrent networks~\citep{wilson2016dkl,alshedivat2017srk}, or learning state-space models on top of nonlinear embeddings of the observations~\citep{gao2016lds,johnson2016composing,krishnan2017structured}.
The goal of this body of work is to design principled structured probabilistic models that enjoy the flexibility of deep learning.
The key difference between {\CENs} and the previous art is that the latter directly integrate neural networks \emph{into} graphical models as components (embeddings, potential functions, etc.).
While flexible, the resulting \emph{deep graphical models} could no longer be interpreted in terms of crisp relationships between specific variables of interest.\footnote{%
To see why this is the case, consider graphical models given in Figure~\ref{fig:PGM} which relate input, $\Xv$, and target, $\Yv$, variables using linear pairwise potential functions.
Linearity allows to directly interpret parameters of the model as associations between the variables.
Substituting inputs, $\Xv$, with deep representations or defining potentials via neural networks would result in a more powerful model.
However, precise relationships between the variables will be no longer directly readable from the model parameters.}
{\CENs}, on the other hand, preserve the simplicity of the explanations and shift complexity into conditioning on the context.

\subsection{Context representation}
Generating probabilistic models after conditioning on a context is the key aspect of our approach.
Previous work on context-specific graphical models represented contexts with a discrete variable that enumerated a finite number of possible contexts~\citep[][Ch.~5.3]{koller2009pgm}.
{\CENs}, on the other hand, are designed to handle arbitrary complex context representations.
Context-specific approaches are widely used in language modeling where the context is typically represented with trainable embeddings~\citep{rudolph2016efe}.
We also note that few-shot learning explicitly considers a setup where the context is represented by a small set of labeled examples~\citep{santoro2016meta,garnelo2018conditional}.

\subsection{Meta-learning}
The way {\CENs} operate resembles the meta-learning setup.
In meta-learning, the goal is to learn a meta-model which, given a task, can produce another model capable of solving the task~\citep{thrun1998learning}.
The representation of the task can be seen as the context while produced task-specific models are similar to {\CEN}-generated explanations.
Meta-training a deep network that generates parameters for another network has been successfully used for zero-shot~\citep{lei2015predicting,changpinyo2016synthesized} and few-shot~\citep{edwards2016towards,vinyals2016matching} learning, cold-start recommendations~\citep{vartak2017meta}, and a few other scenarios~\citep{bertinetto2016learning,de2016dynamic,ha2016hypernetworks}, but is not suitable for interpretability purposes.
In contrast, {\CENs} generate parameters for models from a restricted class (potentially, based on domain knowledge) and use the attention mechanism~\citep{xu2015show} to further improve interpretability.

\subsection{Model interpretability}
While there are many ways to define interpretability~\citep{lipton2016mythos,doshi2017towards},
our discussion focuses on explanations defined as simple models that locally approximate behavior of a complex model.
A few methods that allow to construct such explanations in a \emph{post-hoc} manner have been proposed recently~\citep{ribeiro2016trust,shrikumar2017learning,lundberg2017shap}, some of which we review in the next section.
In contrast, {\CENs} learn to generate such explanations along with predictions.
There are multiple other complementary approaches to interpretability ranging from a variety of visualization techniques~\citep{simonyan2014vgg,yosinski2015understanding,mahendran2015understanding,karpathy2015visualizing}, to explanations by example~\citep{caruana1999case,kim2014bayesian,kim2016examples,koh2017understanding}, to natural language rationales~\citep{lei2016rationalizing}.
Finally, our framework encompasses the so-called \emph{personalized} or \emph{instance-specific} models that learn to partition the space of inputs and fit local sub-models~\citep{joseph2012local}.

\section{Analysis}
\label{sec:analysis}

In this section, we dive into the analysis of {\CEN} as a class of probabilistic models.
First, we mention special cases of {\CEN} model class known in the literature, such as mixture-of-experts \citep{jacobs1991adaptive} and varying-coefficient models \citep{hastie1993vcm}.
Then, we discuss implications of the {\CEN} structure, a potential failure mode of {\CEN} with deterministic encoders and how to rectify it using conditional entropy regularization, and finally analyze relationship between {\CEN}-generated and post-hoc explanations.
Readers who are mostly interested in empirical properties and applications may skip this section.

\subsection{Special Cases of {\CEN}}
\label{sec:CEN-special-calses}

\paragraph{Mixtures of Experts.}
So far, we have represented $\prob[\wv]{\thetav \mid \cv}$ by a delta-function centered around the output of the encoder.
It is natural to extend $\prob[\wv]{\thetav \mid \cv}$ to a mixture of delta-distributions, in which case {\CENs} recover the mixtures-of-experts model~\citep[MoE,][]{jacobs1991adaptive}.
To see this, let $\Dv$ be a dictionary of experts, and define $\prob[\wv,\Dv]{\thetav \mid \cv} := \sum_{k=1}^K \prob[\wv]{k \mid \cv} \delta(\thetav, \thetav_k)$.
The log-likelihood for {\CEN} in such case is the same as for MoE:

\vspace{1.3ex}
\begin{minipage}{0.37\textwidth}
    \flushleft
    \includegraphics[width=\textwidth]{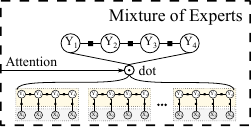}
\end{minipage}%
\begin{minipage}{0.585\textwidth}
    \begin{equation}
        \label{eq:MoE-lik}
        \begin{aligned}
            \MoveEqLeft \log \prob[\wv, \Dv]{\yv_i \mid \xv_i, \cv_i} \\
            & = \log \int \prob{\yv_i | \xv_i, \thetav} \prob[\wv,\Dv]{\thetav | \cv_i} d\thetav \\
            & = \log \sum_{k=1}^K \prob[\wv]{k | \cv_i} \prob{\yv_i | \xv_i, \thetav_k}
        \end{aligned}
    \end{equation}
\end{minipage}
\vspace{1.3ex}

As in Section \ref{sec:constrained-det-enc}, \noindent $\prob[\wv]{k \mid \Cv}$ is represented with a soft attention over the dictionary, $\Dv$, which is now used to combine predictions of the experts with parameters $\{\thetav_k\}_{k=1}^K$ instead of constructing a single context-specific explanation.
Learning of MoE models is done either by optimizing the likelihood or via expectation maximization (EM).
Note another difference between {\CEN} and MoE is that the latter assumed that $\cv \equiv \xv$ and that both $\prob{\yv \mid \xv, \thetav}$ and $\prob{\thetav \mid \cv}$ can be represented by arbitrary complex model classes, ignoring interpretability.

\vspace{-1ex}
\paragraph{Varying-Coefficient Models.}
In statistics, there is a class of (generalized) regression models, called varying-coefficient models \citep[VCMs,][]{hastie1993vcm}, in which coefficients of linear models are allowed to be smooth deterministic functions of other variables (called the ``effect modifiers'').
Interestingly, the motivation for VCM was to increase flexibility of linear regression.
In the original work, \citet{hastie1993vcm} focused on simple dynamic (temporal) linear models and on nonparametric estimation of the varying coefficients, where each coefficient depended on a different effect variable.
CEN generalizes VCM by (i) allowing parameters, $\thetav$, to be random variables that depend on the context, $\cv$, nondeterministically, (ii) letting the ``effect modifiers'' to be high-dimensional context variables (not just scalars), and (iii) modeling the effects using deep neural networks.
In other words, CEN alleviates the limitations of VCM by leveraging the probabilistic graphical models and deep learning frameworks.

\subsection{Implications of the Structure of {\CENs}}
\label{sec:CEN-structure-implications}

{\CENs} represent the predictive distribution in a compound form \citep{lindsay1995mixture}:
\begin{equation*}
    \prob{\Yv \mid \Xv, \Cv} = \int \prob{\Yv \mid \Xv, \thetav} \prob{\thetav \mid \Cv} d\thetav
\end{equation*}
and we assume that the data is generated according to $\Yv \sim \prob{\Yv \mid \Xv, \thetav}$, $\thetav \sim \prob{\thetav \mid \Cv}$.
We would like to understand:

\vspace{1.5ex}
\begin{blockquote}
    \emph{Can {\CEN} represent any conditional distribution, $\prob{\Yv \mid \Xv, \Cv}$, when the class of explanations is limited (\eg, to linear models)? If not, what are the limitations?}
\end{blockquote}
\vspace{1.5ex}

\noindent
Generally, {\CEN} can be seen as a mixture of predictors.
Such mixture models could be quite powerful as long as the mixing distribution, $\prob{\thetav \mid \Cv}$, is rich enough.
In fact, even a finite mixture exponential family regression models can approximate any smooth $d$-dimensional density at a rate $O(m^{-4/d})$ in the KL-distance~\citep{jiang1999hierarchical}.
This result suggests that representing the predictive distribution with contextual mixtures should not limit the representational power of the model.
However, there are two caveats:
\begin{itemize}[itemsep=0pt,parsep=1ex,topsep=1ex]
    \item[(i)] In practice, $\prob{\thetav \mid \Cv}$ is limited, since we represent it either with a delta-function, a finite mixture, or a simple distribution parametrized by a deep network.

    \item[(ii)] Classical predictive mixtures (including MoE) do not separate input features into two subsets, $\cv$ and $\xv$.
    We do this intentionally to produce explanations in terms of specific variables of interest that could be useful for interpretability or model diagnostics down the line.
    However, it could be the case that $\xv$ contains only some limited information about $\yv$, which could limit the predictive power of the full model.
\end{itemize}
To address point (i), we consider $\prob{\thetav \mid \cv}$ that fully factorizes over the dimensions of $\thetav$: $\prob{\thetav \mid \cv} = \prod_{j}\prob{\theta_j \mid \cv}$, and assume that explanations, $\prob{\Yv \mid \xv, \thetav}$, also factorize according to some underlying graph, $\Gc_{\Yv} = (\Vc_{\Yv}, \Ec_{\Yv})$.
The following proposition shows that in such case $\prob{\Yv \mid \xv, \cv}$ inherits the factorization properties of the explanation class.
\begin{proposition}
    \label{prop:CEN-conditional-independence}
    Let $\prob{\thetav \mid \cv} := \prod_{j}\prob{\theta_j \mid \cv}$ and let $\prob{\Yv \mid \xv, \thetav}$ factorize according to some graph $\Gc_{\Yv} = (\Vc_{\Yv}, \Ec_{\Yv})$.
    Then, $\prob{\Yv \mid \xv, \cv}$ defined by {\CEN} with $\prob{\thetav \mid \cv}$ encoder and $\prob{\Yv \mid \xv, \thetav}$ explanations also factorizes according to $\Gc$.
\end{proposition}
\begin{proof}
    The statement directly follows from the definition of {\CEN}
    (see Appendix \ref{proof:CEN-conditional-independence}).
\end{proof}
\begin{remark}
    All encoders, $\prob{\thetav \mid \cv}$, considered in this paper, including delta functions and their mixtures, fully factorize over the dimensions of $\thetav$.
\end{remark}
\begin{remark}
    The proposition has no implications for the case of scalar targets, $\yv$.
    However, in case of structured prediction, regardless of how good the context encoder is, {\CEN} will strictly assume the same set of independencies as given by the explanation class, $\prob{\Yv \mid \xv, \thetav}$.
\end{remark}

\noindent
As indicated in point (ii), {\CENs} assume a fixed split of the input features into context, $\cv$, and variables of interest, $\xv$, which has interesting implications.
Ideally, we would like $\xv$ to be a good predictor of $\yv$ in any context $\cv$.
For instance, following our motivation example (see Figure~\ref{fig:illustration}), if $\cv$ distinguishes between urban and rural areas, $\xv$ must encode enough information for predicting poverty \emph{within} urban or rural neighborhoods.
However, since the variables of interest are often manually selected (\eg, by a domain expert) and limited, we may encounter the following (not mutually exclusive) situations:
\begin{itemize}[itemsep=0pt,parsep=1ex,topsep=1ex]
    \item[(a)] $\cv$ may happen to be a strong predictor of $\yv$ and already contain information available in $\xv$ (\eg, it is the case when $\xv$ is derived from $\cv$).
    \item[(b)] $\xv$ may happen to be a poor predictor of $\yv$, even within the context specified by $\cv$.
\end{itemize}
In both cases, {\CEN} may learn to ignore $\xv$, leading to essentially meaningless explanations.
In the next section, we show that, if (a) is the case, regularization can help eliminate such behavior.
Additionally, if (b) is the case, \ie, $\xv$ are bad features for predicting $\yv$ (and for seeking explanation in terms of these features), {\CEN} must indicate that.
It turns out that the accuracy of {\CEN} depends on the quality of $\xv$, as empirically shown in Section~\ref{sec:cen-consistency}.

\subsection{Conditional Entropy Regularization}
\label{sec:encoder-regularization}

{\CEN} has a failure mode: when the context $\cv$ is highly predictive of the targets $\yv$ and the encoder is represented by a powerful model, {\CEN} may learn to rely entirely on the context variables.
In such case, the encoder would generate spurious explanations, one for each target class.
For example, for binary targets, $\yv \in \{0, 1\}$, {\CEN} may learn to always map $\cv$ to either $\thetav_0$ or $\thetav_1$ when $\yv$ is 0 or 1, respectively.
In other words, $\thetav$ (as a function of $\cv$) would become highly predictive of $\yv$ on its own, and hence $\prob{\Yv \mid \xv, \thetav} \approx \prob{\Yv \mid \thetav}$, \ie, $\Yv$ would be (approximately) conditionally independent of $\Xv$ given $\thetav$.
This is problematic from the interpretation point of view since explanations would become spurious, \ie, no longer used to make predictions from the variables of interest.

\begin{figure}[t]
\centering
\begin{subfigure}[b]{0.49\textwidth}
\begin{tikzpicture}[scale=0.37]

	\coordinate (P1) at (-30,15); %
	\coordinate (P2) at (20,5); %

	\coordinate (A1) at (0,0); %
	\coordinate (A2) at (0,-12); %
	\coordinate (A9) at (0,-8); %
	\coordinate (A13) at (0,-4);

	\coordinate (A3) at ($(P1)!.8!(A2)$); %
	\coordinate (A4) at ($(P1)!.8!(A1)$);
	\coordinate (A10) at ($(P1)!.8!(A9)$);
	\coordinate (A14) at ($(P1)!.8!(A13)$);

	\coordinate (A7) at ($(P2)!.6!(A2)$);
	\coordinate (A8) at ($(P2)!.6!(A1)$);
	\coordinate (A11) at ($(P2)!.6!(A9)$);
	\coordinate (A15) at ($(P2)!.6!(A13)$);

	\coordinate (A5) at
	  (intersection cs: first line={(A8) -- (P1)},
			    second line={(A4) -- (P2)});
	\coordinate (A6) at
	  (intersection cs: first line={(A7) -- (P1)},
			    second line={(A3) -- (P2)});
	\coordinate (A12) at
	  (intersection cs: first line={(A11) -- (P1)},
			    second line={(A10) -- (P2)});
    \coordinate (A16) at
	  (intersection cs: first line={(A15) -- (P1)},
			    second line={(A14) -- (P2)});

    \draw[thick] (A3) -- (A4);
    \draw[fill=black] (A3) circle (0.15) node[above left] {\fontsize{12}{14}\selectfont$C_1$};
    \draw[fill=black] (A10) circle (0.15) node[above left] {\fontsize{12}{14}\selectfont$C_2$};
    \draw[fill=black] (A14) circle (0.15) node[above left] {\fontsize{12}{14}\selectfont$C_3$};
    \draw[fill=black] (A4) circle (0.15) node[above left] {\fontsize{12}{14}\selectfont$C_4$};

	\fill[gray!30,opacity=0.75] (A2) -- (A3) -- (A6) -- (A7) -- cycle;
	\draw[thick] (A2) -- (A3);
	\draw[thick] (A2) -- (A7);
	\draw[thick] (A3) -- (A6);
	\draw[thick] (A6) -- (A7);

	\coordinate (BVA1) at ($(A2)!.8!(A3)$);
	\coordinate (BVA2) at ($(A3)!.4!(A6)$);
	\coordinate (BVA3) at ($(A2)!.82!(A3)$);
	\coordinate (BVA4) at ($(A3)!.38!(A6)$);

	\fill[nice_orange!50,opacity=1.0] (BVA1) -- (BVA3) -- (BVA4) -- (BVA2) -- cycle;
	\draw[nice_orange, very thick] (BVA1) -- (BVA2);

	\node at (A3) [xshift=0.7cm,yshift=-1.3cm] {\fontsize{12}{14}\selectfont$X_1$};
	\node at (A3) [xshift=4.4cm,yshift=-1.1cm] {\fontsize{12}{14}\selectfont$X_2$};

	\coordinate (D1) at ($(A3)!.9!(A7)$);
	\coordinate (D2) at ($(A3)!.9!(A7)$);
	\coordinate (D3) at ($(A2)!.8!(A6)$);
	\coordinate (D4) at ($(A2)!.7!(A6)$);
    \coordinate (D5) at ($(D4)!.55!(D1)$);
    \coordinate (D6) at ($(D2)!.5!(A6)$);
    \coordinate (D7) at ($(A2)!.8!(D5)$);

	\foreach \i in {1,2,...,7}
	{
	  \draw (D\i) node[circle,red,draw,fill=nice_red,minimum size=1mm,inner sep=2pt] {};
	}

	\fill[gray!30,opacity=0.75] (A9) -- (A10) -- (A12) -- (A11) -- cycle;
	\draw[thick] (A9) -- (A10);
	\draw[thick] (A9) -- (A11);
    \draw[thick] (A10) -- (A12);
    \draw[thick] (A12) -- (A11);

    \coordinate (MBVA1) at ($(A11)!.3!(A12)$);
	\coordinate (MBVA2) at ($(A10)!.3!(A12)$);
	\coordinate (MBVA3) at ($(A11)!.26!(A12)$);
	\coordinate (MBVA4) at ($(A10)!.26!(A12)$);

	\fill[nice_green!50,opacity=1.0] (MBVA1) -- (MBVA3) -- (MBVA4) -- (MBVA2) -- cycle;
	\draw[nice_green, very thick] (MBVA1) -- (MBVA2);

    \coordinate (E1) at ($(A10)!.3!(A11)$);
	\coordinate (E2) at ($(A10)!.1!(A11)$);
	\coordinate (E3) at ($(A9)!.2!(A12)$);
	\coordinate (E4) at ($(A9)!.3!(A12)$);
    \coordinate (E5) at ($(E4)!.55!(E1)$);
    \coordinate (E6) at ($(E2)!.5!(A9)$);

	\foreach \i in {1,2,...,6}
	{
	  \draw (E\i) node[diamond,blue,draw,fill=nice_blue,minimum size=1mm,inner sep=2pt] {};
	}

	\fill[gray!30,opacity=0.75] (A13) -- (A14) -- (A16) -- (A15) -- cycle; %
	\draw[thick] (A13) -- (A14);
	\draw[thick] (A13) -- (A15);
    \draw[thick] (A14) -- (A16);
    \draw[thick] (A15) -- (A16);

    \coordinate (MTVA1) at ($(A15)!.3!(A16)$);
	\coordinate (MTVA2) at ($(A14)!.3!(A16)$);
	\coordinate (MTVA3) at ($(A15)!.26!(A16)$);
	\coordinate (MTVA4) at ($(A14)!.26!(A16)$);

	\fill[nice_green!50,opacity=1.0] (MTVA1) -- (MTVA3) -- (MTVA4) -- (MTVA2) -- cycle;
	\draw[nice_green, very thick] (MTVA1) -- (MTVA2);

    \coordinate (V1) at ($(A14)!.8!(A13)$);
    \coordinate (V3) at ($(A16)!.8!(A15)$);

    \coordinate (F1) at ($(V1)!.2!(V3)$);
	\coordinate (F2) at ($(V1)!.6!(V3)$);
	\coordinate (F3) at ($(V1)!.4!(V3)$);
	\coordinate (F4) at ($(V1)!.9!(V3)$);
	\coordinate (F5) at ($(V1)!.7!(V3)$);

	\foreach \i in {1,2,...,5}
	{
	  \draw (F\i) node[diamond,blue,draw,fill=nice_blue,minimum size=1mm,inner sep=2pt] {};
	}

	\fill[gray!30,opacity=0.75] (A1) -- (A4) -- (A5) -- (A8) -- cycle; %
	\node at (barycentric cs:A1=1,A4=1,A5=1,A8=1) {};
	\draw[thick] (A1) -- (A4);
	\draw[thick] (A1) -- (A8);
    \draw[thick] (A4) -- (A5);
    \draw[thick] (A5) -- (A8);

	\coordinate (VA1) at ($(A1)!.8!(A4)$);
	\coordinate (VA2) at ($(A4)!.4!(A5)$);
	\coordinate (VA3) at ($(A1)!.83!(A4)$);
	\coordinate (VA4) at ($(A4)!.37!(A5)$);

	\fill[nice_orange!50,opacity=1.0] (VA1) -- (VA3) -- (VA4) -- (VA2) -- cycle;
	\draw[nice_orange, very thick] (VA1) -- (VA2);

	\coordinate (V1) at ($(A4)!.2!(A1)$);
	\coordinate (V3) at ($(A5)!.2!(A8)$);

	\coordinate (G1) at ($(V1)!.8!(V3)$);
	\coordinate (G2) at ($(V1)!.4!(V3)$);
	\coordinate (G3) at ($(V1)!.6!(V3)$);
	\coordinate (G4) at ($(V1)!.1!(V3)$);
	\coordinate (G5) at ($(V1)!.3!(V3)$);

	\foreach \i in {1,2,...,5}
	{
	  \draw (G\i) node[circle,red,draw,fill=nice_red,minimum size=1mm,inner sep=2pt] {};
	}
\end{tikzpicture}
\caption{}\label{fig:synthetic-illustration-noentreg}
\end{subfigure}
\begin{subfigure}[b]{0.49\textwidth}
\begin{tikzpicture}[scale=0.37]

	\coordinate (P1) at (-30,15); %
	\coordinate (P2) at (20,5); %

	\coordinate (A1) at (0,0); %
	\coordinate (A2) at (0,-12); %
	\coordinate (A9) at (0,-8); %
	\coordinate (A13) at (0,-4);

	\coordinate (A3) at ($(P1)!.8!(A2)$); %
	\coordinate (A4) at ($(P1)!.8!(A1)$);
	\coordinate (A10) at ($(P1)!.8!(A9)$);
	\coordinate (A14) at ($(P1)!.8!(A13)$);

	\coordinate (A7) at ($(P2)!.6!(A2)$);
	\coordinate (A8) at ($(P2)!.6!(A1)$);
	\coordinate (A11) at ($(P2)!.6!(A9)$);
	\coordinate (A15) at ($(P2)!.6!(A13)$);

	\coordinate (A5) at
	  (intersection cs: first line={(A8) -- (P1)},
			    second line={(A4) -- (P2)});
	\coordinate (A6) at
	  (intersection cs: first line={(A7) -- (P1)},
			    second line={(A3) -- (P2)});
	\coordinate (A12) at
	  (intersection cs: first line={(A11) -- (P1)},
			    second line={(A10) -- (P2)});
    \coordinate (A16) at
	  (intersection cs: first line={(A15) -- (P1)},
			    second line={(A14) -- (P2)});

    \draw[thick] (A3) -- (A4);
    \draw[fill=black] (A3) circle (0.15) node[above left] {\fontsize{12}{14}\selectfont$C_1$};
    \draw[fill=black] (A10) circle (0.15) node[above left] {\fontsize{12}{14}\selectfont$C_2$};
    \draw[fill=black] (A14) circle (0.15) node[above left] {\fontsize{12}{14}\selectfont$C_3$};
    \draw[fill=black] (A4) circle (0.15) node[above left] {\fontsize{12}{14}\selectfont$C_4$};

	\fill[gray!30,opacity=0.75] (A2) -- (A3) -- (A6) -- (A7) -- cycle;
	\draw[thick] (A2) -- (A3);
	\draw[thick] (A2) -- (A7);
	\draw[thick] (A3) -- (A6);
	\draw[thick] (A6) -- (A7);

	\coordinate (BVA1) at ($(A2)!.5!(A7)$);
	\coordinate (BVA2) at ($(A3)!.66!(A6)$);
	\coordinate (BVA3) at ($(A2)!.48!(A7)$);
	\coordinate (BVA4) at ($(A3)!.64!(A6)$);

	\fill[nice_orange!50,opacity=1.0] (BVA1) -- (BVA3) -- (BVA4) -- (BVA2) -- cycle;
	\draw[nice_orange, very thick] (BVA1) -- (BVA2);

	\node at (A3) [xshift=0.7cm,yshift=-1.3cm] {\fontsize{12}{14}\selectfont$X_1$};
	\node at (A3) [xshift=4.4cm,yshift=-1.1cm] {\fontsize{12}{14}\selectfont$X_2$};

	\coordinate (D1) at ($(A3)!.9!(A7)$);
	\coordinate (D2) at ($(A3)!.9!(A7)$);
	\coordinate (D3) at ($(A2)!.8!(A6)$);
	\coordinate (D4) at ($(A2)!.7!(A6)$);
    \coordinate (D5) at ($(D4)!.55!(D1)$);
    \coordinate (D6) at ($(D2)!.5!(A6)$);
    \coordinate (D7) at ($(A2)!.8!(D5)$);

	\foreach \i in {1,2,...,7}
	{
	  \draw (D\i) node[circle,red,draw,fill=nice_red,minimum size=1mm,inner sep=2pt] {};
	}

	\fill[gray!30,opacity=0.75] (A9) -- (A10) -- (A12) -- (A11) -- cycle;
	\draw[thick] (A9) -- (A10);
	\draw[thick] (A9) -- (A11);
    \draw[thick] (A10) -- (A12);
    \draw[thick] (A12) -- (A11);

    \coordinate (MBVA1) at ($(A9)!.5!(A11)$);
	\coordinate (MBVA2) at ($(A10)!.66!(A12)$);
	\coordinate (MBVA3) at ($(A9)!.48!(A11)$);
	\coordinate (MBVA4) at ($(A10)!.64!(A12)$);

	\fill[nice_orange!50,opacity=1.0] (MBVA1) -- (MBVA3) -- (MBVA4) -- (MBVA2) -- cycle;
	\draw[nice_orange, very thick] (MBVA1) -- (MBVA2);

    \coordinate (E1) at ($(A10)!.3!(A11)$);
	\coordinate (E2) at ($(A10)!.1!(A11)$);
	\coordinate (E3) at ($(A9)!.2!(A12)$);
	\coordinate (E4) at ($(A9)!.3!(A12)$);
    \coordinate (E5) at ($(E4)!.55!(E1)$);
    \coordinate (E6) at ($(E2)!.5!(A9)$);

	\foreach \i in {1,2,...,6}
	{
	  \draw (E\i) node[diamond,blue,draw,fill=nice_blue,minimum size=1mm,inner sep=2pt] {};
	}

	\fill[gray!30,opacity=0.75] (A13) -- (A14) -- (A16) -- (A15) -- cycle; %
	\draw[thick] (A13) -- (A14);
	\draw[thick] (A13) -- (A15);
    \draw[thick] (A14) -- (A16);
    \draw[thick] (A15) -- (A16);

    \coordinate (MTVA1) at ($(A13)!.4!(A14)$);
	\coordinate (MTVA2) at ($(A15)!.6!(A16)$);
	\coordinate (MTVA3) at ($(A13)!.37!(A14)$);
	\coordinate (MTVA4) at ($(A15)!.57!(A16)$);

	\fill[nice_green!50,opacity=1.0] (MTVA1) -- (MTVA3) -- (MTVA4) -- (MTVA2) -- cycle;
	\draw[nice_green, very thick] (MTVA1) -- (MTVA2);

    \coordinate (V1) at ($(A14)!.8!(A13)$);
    \coordinate (V3) at ($(A16)!.8!(A15)$);

    \coordinate (F1) at ($(V1)!.2!(V3)$);
	\coordinate (F2) at ($(V1)!.6!(V3)$);
	\coordinate (F3) at ($(V1)!.4!(V3)$);
	\coordinate (F4) at ($(V1)!.9!(V3)$);
	\coordinate (F5) at ($(V1)!.7!(V3)$);

	\foreach \i in {1,2,...,5}
	{
	  \draw (F\i) node[diamond,blue,draw,fill=nice_blue,minimum size=1mm,inner sep=2pt] {};
	}

	\fill[gray!30,opacity=0.75] (A1) -- (A4) -- (A5) -- (A8) -- cycle; %
	\node at (barycentric cs:A1=1,A4=1,A5=1,A8=1) {};
	\draw[thick] (A1) -- (A4);
	\draw[thick] (A1) -- (A8);
    \draw[thick] (A4) -- (A5);
    \draw[thick] (A5) -- (A8);

	\coordinate (VA1) at ($(A1)!.4!(A4)$);
	\coordinate (VA2) at ($(A8)!.6!(A5)$);
	\coordinate (VA3) at ($(A1)!.37!(A4)$);
	\coordinate (VA4) at ($(A8)!.57!(A5)$);

	\fill[nice_green!50,opacity=1.0] (VA1) -- (VA3) -- (VA4) -- (VA2) -- cycle;
	\draw[nice_green, very thick] (VA1) -- (VA2);

	\coordinate (V1) at ($(A4)!.2!(A1)$);
	\coordinate (V3) at ($(A5)!.2!(A8)$);

	\coordinate (G1) at ($(V1)!.8!(V3)$);
	\coordinate (G2) at ($(V1)!.4!(V3)$);
	\coordinate (G3) at ($(V1)!.6!(V3)$);
	\coordinate (G4) at ($(V1)!.1!(V3)$);
	\coordinate (G5) at ($(V1)!.3!(V3)$);

	\foreach \i in {1,2,...,5}
	{
	  \draw (G\i) node[circle,red,draw,fill=nice_red,minimum size=1mm,inner sep=2pt] {};
	}
\end{tikzpicture}
\caption{}\label{fig:synthetic-illustration-entreg}
\end{subfigure}
\caption{A toy synthetic dataset and two linear explanations (\textcolor{nice_green}{green} and \textcolor{nice_orange}{orange}) produced by a {\CEN} model trained (a) with no regularization or (b) with conditional entropy regularization.}
\label{fig:synthetic-illustration}
\end{figure}

Note that such a model would be accurate only when the generated $\thetav$ is always highly predictive of $\Yv$, \ie, when the conditional entropy $\Hc(\Yv \mid \thetav)$ is low.
Following this observation, we propose to regularize the model by approximately \emph{maximizing} $\Hc(\Yv \mid \thetav)$.
For a {\CEN} with a deterministic encoder (Sections~\ref{sec:det-enc} and \ref{sec:constrained-det-enc}), we can compute an unbiased estimate of $\Hc(\Yv \mid \thetav)$ given a mini-batch of samples from the dataset as follows:
\begin{align}
    \label{eq:cond-ent-reg}
    \Hc(\Yv \mid \thetav)
    & = \int \prob{\yv, \thetav} \log \prob{\yv \mid \thetav} d\yv d\thetav \\
    & = \ep[(\cv, \xv) \sim \prob{\Cv, \Xv}]{\int \prob{\yv \mid \xv, \phi(\cv)} \log \ep[\xv^\prime \sim \prob{\Xv \mid \cv}]{\prob{\yv \mid \xv^\prime, \phi(\cv)}} d\yv} \label{eq:cond-entropy-batch} \\
    & \approx \frac{1}{|B|} \sum_{i \in B} \int \prob{\yv \mid \xv_i, \phi(\cv_i)} \log \left[\sum_{\xv^\prime \sim \prob{\Xv \mid \cv_i}} \prob{\yv \mid \xv^\prime, \phi(\cv_i)}\right] d\yv \label{eq:cond-entropy-batch-approx}
\end{align}
In the given expressions, elements of $B$ index training samples (\eg, $B$ represents a mini-batch), \eqref{eq:cond-entropy-batch} is obtained by using the definition of {\CEN} and marginalizing out $\thetav$, \eqref{eq:cond-entropy-batch-approx} is a stochastic estimate that approximates expectations using a mini-batch and samples from $\prob{\Xv \mid \cv_i}$.
In practice, approximate samples $\xv^\prime$ from the latter distribution can be obtained either by simply perturbing $\xv_i$ or first learning $\prob{\Xv \mid \Cv}$ and then sampling from it.
Intuitively, if the predictions are accurate while $\Hc(\Yv \mid \thetav)$ is high, we can be sure that {\CEN} learned to generate contextual $\thetav$'s that are uncorrelated with the targets but result into accurate conditional models, $\prob{\Yv \mid \xv, \thetav}$.

\vspace{-0.5ex}
\paragraph{An illustration on synthetic data.}
To illustrate the problem, we consider a toy synthetic 3D dataset with 2 classes that are not separable linearly (Figure~\ref{fig:synthetic-illustration}).
The coordinates along the vertical axis $C$ correspond to different contexts, and $(X_1, X_2)$ represent variables of interest.
Note we can perfectly distinguish between the two classes by using only the context information.
{\CEN} with a dictionary of size 2 learns to select one of the two linear explanations for each of the contexts.
When trained without regularization (Figure~\ref{fig:synthetic-illustration-noentreg}), selected explanations are spurious hyperplanes since each of them is used for points of a single class only.
Adding entropy regularization (Figure~\ref{fig:synthetic-illustration-entreg}) makes {\CEN} select hyperplanes that meaningfully distinguish between the classes within different contexts.

\vspace{-0.5ex}
\paragraph{Quantifying contribution of the explanations.}
Starting from the introduction, we have argued that explanations are meaningful when they are used for prediction.
In other words, we would like explanations have a non-zero contribution to the overall accuracy of the model.
The following proposition quantifies the contribution of explanations to the predictive performance of entropy-regularized {\CEN}.

\begin{proposition}
    \label{prop:explanation-contribution-bound}
    Let {\CEN} with linear explanations have the expected predictive accuracy
    \begin{equation}
        \label{eq:cen-predictive-acc-bound}
        \ep[\Xv, \thetav \sim \prob{\Xv, \thetav}]{\prob{\hat \Yv = \Yv \mid \Xv, \thetav}} \geq 1 - \varepsilon,
    \end{equation}
    where $\varepsilon \in (0, 1)$ is small.
    Let also the conditional entropy be $\Hc(\Yv \mid \thetav) \geq \delta$ for some $\delta \geq 0$.
    Then, the expected contribution of the explanations to the predictive performance of {\CEN} is given by the following lower bound:
    \begin{equation}
        \label{eq:cen-exp-contribution-bound}
        \ep[\Xv, \thetav \sim \prob{\Xv, \thetav}]{\prob{\hat \Yv = \Yv \mid \Xv, \thetav} - \prob{\hat \Yv = \Yv \mid \thetav}} \geq \frac{\delta - 1}{\log |\Yc|} - \varepsilon,
    \end{equation}
    where $|\Yc|$ denotes the cardinality of the target space.
\end{proposition}
\begin{proof}
    The statement follows from Fano's inequality. For details, see Appendix~\ref{proof:explanation-contribution-bound}.
\end{proof}
\vspace{-3ex}
\begin{remark}
    The proposition states that explanations are meaningful (as contextual models) only when {\CEN} is accurate (\ie, the expected predictive error is less than $\varepsilon$) and the conditional entropy $\Hc(\Yv \mid \thetav)$ is high.
    High accuracy and low entropy imply spurious explanations.
    Low accuracy and high entropy imply that $\xv$ features are not predictive of $\yv$ within the class of explanations, suggesting to reconsider our modeling assumptions.
\end{remark}

\subsection{{\CEN}-generated vs. Post-hoc Explanations}
\label{sec:CEN-vs-LIME}

In this section, we analyze the relationship between {\CEN}-generated and {\LIME}-generated \emph{post-hoc} explanations.
Given a trained {\CEN}, we can use {\LIME} to approximate its decision boundary and compare the explanations produced by both methods.
The question we ask:

\begin{blockquote}
    \emph{How does the local approximation, $\hat \thetav$, relate to the actual explanation, $\thetav^\star$, generated and used by {\CEN} to make a prediction in the first place?}
\end{blockquote}

\noindent
For the case of binary\footnote{Analysis of the multi-class case can be reduced to the binary in the one-vs-all fashion.} classification, it turns out that when the context encoder is deterministic and the space of explanations is \emph{linear}, local approximations, $\hat\thetav$, obtained by solving \eqref{eq:LIME-general} recover the original {\CEN}-generated explanations, $\thetav^\star$.
Formally, our result is stated in the following theorem.

\begin{theorem}
\label{thm:LIME-CEN}
Let the explanations and the local approximations be in the class of linear models, $\prob{Y = 1 \mid \xv, \thetav} \propto \exp\left\{\xv^\top \thetav\right\}$.
Further, let the encoder be $L$-Lipschitz and pick a sampling distribution, $\pi_{\xv, \cv}$, that concentrates around the point $(\xv, \cv)$, such that $\prob[\pi_{\xv, \cv}]{\|\zv^\prime - \zv\| > t} < \varepsilon(t)$, where $\zv := (\xv, \cv)$ and $\varepsilon(t) \rightarrow 0$ as $t \rightarrow \infty$.
Then, if the loss function is defined as
\begin{equation}
    \Lc = \frac{1}{K} \sum_{k=1}^K \left(\logit{\prob{Y = 1 \mid \xv_k, \cv_k}} - \logit{\prob{Y = 1 \mid \xv_k, \thetav}}\right)^2,\, (\xv_k, \cv_k) \sim \pi_{\xv, \cv},
\end{equation}
the solution of \eqref{eq:LIME-general} concentrates around $\thetav^\star$ as $\mathbb{P}_{\pi_{\xv, \cv}}\left(\|\hat\thetav - \thetav^\star\| > t\right) \leq  \delta_{K, L}(t)$, $\delta_{K, L} \underset{t \rightarrow \infty}{\longrightarrow} 0$.
\end{theorem}
Intuitively, by sampling from a distribution sharply concentrated around $(\xv, \cv)$, we ensure that $\hat\thetav$ will recover $\thetav^\star$ with high probability.
A detailed proof is given in Appendix~\ref{proof:LIME-CEN}.

This result establishes an equivalence between the explanations generated by {\CEN} and those produced by {\LIME} post-hoc \emph{when approximating {\CEN}}.
Note that when {\LIME} is applied to a model other than {\CEN}, equivalence between explanations is not guaranteed.
Moreover, as we further show experimentally, certain conditions such as incomplete or noisy interpretable features may lead to {\LIME} producing inconsistent and erroneous explanations.

\section{Case Studies}
\label{sec:case-studies}

In this section, we move to a number of case studies where we empirically analyze properties of the proposed {\CEN} framework on classification and survival analysis tasks.
In particular, we evaluate {\CEN} with linear explanations on a few classification tasks that involve different data modalities of the context (\eg, images or text).
For survival prediction, we design {\CEN} architectures with structured explanations, derive learning and inference algorithms, and showcase our models on problems from the healthcare domain.

\subsection{Solving Classification using {\CEN} with Linear Explanations}
\label{sec:applications-classification}

We start by examining the properties of {\CEN} with linear explanations (Table~\ref{tab:cen-components}) on a few classification tasks.
Our experiments are designed to answer the following questions:
\begin{itemize}[itemsep=0pt,topsep=1ex,parsep=2pt,leftmargin=2em]
    \item[(i)] When explanation is a part of the learning and prediction process, how does that affect performance of the final predictive model \emph{quantitatively}?

    \item[(ii)] \emph{Qualitatively}, what kind of insight can we gain by inspecting explanations?

    \item[(iii)] Finally, we analyze \emph{consistency} of linear explanations generated by {\CEN} versus those generated using {\LIME}~\citep{ribeiro2016trust}, a popular post-hoc method.
\end{itemize}
Details on our experimental setup, all hyperparameters, and training procedures are given in the tables in Appendix~\ref{app:architectures}.

\subsubsection{Poverty Prediction}
\label{sec:applications-classification-poverty-prediction}

We consider the problem of poverty prediction for household clusters in Uganda from satellite imagery and survey data.
Each household cluster is represented by a collection of $400 \times 400$ satellite images (used as the context) and 65 categorical variables from living standards measurement survey (used as the interpretable attributes).
The task is binary classification of the households into being either below or above the poverty line.

We follow the original study of~\citet{jean2016combining} and use a {\VGGF} network (pre-trained on nightlight intensity prediction) to compute 4096-dimensional embeddings of the satellite images on top of which we build contextual models.
Note that this datasets is fairly small (500 training and 142 test points), and so we keep the {\VGGF} part frozen to avoid overfitting.

\begin{wraptable}{r}{0.35\textwidth}
    \centering
    \caption{Performance of the models on the poverty prediction task.}
    \label{tab:satellite-accuracy}
    \fontsize{10}{12}\selectfont
    \renewcommand{\arraystretch}{1.1}
    \begin{tabular}{@{}lrr@{}}
        \toprule
        & \textbf{Acc $\uparrow$} & \textbf{AUC $\uparrow$} \\
        \midrule
        LR$_\texttt{emb}$                               & $62.5\%$ & $68.1\%$ \\
        LR$_\texttt{att}$                               & $75.7\%$ & $82.2\%$ \\
        MLP                                             & $77.4\%$ & $78.7\%$ \\
        \midrule
        MoE$_\texttt{att}$                              & $77.9\%$ & $\mathbf{85.4}\%$ \\
        CEN$_\texttt{att}$                              & $\mathbf{81.5\%}$ & $84.2\%$ \\
        \bottomrule
    \end{tabular}
\end{wraptable}

\vspace{5ex}

\paragraph{Models.}
For baselines, we use logistic regression (LR) and multi-layer perceptrons (MLP) with 1 hidden layer.
The LR uses either {\VGGF} embeddings (LR$_\texttt{emb}$) or the categorical attributes (LR$_\texttt{att}$) as inputs.
The input of the MLP is concatenated {\VGGF} embeddings and categorical attributes.
Context encoder of the {\CEN} model uses {\VGGF} to process images, followed by an attention layer over a dictionary of 16 trainable linear explanations defined over the categorical features (Figure~\ref{fig:explanation-net}).
Finally, we evaluate a mixture-of-experts (MoE) model of the same architecture as {\CEN}, since it is a special case (see Section~\ref{sec:CEN-special-calses}).
Both {\CEN} and MoE are trained with the dictionary constraint and $L_1$ regularization over the dictionary elements to encourage sparse explanations.

\paragraph{Performance.}
The results are presented in Table \ref{tab:satellite-accuracy}.
Both in terms of accuracy and AUC, {\CEN} models outperform both simple logistic regression and vanilla MLP.
Even though the results suggest that categorical features are better predictors of poverty than {\VGGF} embeddings of images, note that using embeddings to \emph{contextualize} linear models reduces the error.
This indicates that \emph{different} linear models are optimal in different contexts.

\begin{figure}[t]
\begin{subfigure}[b]{0.19\textwidth}
    \hspace{-1.5ex}
    \includegraphics[width=1.03\textwidth]{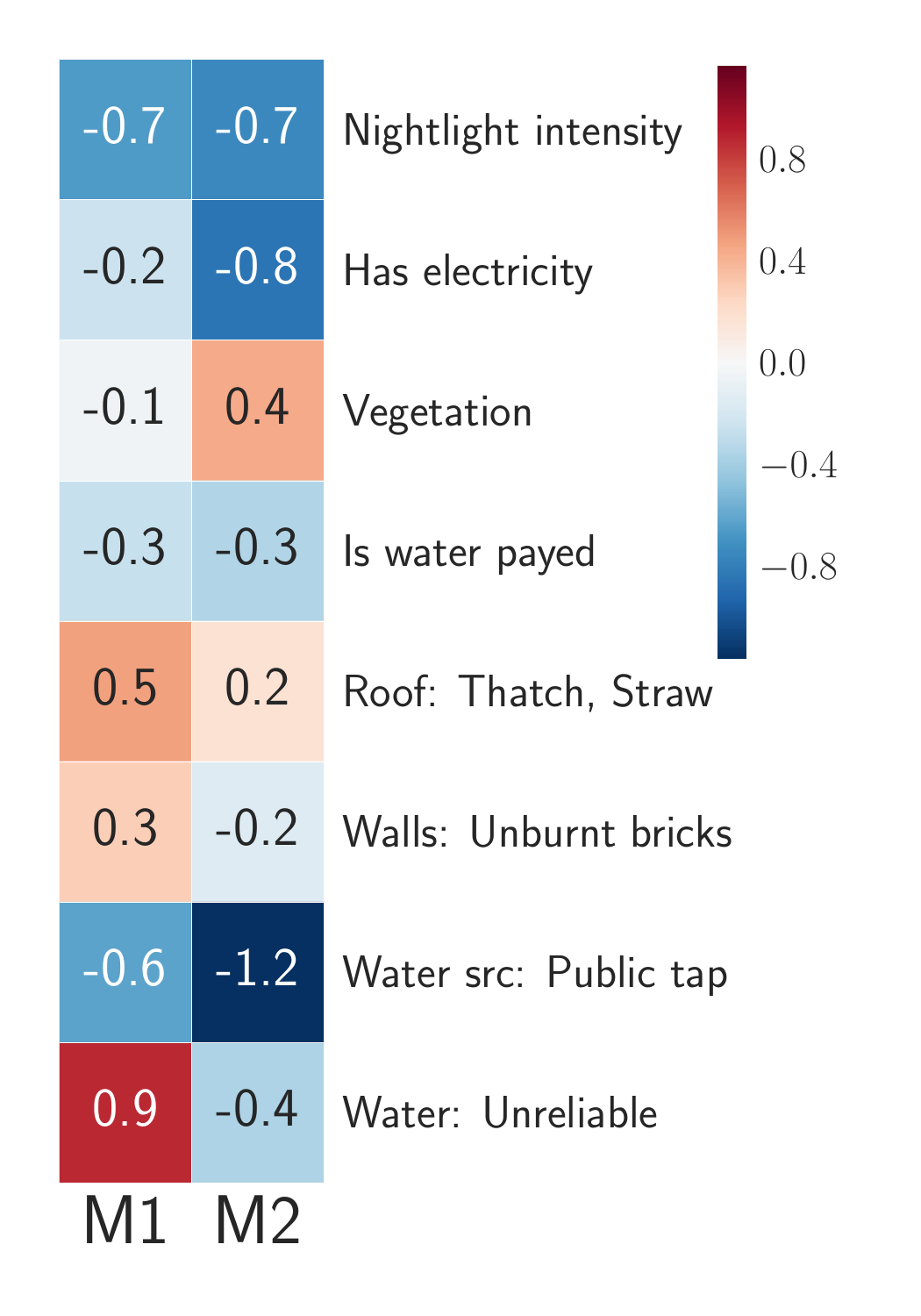}%
    \vspace{-2ex}\caption{}\label{fig:satellite-models}
\end{subfigure}%
\begin{subfigure}[b]{0.19\textwidth}
    \hspace{-0.5ex}
    \includegraphics[width=1.03\textwidth]{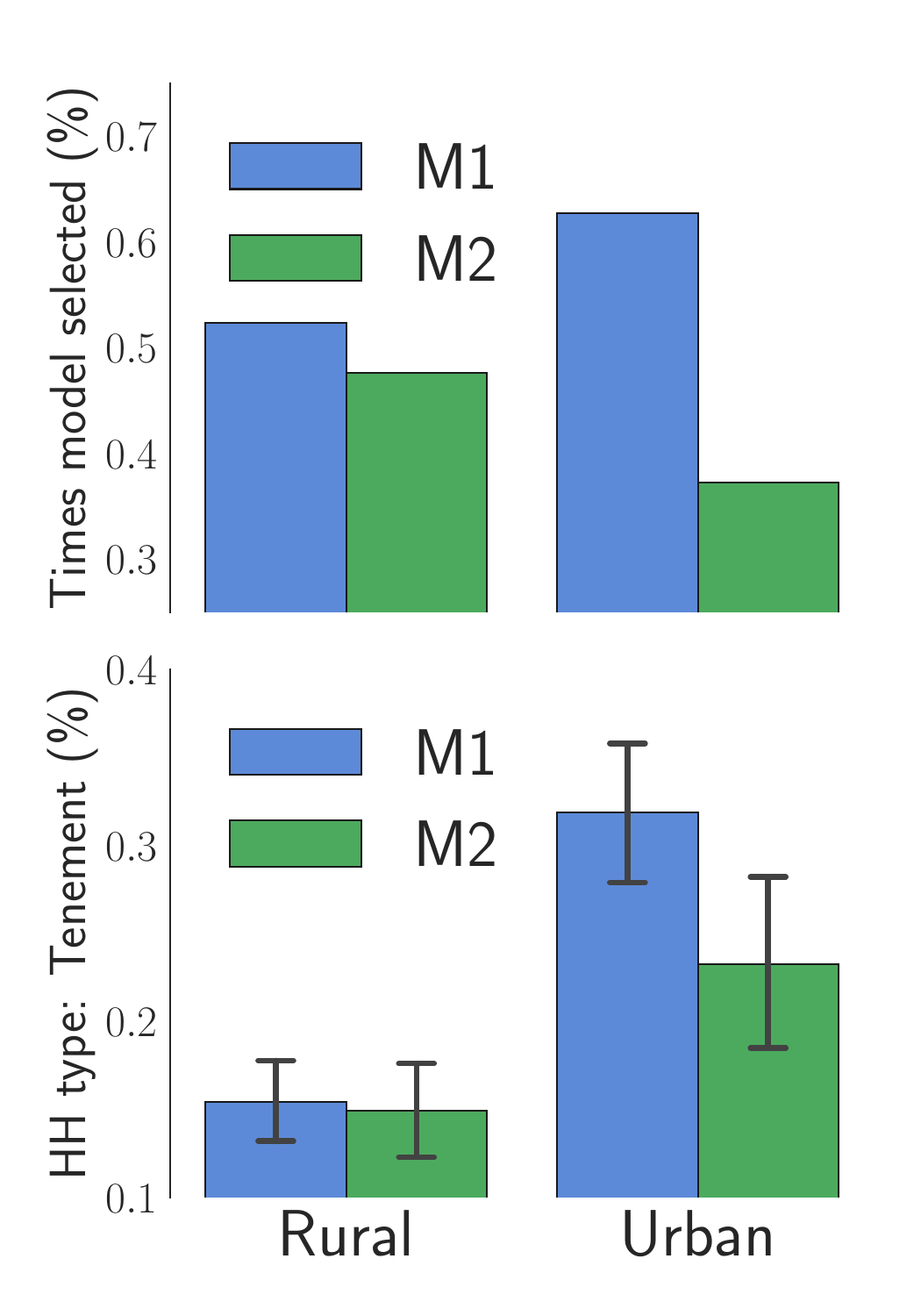}%
    \vspace{-2ex}\caption{}\label{fig:satellite-barplots}
\end{subfigure}%
\begin{subfigure}[b]{0.31\textwidth}
    \hspace{1.5ex}
    \includegraphics[width=1.08\textwidth]{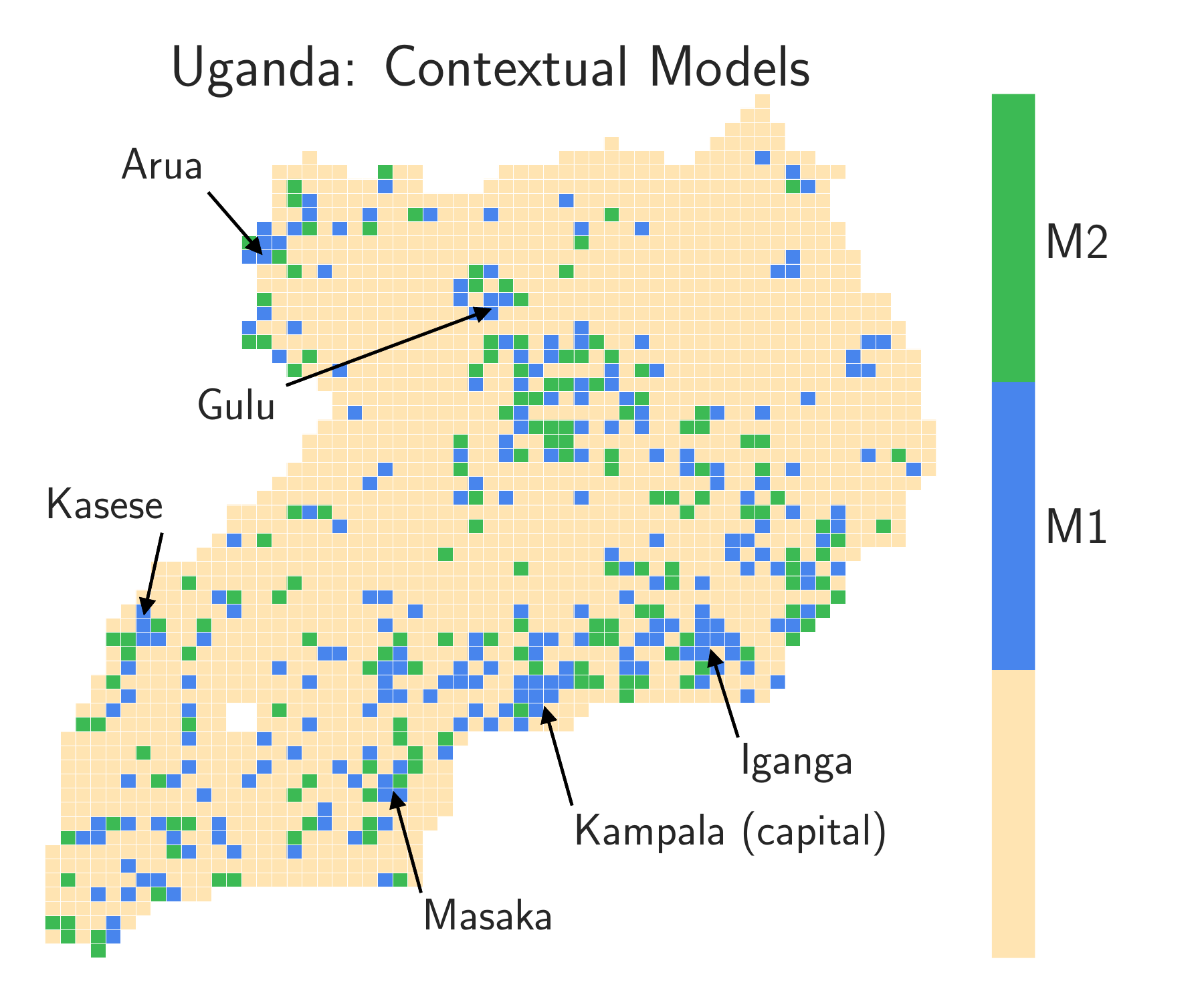}%
    \vspace{-2ex}\caption{}\label{fig:satellite-map-cen}
\end{subfigure}%
\begin{subfigure}[b]{0.31\textwidth}
    \hspace{1.5ex}
    \includegraphics[width=1.08\textwidth]{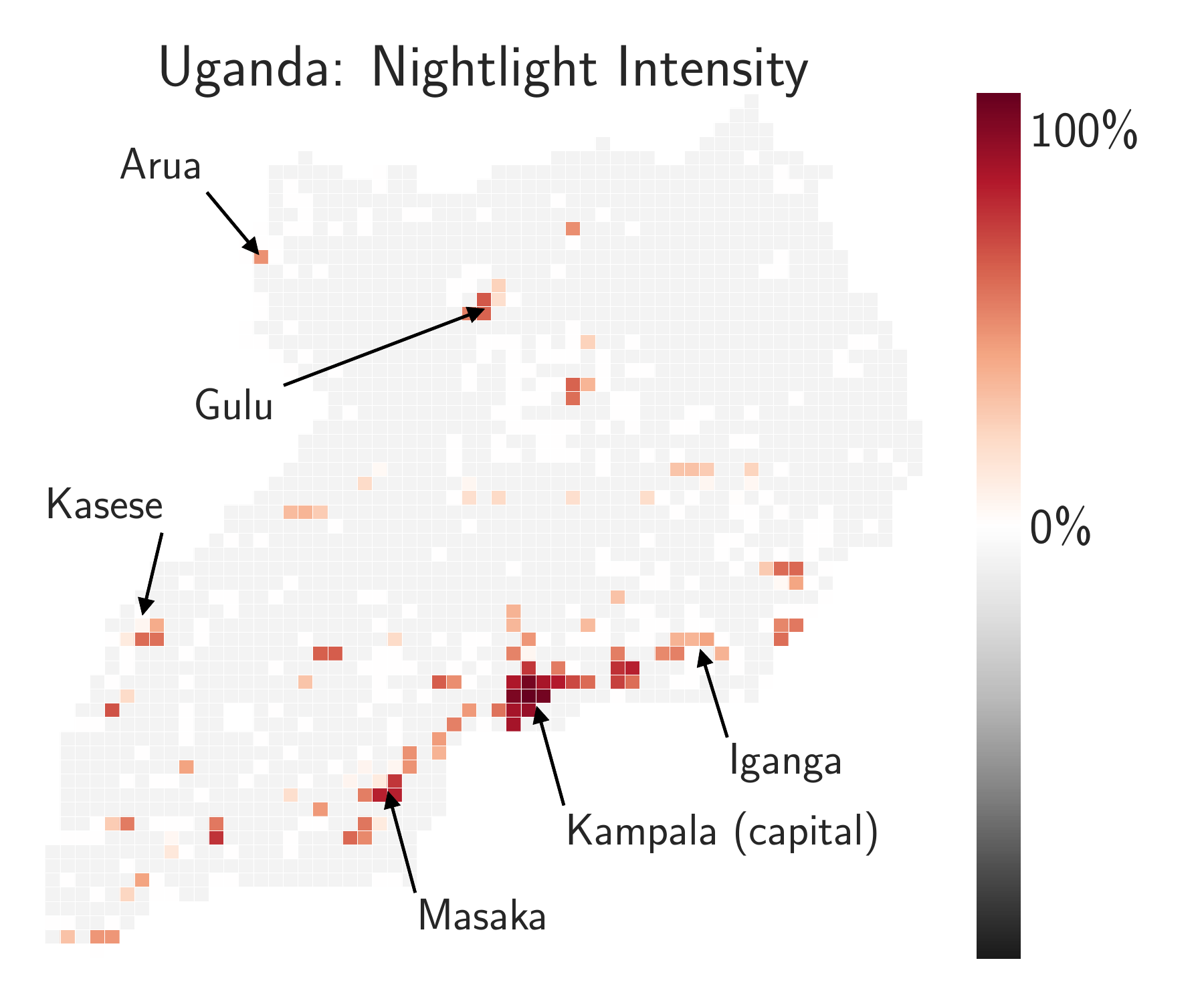}%
    \vspace{-2ex}\caption{}\label{fig:satellite-map-nl}
\end{subfigure}
\caption{%
Qualitative results for the Satellite dataset:
(a) Weights given to a subset of features by the two models (M1 and M2) discovered by CEN.
(b) How frequently M1 and M2 are selected for areas marked rural or urban (top) and the average proportion of Tenement-type households in an urban/rural area for which M1 or M2 was selected.
(c) M1 and M2 models selected for different areas on the Uganda map.
M1 tends to be selected for more urbanized areas while M2 is selected for the rest.
(d) Nightlight intensity of different areas.}
\label{fig:satellite}
\end{figure}

\paragraph{Qualitative analysis.}
We have discovered that, on this task, {\CEN} encoder tends to sharply select one of the two explanations from the dictionary (denoted M1 and M2) for different household clusters in Uganda (Figure~\ref{fig:satellite-models}).
In the survey data, each household cluster is marked as either urban or rural.
Conditional on a satellite image, {\CEN} tends to pick M1 more often for urban areas and M2 for rural (Figure~\ref{fig:satellite-barplots}).
Notice that different explanations weigh categorical features, such as \emph{reliability of the water source} or the \emph{proportion of houses with walls made of unburnt brick}, quite differently.
When visualized on the map, we see that {\CEN} selects M1 more frequently around the major city areas (Figures~\ref{fig:satellite-map-cen}), which also correlates with high nightlight intensity in those areas (Figures~\ref{fig:satellite-map-nl}).

The estimated approximate conditional entropy of the binary targets (poor vs. not poor) given the selected model: $\Hc(\Yv \mid \thetav = \mathrm{M1}) \approx 77\%$ and $\Hc(\Yv \mid \thetav = \mathrm{M2}) \approx 72\%$.
The high performance of {\CEN} along with high conditional entropy makes us confident in the produced explanations (Section~\ref{sec:encoder-regularization}) and allows us to draw conclusions about \emph{what causes the model} to classify certain households in different neighborhoods as poor in terms of interpretable categorical variables.

\subsubsection{Sentiment Analysis}
\label{sec:applications-classification-sentiment-analysis}

The next problem we consider is sentiment prediction of IMDB reviews~\citep{maas2011learning}.
The reviews are given in the form of English text (sequences of words) and the sentiment labels are binary (good/bad movie).
This dataset has 25k labelled reviews used for training and validation, 25k labelled reviews that are held out for test, and 50k unlabelled reviews.

\paragraph{Models.}
Following~\citet{johnson2016supervised}, we use a bi-directional LSTM with max-pooling as our baseline that predicts sentiment directly from text sequences.
The same architecture is used as the context encoder in {\CEN} that produces parameters for linear explanations.
The explanations are applied to either (a) a bag-of-words (BoW) features (with a vocabulary limited to 2,000 most frequent words excluding English stop-words) or (b) a 200-dimensional topic representation produced by a separately trained off-the-shelf topic model~\citep{blei2003latent}.

\begin{table}[t]
\renewcommand{\arraystretch}{1.1}
\centering
\caption{%
Sentiment classification error rate on IMDB dataset.
The standard error ($\pm$) is based on 5 different runs.
It is interesting to note that CENs establishes a new state of the art performance on the supervised prediction task while also outperforming or coming close to many of the semi-supervised methods that used additional 50k unlabeled reviews for pretraining.
All current state of the art methods leverage large-scale pretraining (the bottom section of the table); these results are not directly comparable with methods trained on IMDB data only and included for completeness.}
\label{tab:performance-imdb}
\small
\begin{tabular}{@{}llr@{}}
    \toprule
    \textbf{Reference}                              & \textbf{Method}                           & \textbf{Error $\downarrow$ (\%)}      \\
    \midrule
    \multicolumn{3}{c}{\textbf{Supervised} (trained on 25K labeled reviews only)}                                                       \\
    \midrule
    \citet{Maas:2011:LWV:2002472.2002491}           & Full + BoW (bnc)                          & 11.67                                 \\
    \citet{Dahl:2012:TRB:3042573.3042723}           & WRRBM + BoW (bnc)                         & 10.77                                 \\
    \citet{Wang:2012:BBS:2390665.2390688}           & NBSVM-bi                                  &  8.78                                 \\
    \citet{johnson2014effective}                    & seq2-bow$n$-CNN                           &  7.67                                 \\
    \citet{johnson2015semi}                         & oh-CNN (best)                             &  8.39                                 \\
    \citet{johnson2016supervised}                   & oh-2LSTMp (best)                          &  7.33                                 \\
    \midrule
    Ours                                            & CEN-\texttt{bow}                          & $6.52 \pm 0.15$                       \\
                                                    & CEN-\texttt{tpc}                          & $\boldsymbol{6.24} \pm 0.12$          \\
    \midrule
    \multicolumn{3}{c}{\textbf{Semi-supervised} (trained on 25K labeled + 50K unlabeled only)}                                          \\
    \midrule
    \citet{Maas:2011:LWV:2002472.2002491}           & Full + Unlabeled + BoW                    & 11.11                                 \\
    \citet{le2014distributed}                       & Paragraph vectors                         &  7.42                                 \\
    \citet{dai2015semi}                             & wv-LSTM                                   &  7.24                                 \\
    \citet{johnson2015semi}                         & oh-CNN                                    &  6.51                                 \\
    \citet{johnson2016supervised}                   & oh-2LSTMp                                 &  5.94                  \\
    \citet{Adji2017TopicRNN}                        & TopicRNN                                  &  6.28                                 \\
    \citet{miyato2016adversarial}                   & Virtual adversarial                       &  5.94                  \\
    \midrule
    Ours                                            & CEN-\texttt{bow}                          & ---                                   \\
                                                    & CEN-\texttt{tpc}                          & $\boldsymbol{5.48} \pm 0.09$          \\
    \midrule
    \multicolumn{3}{c}{\textbf{Semi-supervised via large-scale pre-training} (massive external data)}                                   \\
    \midrule
    \citet{gray2017gpu}                             & block-sparse LSTM                         &  5.01                                 \\
    \citet{howard2018universal}                     & ULMFiT                                    &  4.60                                 \\
    \citet{sachan2019revisiting}                    & Mixed-objective LSTM                      &  4.32                                 \\
    \citet{xie2019unsupervised}                     & BERT-large                                &  4.20                                 \\
    \citet{haonan2019graph}                         & Graph Star                                &  4.00                                 \\
    \bottomrule
\end{tabular}
\end{table}

\paragraph{Performance.}
Table~\ref{tab:performance-imdb} compares {\CEN} with other models from the literature.
Not only {\CEN} achieves the state-of-the-art accuracy on this dataset in the supervised setting, it also outperforms or comes close to many of the semi-supervised methods.
This indicates that the inductive biases provided by the {\CEN} architecture lead to a more significant performance improvement than most of the semi-supervised training methods on this dataset.
We also remark that classifiers derived from large-scale language models pretrained on massive unsupervised corpora~\citep[\eg,][]{gray2017gpu, howard2018universal, xie2019unsupervised} have become popular and now dominate the leaderboard for this task.

\paragraph{Qualitative analysis.}
After training {\CEN}-\texttt{tpc} with linear explanations in terms of topics on the IMDB dataset, we generate explanations for each test example and visualize histograms of the weights assigned by the explanations to the 6 selected topics in Figure~\ref{fig:imdb-dict-hist}.
The 3 topics in the top row are acting- and plot-related (and intuitively have positive, negative, or neutral connotation), while the 3 topics in the bottom are related to particular genre of the movies.
Note that acting-related topics turn out to be bimodal, \ie, contributing either positively, negatively, or neutrally to the sentiment prediction in different contexts.
{\CEN} assigns a high negative weight to the topic related to ``bad acting/plot'' and a high positive weight to ``great story/performance'' in most of the contexts (and treats those neutrally conditional on some of the reviews).
Interestingly, genre-related topics almost always have a negligible contribution to the sentiment which indicates that the learned model does not have any particular bias towards or against a given genre.

\begin{figure}[t]
    \centering
    \includegraphics[width=0.9\textwidth]{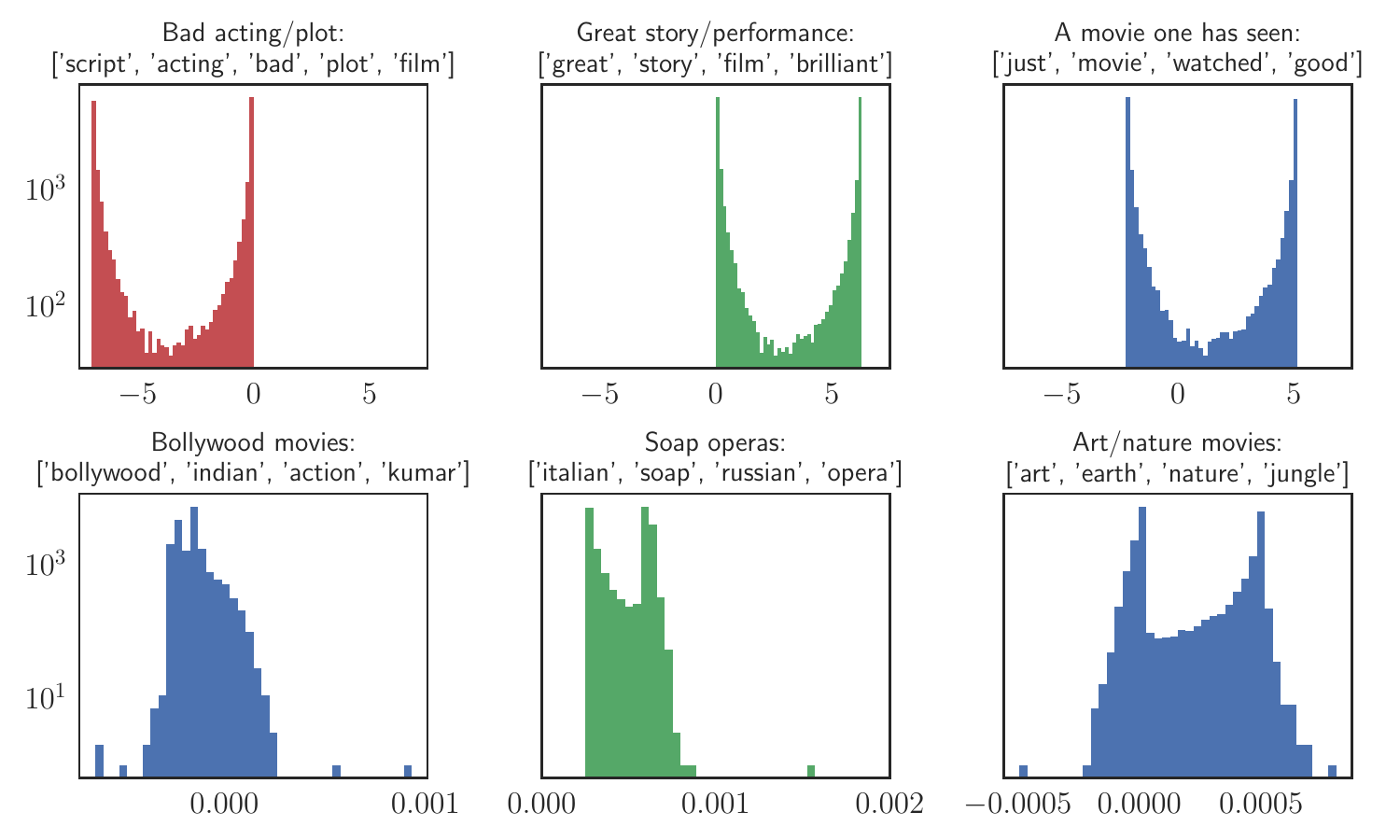}
    \caption{Histograms of test weights assigned by CEN to 6 topics: acting- and plot-related topics (upper charts), genre topics (bottom charts).}
    \label{fig:imdb-dict-hist}
\end{figure}

\subsubsection{Image Classification}
\label{sec:applications-classification-image-classification}

For the purpose of completeness, we also provide results on two classical image datasets: MNIST and CIFAR-10.
For {\CEN}, full images are used as the context; to imitate high-level features, we use (a) the original images cubically downscaled to $20 \times 20$ pixels, gray-scaled and normalized, and (b) HOG descriptors computed using $3 \times 3$ blocks~\citep{dalal2005histograms}.
For each task, we use linear regression and vanilla convolutional networks as baselines (a small convnet for MNIST and {\VGG} for CIFAR-10).
The results are reported in Table~\ref{tab:img_performance}.
{\CENs} are competitive with the baselines and do not exhibit deterioration in performance.
Visualization and analysis of the learned explanations is given in Appendix~\ref{app:qualitative-analysis} and the details on the architectures, hyperparameters, and training are given in Appendix~\ref{app:architectures}

\begin{table}[H]
\centering
\caption{Prediction error of the models on image classification tasks (averaged over 5 runs; the std. are on the order of the least significant digit).
The subscripts denote the features on which the linear models are built: pixels~(\texttt{pxl}), HOG~(\texttt{hog}).}
\label{tab:img_performance}
\vspace{-1ex}
\fontsize{8}{10}\selectfont
\def\arraystretch{1.2}
\setlength\tabcolsep{2.5pt}
\begin{tabular}[t]{@{}rrr rrrr rr rrr rrrr@{}}
\toprule
    \multicolumn{7}{c}{\textbf{MNIST (Error $\downarrow$, \%)}} & & &
    \multicolumn{7}{c}{\textbf{CIFAR10 (Error $\downarrow$, \%)}} \\
    \cmidrule{1-7} \cmidrule{10-16}
    LR$_\texttt{pxl}$ & LR$_\texttt{hog}$ & CNN & MoE$_\texttt{pxl}$  & MoE$_\texttt{hog}$ & CEN$_\texttt{pxl}$ & CEN$_\texttt{hog}$ & & & LR$_\texttt{pxl}$ &  LR$_\texttt{hog}$ & VGG & MoE$_\texttt{pxl}$ & MoE$_\texttt{hog}$ &  CEN$_\texttt{pxl}$ & CEN$_\texttt{hog}$ \\
\cmidrule{1-7} \cmidrule{10-16}
$8.00$ & $2.98$  & $\mathbf{0.75}$ & $1.23$ & $1.10$ & $\mathbf{0.76}$ & $\mathbf{0.73}$  & & & $60.1$ & $48.6$ & $9.4$ & $13.0$ & $11.7$ & $9.6$ & $\mathbf{9.2}$ \\
\bottomrule
\end{tabular}
\vspace{-1ex}
\end{table}

\subsection{Properties of Explanations}
\label{sec:applications-properties}

In this section, we look at the explanations from the regularization and consistency point of view.
As we show next, prediction via explanation not only has a strong regularization effect, but also always produces consistent locally linear models.
Additionally, we analyze the effect of entropy regularization, quantify how much CEN's performance relies on explanations, and discuss computational considerations and tradeoffs for CEN and LIME.

\begin{figure}[t]
    \centering
    \begin{subfigure}[b]{0.48\textwidth}
        \centering
        \includegraphics[width=0.99\textwidth]{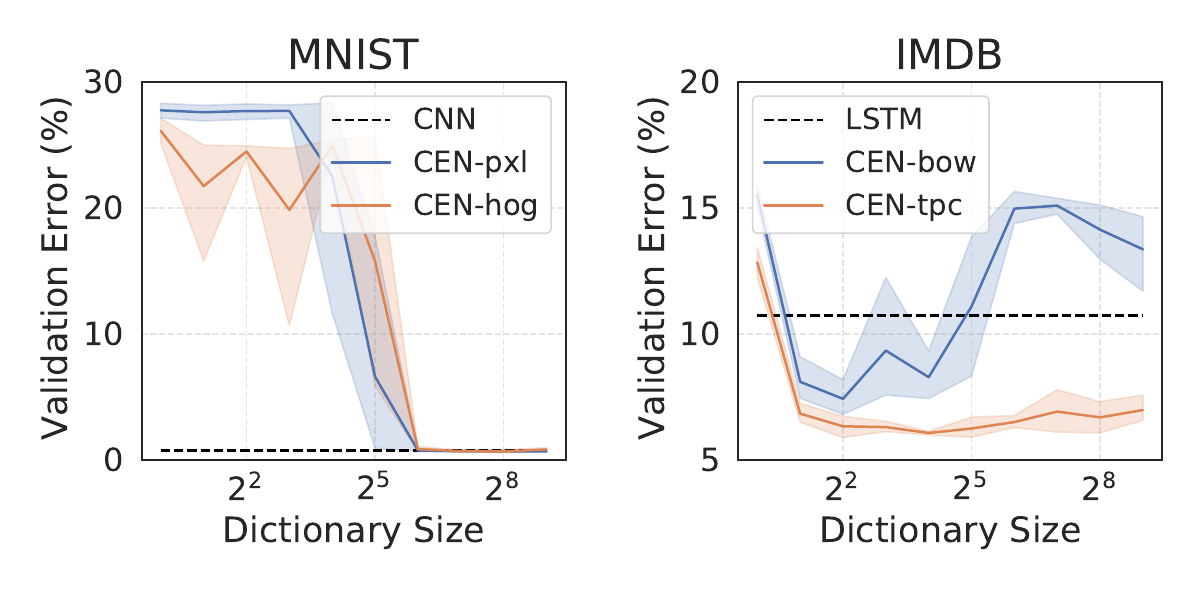}
        \vspace{-16pt}\caption{Validation error vs. dictionary size.}
        \label{fig:dict-size}
    \end{subfigure}
    \begin{subfigure}[b]{0.48\textwidth}
        \centering
        \includegraphics[width=0.96\textwidth]{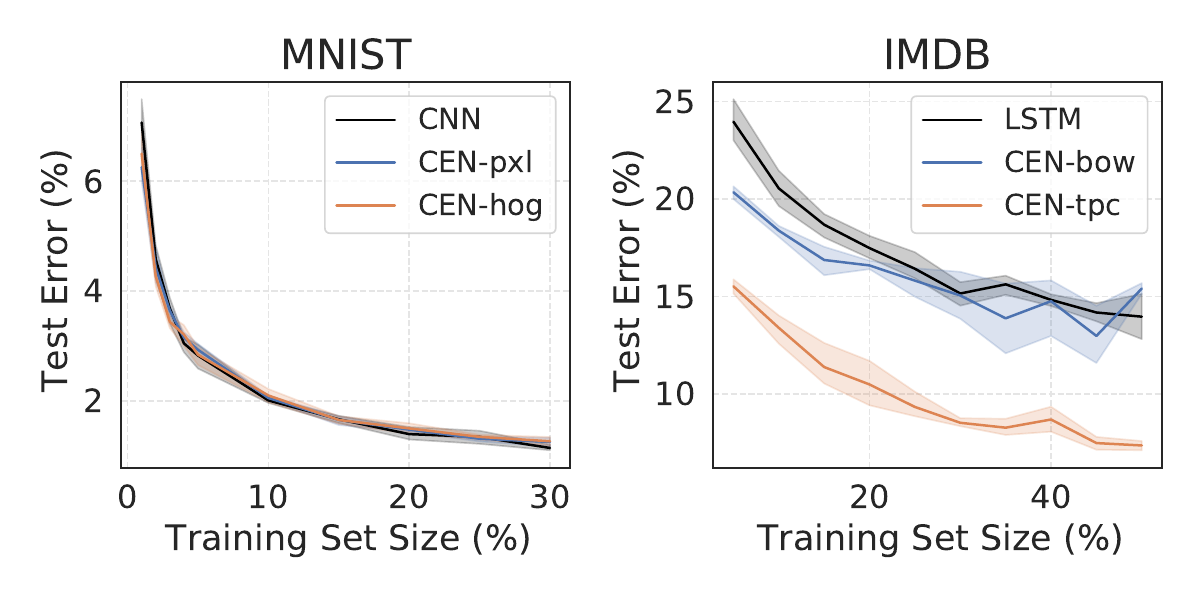}
        \caption{Test error vs. training data size.}
        \label{fig:sample-complexity}
    \end{subfigure}
    \caption{Analysis of the behavior of different {\CEN} models with different dictionary sizes (varied between 1 and 512), feature types, trained on full or on a subset of the data.
    Shaded regions denote 95\% CI based on 5 runs with different random seeds.
    (a) {\CEN} is sensitive to the size of the dictionary---there is a critical size such that models with explanation dictionaries smaller than that tend to significantly underperform.
    (b) Sample complexity of CENs.
    Models are trained with early stopping based on validation performance.}
\end{figure}

\subsubsection{Explanations as a Regularizer}
By controlling the dictionary size, we can control the expressivity of the model class specified by {\CEN}.
For example, when the dictionary size is 1, {\CEN} becomes equivalent to a linear model.\footnote{Note that {\CENs} with the dictionary size of 1 is still trained using stochastic optimization method as a neural network, which tends to yield a somewhat worse performance than the vanilla logistic regression.}
For larger dictionaries, {\CEN} becomes as flexible as a deep network (Figure~\ref{fig:dict-size}).
Adding a small sparsity penalty to each element of the dictionary (between $10^{-6}$ and $10^{-3}$, see Appendix~\ref{app:architectures}) helps to avoid overfitting for very large dictionary sizes, so that the model learns to use only a few dictionary atoms for prediction while shrinking the rest to zero.
Generally, dictionary size is a hyperparameter which optimal value depends on the data and the type of the interpretable features (\cf, \CEN-\texttt{bow} and \CEN-\texttt{tpc} on Figure~\ref{fig:dict-size}).

If explanations can act as a proper regularizer, we must observe improved sample efficiency of the model.
To verify this, we trained \CEN models on subsets of the data (size varied between 1\% and 30\% for MNIST and 2\% and 50\% for IMDB) with early stopping based on the validation performance.
The test error on MNIST and IMDB for different training set sizes is presented on Figure~\ref{fig:sample-complexity}.
On the IMBD dataset, \CEN-\texttt{tpc} required an order of magnitude fewer samples to match the baseline's performance, indicating efficient use of explanations for prediction.
Note that such drastic sample efficiency gains were observed on IMDB only for \CEN-\texttt{tpc} (\ie, when using topics as interpretable features); gains for \CEN-\texttt{bow} were noticeable but moderate; no sample efficiency gains were observed on MNIST for any of our \CEN models.

\subsubsection{Quantifying Contribution of the Explanations}
\label{sec:cen-explanation-contributions}

\begin{figure}[t]
    \centering
    \begin{subfigure}[b]{0.48\textwidth}
        \centering
        \includegraphics[width=0.99\textwidth]{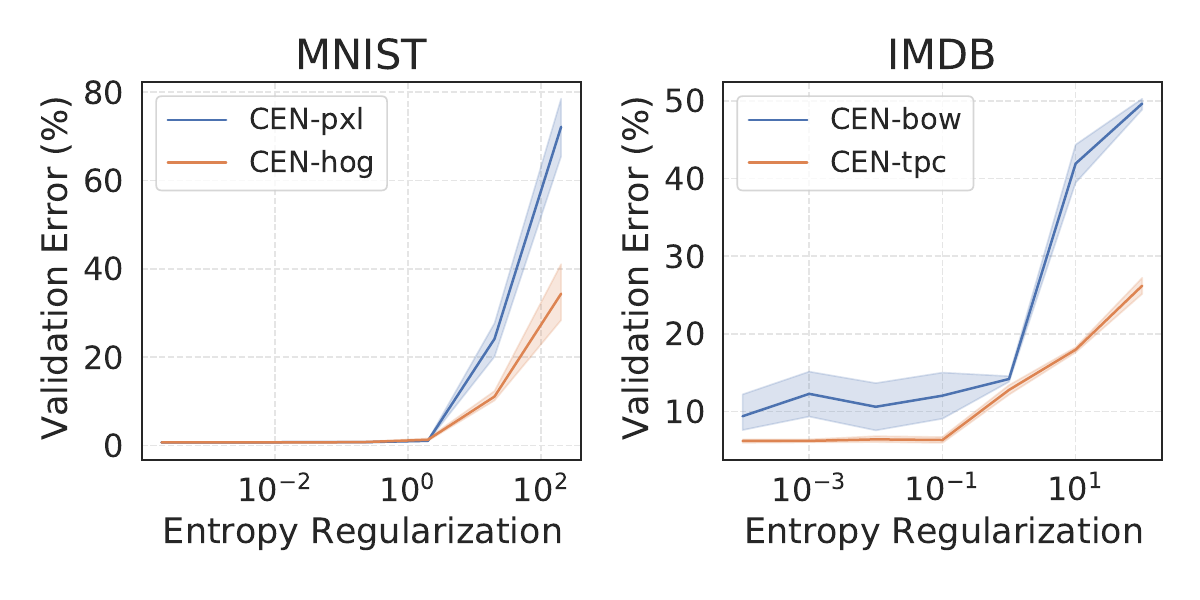}
        \caption{Validation error vs. entropy regularization.}
        \label{fig:err-vs-ent-reg}
    \end{subfigure}%
    \quad
    \begin{subfigure}[b]{0.48\textwidth}
        \centering
        \includegraphics[width=0.99\textwidth]{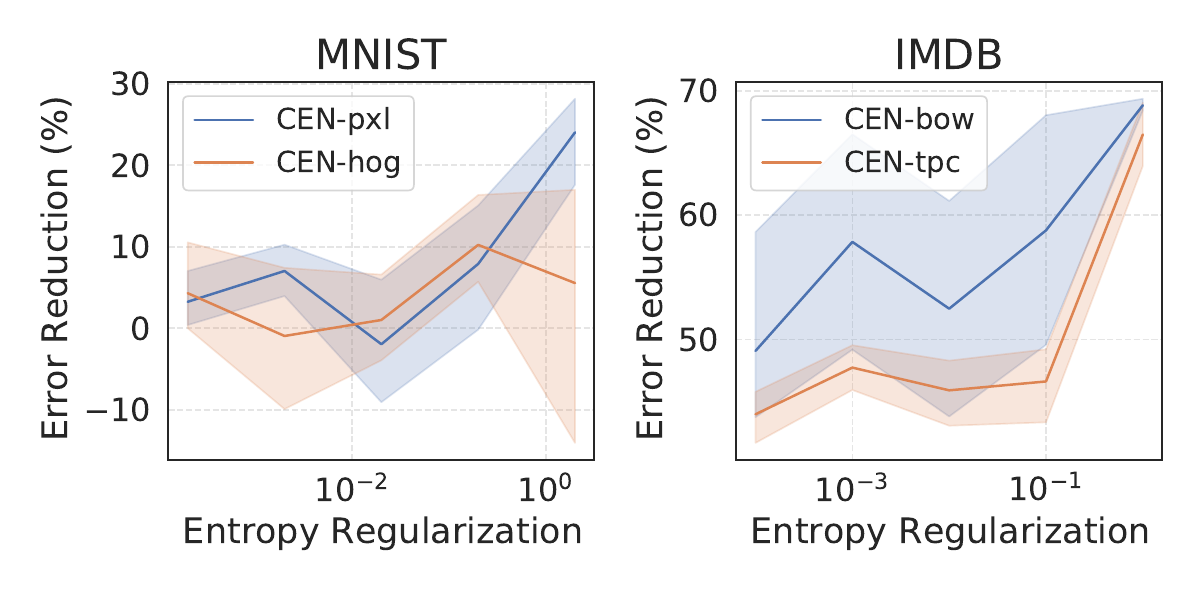}
        \caption{Expected contribution of the explanations.}
        \label{fig:exp-contrib-vs-ent-reg}
    \end{subfigure}
    \caption{The effects of entropy regularization on (a) the predictive performance of a \CEN model and (b) the lower bound on the contribution of the explanations to the relative predictive error reduction.
    Shaded regions are 95\% CI based on 5 runs with different random seeds.}
    \label{fig:ent-reg-effects}
\end{figure}

Even though improved sample efficiency and regularizing effects of explanations indicate their non-trivial contribution indirectly, we wish to further quantify such contribution of explanations to the predictive performance of \CEN.
To do so, we run a set of experiments where we vary conditional entropy regularization coefficient and measure (a) performance of \CEN on the validation set and (b) expected lower bound on the relative reduction of predictive error due to explanations, defined as $\left[\prob{\hat\Yv \neq \Yv \mid \cv} - \prob{\hat\Yv \neq \Yv \mid \xv, \cv} \right] \big{/}\ \prob{\hat\Yv \neq \Yv \mid \cv}$.

As we have shown in Section~\ref{sec:encoder-regularization}, conditional entropy regularization encourages \CEN models to learn context representations that are minimally correlated with the targets, and hence makes the model rely on the explanations rather than contextual information only.
Figure~\ref{fig:err-vs-ent-reg} shows that entropy regularization generally does not affect predictive performance of a \CEN model, unless the regularization coefficient becomes too large (\eg, an order of magnitude larger than the predictive cross-entropy loss).
Increasing conditional entropy regularization leads to \CEN models whose performance relies more on explanations (Figure~\ref{fig:exp-contrib-vs-ent-reg}).
However, note that even without entropy regularization, explanations have a significant relative contribution to the reduction of the predictive error of \CEN, ranging between 10-20\% on MNIST and 40-60\% on IMDB.
This indicates that, while conditional entropy regularization is beneficial, even without it \CEN still learns to generate meaningful, non-spurious explanations.

\subsubsection{Consistency of Explanations}
\label{sec:cen-consistency}

While regularization is a useful aspect, the main use case for explanations is model diagnostics.
Linear explanations assign weights to the interpretable features, $\Xv$, and thus the quality of explanations depends on the quality of the selected features.
In this section, we evaluate explanations generated by {\CEN} and LIME (a post-hoc method).
In particular, we consider two cases: (a) the features are corrupted with additive noise, and (b) the selected features are incomplete.
For analysis, we use MNIST and IMDB datasets.
Our key question is:

\vspace{0.4ex}
\begin{blockquote}
    \centering
    \emph{Can we trust the explanations built on noisy or incomplete features?}
\end{blockquote}

\begin{figure}[h]
    \centering
    \begin{subfigure}[t]{0.48\textwidth}
        \includegraphics[width=\textwidth]{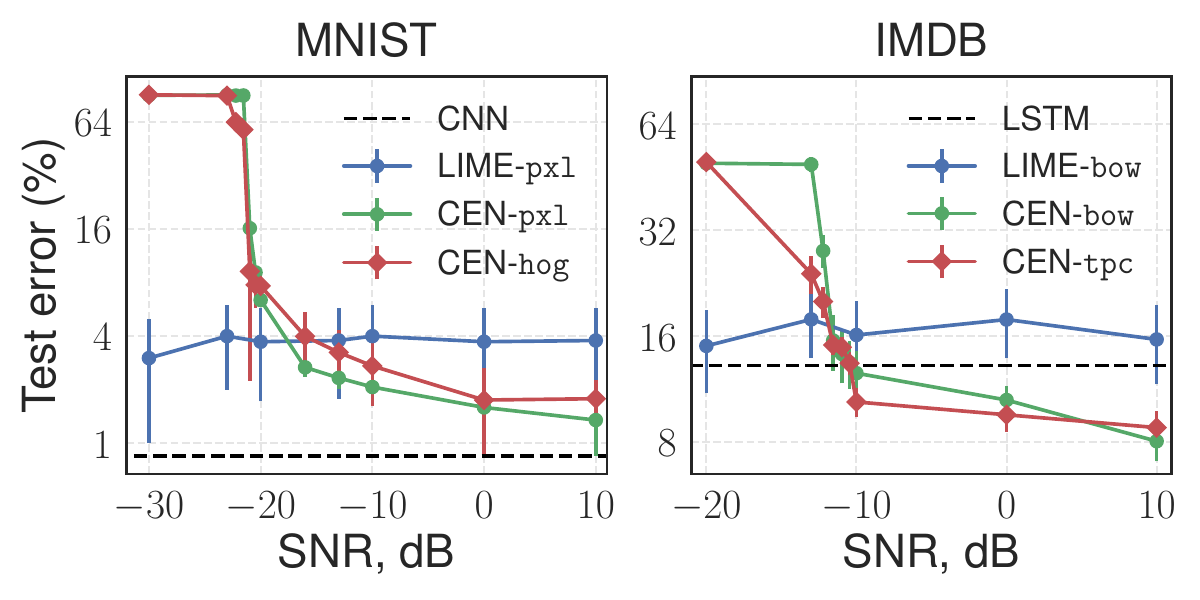}%
        \caption{Explanation test error vs. feature noise.}\label{fig:signal-to-noise}
    \end{subfigure}
    \quad
    \begin{subfigure}[t]{0.48\textwidth}
        \includegraphics[width=\textwidth]{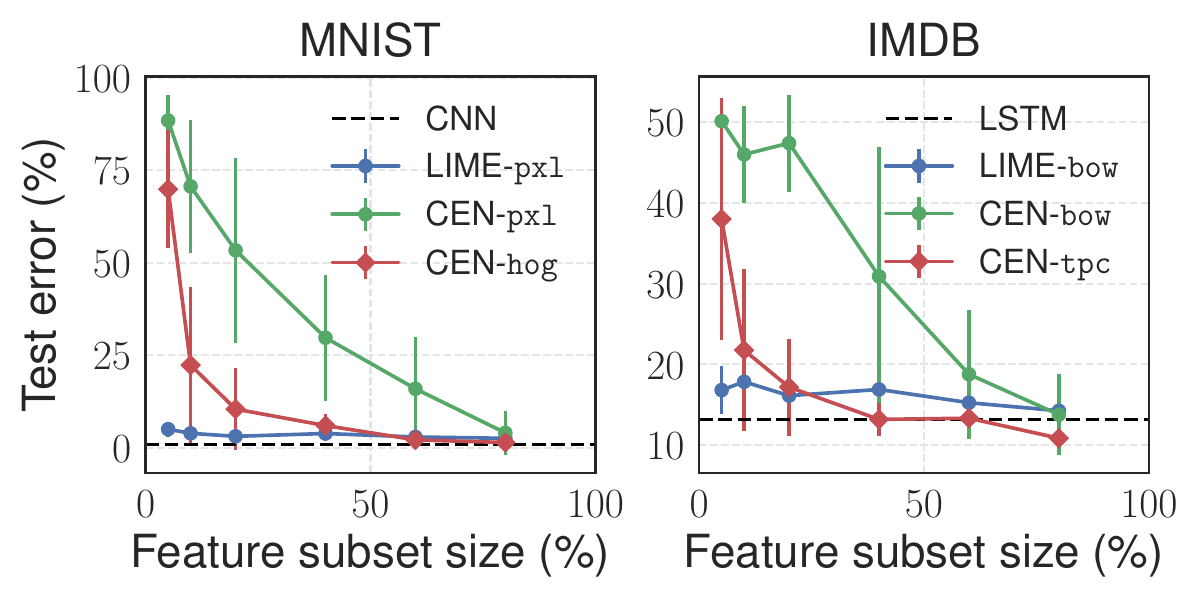}%
        \caption{Explanation test error vs. feature size.}\label{fig:incomplete-features}
    \end{subfigure}
    \caption{%
    The effect of feature quality on explanations.
    (a) Explanation test error vs. the level of the noise added to the interpretable features.
    (b) Explanation test error vs. the total number of interpretable features.
    Error bars indicate 95\% CI.}
\end{figure}

\paragraph{The effect of noisy features.}
In this experiment, we inject noise\footnote{We use Gaussian noise with zero mean and select variance for each signal-to-noise ratio level appropriately.} into the features $\Xv$ and ask {\LIME} and {\CEN} to fit explanations to the corrupted features.
Note that after injecting noise, each data point has a noiseless representation $\Cv$ and a noisy $\Xv$.
{\LIME} constructs explanations by approximating the decision boundary of the baseline model trained to predict $\Yv$ from $\Cv$ features only.
{\CEN} is trained to construct explanations given $\Cv$ and then make predictions by applying explanations to $\Xv$.
The predictive performance of the produced explanations on noisy features is given on Figure~\ref{fig:signal-to-noise}.
Since baselines take only $\Cv$ as inputs, their performance stays the same (dashed line)
Regardless of the noise level, {\LIME} ``successfully'' overfits explanations---it is able to almost perfectly approximate the decision boundary of the baselines essentially using pure noise.
On the other hand, performance of {\CEN} degenerates with the increasing noise level indicating that the model fails to learn when the selected interpretable representation is of very low quality.

\paragraph{The effect of feature selection.}
Using the same setup, instead of injecting noise into $\Xv$, we construct $\Xv$ by randomly subsampling a set of dimensions.\footnote{Subsampling dimensions from $\Xv$ is done to resemble human subjectivity in selecting semantically meaningful features for model interpretation.}
Figure~\ref{fig:incomplete-features} demonstrates that while performance of {\CENs} degrades proportionally to the size of $\Xv$ (\ie, less informative features imply worse performance for \CEN), we see that, again, {\LIME} is again able to perfectly fit explanations to the decision boundary of the original models, despite the loss of information in the interpretable features $\Xv$.

These two experiments indicate a major drawback of explaining predictions post-hoc: when constructed on poor, noisy, or incomplete features, such explanations can overfit an arbitrary decision boundary of a predictor and are likely to be meaningless or misleading.
For example, predictions of a perfectly valid model might end up getting absurd explanations which is unacceptable from the decision support point of view.\footnote{Similar behavior has been observed in recent work that studied post-hoc explanation systems in adversarial settings~\citep{dombrowski2019explanations, lakkaraju2019fool}.}
On the other hand, if we use {\CEN} to generate explanations, high predictive performance would indicate presence of a meaningful signal the selected interpretable features and explanations.

\subsubsection{Computational Overhead and Considerations}
\label{sec:computational-considerations}

\begin{wraptable}{r}{5.2cm}
\vspace{-1.5ex}
\caption{Compute time overhead.}
\vspace{-1ex}
\centering
\scriptsize
\def\arraystretch{1.2}
\begin{tabular}[t]{@{}l|r|r@{}}
    \toprule
    \textbf{Dataset} & \textbf{CEN} & \textbf{LIME} \\
    \midrule
    \multicolumn{3}{c}{Training time overhead} \\
    \midrule
    MNIST & $18.6 \pm 1.7 \%$ & --- \\
    IMDB & $1.8 \pm 0.5 \%$ & --- \\
    Satellite & $0.4 \pm 0.1 \%$ & --- \\
    \midrule
    \multicolumn{3}{c}{Explanation time per data point} \\
    \midrule
    MNIST & $0.05 \pm 0.03$ ms  & $77 \pm 9$ ms \\
    IMDB & $0.07 \pm 0.03$ ms  & $38 \pm 5$ ms \\
    Satellite & $0.01 \pm 0.01$ ms & $22 \pm 6$ ms \\
    \bottomrule
\end{tabular}
\label{tab:cen-vs-lime}
\end{wraptable}

Given all the advantages of \CEN, such as often improved performance and consistency of linear explanations, what is the added computational overhead?
It turns out that \CEN compares quite favorably against the typical bundle solution: \emph{a vanilla deep network plus a post-hoc explanation system (\eg, LIME)}.
The \CEN architecture essentially adds a single bi-linear layer to the top of a network, resulting in a mild overhead of $O(D \times |\Xc|)$ multiplication and addition operations during the forward pass through the model.
The training time overhead in aggregate does not exceed 20\% when compared to a vanilla deep network of the same architecture (Table~\ref{tab:cen-vs-lime}).
Note that the models we used in our experiments are tiny by the modern standards, and we expect \CEN's relative compute overhead to be even smaller for modern large-scale architectures.
Also note that \CENs generate explanations more than three orders of magnitude faster than \LIME, manly because the latter has to solve an optimization problem for each instance of interest to obtain an explanation.%

\subsection{Solving Survival Analysis using {\CEN} with Structured Explanations}
\label{sec:applications-survival-analysis}

In this final case study, we design {\CENs} with structured explanations for survival prediction.
We provide some general background on survival analysis and the structured prediction approach proposed by \citet{lin2011learning}, then introduce {\CENs} with linear CRF-based explanations for survival analysis, and conclude with experimental results on two public datasets from the healthcare domain.

\subsubsection{Background on Survival Analysis via Structured Prediction}
\label{sec:survival-analysis-background}

In survival time prediction, our goal is to estimate the risk and occurrence time of an undesirable event in the future (\eg, death of a patient, earthquake, hard drive failure, customer turnover, etc.).
A common approach is to model the \emph{survival time}, $T$, either for a population (i.e., average survival time) or for each instance.
Classical approaches, such as Aalen additive hazard~\citep{aalen} and Cox proportional hazard~\citep{cox} models, view survival analysis as continuous time prediction and hence a regression problem.

Alternatively, the time can be discretized into intervals (\eg, days, weeks, etc.), and the survival time prediction can be converted into a multi-task classification problem \citep{efron1988logistic}.
Taking this approach one step further, \citet{lin2011learning} noticed that the output space of such a multitask classifier is structured in a particular way, and proposed a model called \emph{sequence of dependent regressors}.
The model is essentially a CRF with a particular structure of the pairwise potentials between the labels.
We introduce the setup in our notation below.

Let the data instances be represented by tuples $(\cv, \xv, \yv)$, where targets are now sequences of $m$ binary variables, $\yv := (y^1, \dots, y^m)$, that indicate occurrence of an event at the corresponding time intervals.\footnote{We assume that the occurrence time is lower bounded by $t_0 = 0$, upper bounded by some $t_m = T$, and discretized into intervals $[t_{i}, t_{i+1})$, where $i \in \{0, \dots, m-1\}$.}
If the event occurred at time $t \in [t_{i}, t_{i+1})$, then $y^j = 0,\, \forall j \leq i$ and $y^k = 1,\, \forall k > i$.
If the event was \emph{censored} (\ie, we lack information for times after $t$), we represent targets $(y^{i+1}, \dots, y^m)$ with latent variables.
Importantly, only $m + 1$ sequences are valid under these conditions, \ie, assigned non-zero probability by the model.
This suggests a linear CRF model defined as follows:
\begin{equation}
    \prob{\Yv = (y^1, y^2, \dots, y^m) \mid \xv, \thetav^{1:m}} \propto
    \exp\left\{\sum_{t=1}^m y^i (\xv^\top \thetav^t) + \omega(y^t, y^{t + 1})\right\}
\end{equation}
The potentials between  $\xv$ and $y^{1:m}$ are linear functions parameterized by $\thetav^{1:m}$.
The pairwise potentials between targets, $\omega(y_i, y_{i + 1})$, ensure that non-permissible configurations where $(y_i = 1, y_{i+1} = 0)$ for some $i \in \{0, \dots, m-1\}$ are improbable (\ie, $\omega(1, 0) = -\infty$ and $\omega(0, 0) = \omega_{00}$, $\omega(0, 1) = \omega_{01}$, $\omega(1, 1) = \omega_{10}$ are learnable parameters).

\begin{figure}[t]
    \centering
    \begin{subfigure}[b]{0.49\textwidth}
        \centering
        \begin{tikzpicture}[
            inner/.style={draw, fill=yellow!20, thin, inner sep=2pt},
        ]
            \tikzstyle{latent} = [circle,fill=white,draw=black,inner sep=1pt,
            minimum size=17pt, font=\fontsize{9}{9}\selectfont, node distance=1]

            \node (c) [obs] {$\cv$};
            \node (x_) [obs, transparent, right=10pt of c] {};

            \node (h4) [det, right=10pt of c] {$\hv^1$};
            \node (h5) [det, right=15pt of h4] {$\hv^2$};
            \node (h6) [det, right=15pt of h5] {$\hv^3$};

            \edge[bend right=45] {c} {h4};
            \edge[bend right=45] {c} {h5};
            \edge[bend right=45] {c} {h6};

            \edge {h4} {h5};
            \edge {h5} {h6};

            \node (x1_) [obs, transparent, above=10pt of h4] {};
            \node (x3_) [obs, transparent, above=10pt of h6] {};
            \node (y3_) [obs, transparent, above=17pt of x3_] {};
            \plate[inner, inner sep=2pt] {CEN} {(x1_)(x3_)(y3_)} {};

            \node (x1) [obs, above=10pt of h4] {$\xv^1$};
            \node (x2) [obs, above=10pt of h5] {$\xv^2$};
            \node (x3) [obs, above=10pt of h6] {$\xv^3$};

            \node (y1) [obs, above=17pt of x1] {$y^1$};
            \node (y2) [obs, above=17pt of x2] {$y^2$};
            \node (y3) [latent, above=17pt of x3] {$y^3$};

            \factor[above=3pt of x1]{x1-y1} {} {} {};
            \node (theta1) [const, right=9pt of x1-y1] {$\thetav^1$};
            \factoredge {x1} {x1-y1} {y1};
            \factor[above=3pt of x2]{x2-y2} {} {} {};
            \node (theta1) [const, right=9pt of x2-y2] {$\thetav^2$};
            \factoredge {x2} {x2-y2} {y2};
            \factor[above=3pt of x3]{x3-y3} {} {} {};
            \node (theta1) [const, right=9pt of x3-y3] {$\thetav^3$};
            \factoredge {x3} {x3-y3} {y3};

            \factor[right=7pt of y1]{y1-y2} {} {} {};
            \factoredge[-] {y1} {y1-y2} {y2};
            \factor[right=7pt of y2]{y2-y3} {} {} {};
            \factoredge[-] {y2} {y2-y3} {y3};

            \edge[bend right=60, dashed] {h4} {x1-y1};
            \edge[bend right=60, dashed] {h5} {x2-y2};
            \edge[bend right=60, dashed] {h6} {x3-y3};

            \draw[dotted, thick] (3.0, 2.75) -- (3.0, 1.75);
            \node[anchor=west, font=\fontsize{7}{8}\selectfont, inner sep=0pt] at (3.05, 2.95) {$t \in [t_2, t_3)$};
        \end{tikzpicture}
        \vspace{-2ex}
        \caption{Architecture used for SUPPORT2.}
        \label{fig:MLP-CRF}
    \end{subfigure}
    \begin{subfigure}[b]{0.49\textwidth}
        \centering
        \begin{tikzpicture}[
            inner/.style={draw, fill=yellow!20, thin, inner sep=2pt},
        ]
            \tikzstyle{latent} = [circle,fill=white,draw=black,inner sep=1pt,
            minimum size=17pt, font=\fontsize{9}{9}\selectfont, node distance=1]

            \node (c1) [obs] {$\cv_1$};
            \node (c2) [obs, right=14pt of c1] {$\cv_2$};
            \node (c3) [obs, right=14pt of c2] {$\cv_3$};
            \node (x_) [obs, transparent, right=10pt of c3] {};

            \node (h1) [det, below=10pt of c1] {$\hv_1$};
            \node (h2) [det, right=10pt of h1] {$\hv_2$};
            \node (h3) [det, right=10pt of h2] {$\hv_3$};

            \node (h4) [det, right=10pt of h3] {$\hv^1$};
            \node (h5) [det, right=15pt of h4] {$\hv^2$};
            \node (h6) [det, right=15pt of h5] {$\hv^3$};

            \edge {c1} {h1};
            \edge {c2} {h2};
            \edge {c3} {h3};

            \edge {h1} {h2};
            \edge {h2} {h3};
            \edge[bend right=45] {h3} {h4};
            \edge[bend right=45] {h3} {h5};
            \edge[bend right=45] {h3} {h6};
            \edge {h4} {h5};
            \edge {h5} {h6};

            \node (x1_) [obs, transparent, above=10pt of h4] {};
            \node (x3_) [obs, transparent, above=10pt of h6] {};
            \node (y3_) [obs, transparent, above=17pt of x3_] {};
            \plate[inner, inner sep=2pt] {CEN} {(x1_)(x3_)(y3_)} {};

            \node (x1) [obs, above=10pt of h4] {$\xv^1$};
            \node (x2) [obs, above=10pt of h5] {$\xv^2$};
            \node (x3) [obs, above=10pt of h6] {$\xv^3$};

            \node (y1) [obs, above=17pt of x1] {$y^1$};
            \node (y2) [obs, above=17pt of x2] {$y^2$};
            \node (y3) [latent, above=17pt of x3] {$y^3$};

            \factor[above=3pt of x1]{x1-y1} {} {} {};
            \node (theta1) [const, right=9pt of x1-y1] {$\thetav^1$};
            \factoredge {x1} {x1-y1} {y1};
            \factor[above=3pt of x2]{x2-y2} {} {} {};
            \node (theta1) [const, right=9pt of x2-y2] {$\thetav^2$};
            \factoredge {x2} {x2-y2} {y2};
            \factor[above=3pt of x3]{x3-y3} {} {} {};
            \node (theta1) [const, right=9pt of x3-y3] {$\thetav^3$};
            \factoredge {x3} {x3-y3} {y3};

            \factor[right=7pt of y1]{y1-y2} {} {} {};
            \factoredge[-] {y1} {y1-y2} {y2};
            \factor[right=7pt of y2]{y2-y3} {} {} {};
            \factoredge[-] {y2} {y2-y3} {y3};

            \edge[bend right=60, dashed] {h4} {x1-y1};
            \edge[bend right=60, dashed] {h5} {x2-y2};
            \edge[bend right=60, dashed] {h6} {x3-y3};

            \draw[dotted, thick] (5.30, 1.75) -- (5.30, 0.75);
            \node[anchor=west, font=\fontsize{7}{8}\selectfont, inner sep=0pt] at (5.25, 1.95) {$t \in [t_2, t_3)$};
        \end{tikzpicture}
        \vspace{-2ex}
        \caption{Architecture used for PhysioNet.}
        \label{fig:LSTM-CRF}
    \end{subfigure}%
    \caption{%
    {\CEN} architectures used in our survival analysis experiments.
    Context encoders were (a) single hidden layer MLP and (b) LSTM.
    Encoders produced inputs for another LSTM over the output time intervals (denoted with $\hv^1$, $\hv^2$, $\hv^3$ hidden states respectively).}
    \label{fig:CEN-CRF}
    \vspace{-2ex}
\end{figure}

To train the model, \citet{lin2011learning} optimize the following objective:
\begin{equation}
    \label{eq:seq-dep-reg-objective}
    \min_{\Thetav} C_1\sum_{t=1}^m \|\theta^{t}\|^2 + C_2 \sum_{t=1}^{m-1} \|\theta^{t+1} - \theta^{t}\|^2 - \log \Lc(\Yv, \Xv; \thetav^{1:m})
\end{equation}
where the first two terms are regularization and the last term is the log of the likelihood:
\begin{equation}
    \label{eq:seq-dep-reg-likelihood}
    \Lc(\Yv, \Xv; \Thetav) = \sum_{i \in \text{NC}} \prob{T = t_i \mid \xv_i, \Thetav} + \sum_{j \in \text{C}} \prob{T > t_j \mid \xv_j, \Thetav}
\end{equation}
where $\text{NC}$ denotes the set of non-censored instances (for which we know the outcome times, $t_i$) and $\text{C}$ is the set of censored inputs (for which we only know the censorship times, $t_j$).
The likelihood of an uncensored and a censored event at time $t \in [t_j, t_{j+1})$ are as follows:\\[-2ex]
\begin{equation}
    \label{eq:uncensored-censored-likelihood}
    \begin{aligned}
        \prob{T = t \mid \xv, \thetav^{1:m}} & = \exp\left\{\sum_{i=j}^m \xv^\top \thetav^i\right\} \Bigg{/} \sum_{k=0}^m \exp\left\{\sum_{i=k+1}^m \xv^\top \thetav^i\right\} \\
        \prob{T \geq t \mid \xv, \thetav^{1:m}} & = \sum_{k=j+1}^m \exp\left\{\sum_{i=k+1}^m \xv^\top \thetav^i\right\}\Bigg{/} \sum_{k=0}^m \exp\left\{\sum_{i=k+1}^m \xv^\top \thetav^i\right\}
    \end{aligned}
\end{equation}

\subsubsection{{\CEN} with Structured Explanations for Survival Analysis}
\label{sec:survival-analysis-cen}

To construct {\CEN} for survival analysis, we follow the structured survival prediction setup described in the previous section.
We define {\CEN} with linear CRF explanations as follows:
\begin{equation}
\label{eq:CEN-CRF}
\begin{aligned}
    & \thetav^t \sim \prob[\wv]{\thetav^t \mid \cv},\, \yv \sim \prob{\Yv \mid \xv, \thetav^{1:m}}, \\
    & \prob{\Yv = (y^1, y^2, \dots, y^m) \mid \xv, \thetav^{1:m}} \propto \exp\left\{\sum_{t=1}^m y^i (\xv^\top \thetav^t) + \omega(y^t, y^{t + 1})\right\}, \\
    & \prob[\wv]{\thetav^t \mid \cv} := \delta(\thetav^t, \phi^t_{\wv, \Dv}(\cv)),\, \phi^t_{\wv, \Dv}(\cv) := \alphav(\hv^{t})^\top \Dv,\, \hv^t := \mathrm{RNN}(\hv^{t-1}, \cv)
    \end{aligned}
\end{equation}
Note that an RNN-based context encoder generates different explanations for each time point, $\thetav^{t}$ (Figure~\ref{fig:CEN-CRF}).
All $\thetav^{t}$ are generated using context- and time-specific attention $\alphav(\hv^t)$ over the dictionary $\Dv$.
We adopt the training objective from \eqref{eq:seq-dep-reg-objective} with the same likelihood \eqref{eq:seq-dep-reg-likelihood}.
The model is a special case of {\CENs} with structured explanations (Section~\ref{sec:structured-explanations}).

\subsubsection{Survival Analysis of Patients in Intense Care Units}
\label{sec:survival-analysis-experiments}

We evaluate the proposed model against baselines on two survival prediction tasks.

\paragraph{Datasets.}
We use two publicly available datasets for survival analysis of of the intense care unit (ICU) patients:
(a) SUPPORT2,\footnote{\url{http://biostat.mc.vanderbilt.edu/wiki/Main/DataSets}.} and
(b) data from the PhysioNet 2012 challenge.\footnote{\url{https://physionet.org/challenge/2012/}.}
The data was preprocessed and used as follows.

\begin{table*}[t!]
    \centering
    \caption{\small%
    Performance of the baselines and {\CENs} with structured explanations.
    The numbers are averages from 5-fold cross-validation; the std. are on the order of the least significant digit.
    ``Acc@K'' denotes accuracy at the K-th temporal quantile (see main text for explanation).}
    \vspace{-1ex}
    \fontsize{8}{10}\selectfont
    \def\arraystretch{1.0}
    \begin{tabular}[t]{@{}lrrrr|lrrrr@{}}
        \toprule
        \multicolumn{5}{c|}{\textbf{SUPPORT2}} &
        \multicolumn{5}{c}{\textbf{PhysioNet Challenge 2012}} \\
        \midrule
        \textbf{Model} &
        \textbf{Acc@25} & \textbf{Acc@50} & \textbf{Acc@75} & \textbf{RAE} &
        \textbf{Model} &
        \textbf{Acc@25} & \textbf{Acc@50} & \textbf{Acc@75} & \textbf{RAE} \\
        \midrule
        Cox         & $84.1$    & $73.7$    & $47.6$    & $0.90$    &
        Cox         & $93.0$    & $69.6$    & $49.1$    & $0.24$    \\
        Aalen       & $87.1$    & $66.2$    & $45.8$    & $0.98$    &
        Aalen       & $93.3$    & $78.7$    & $57.1$    & $0.31$    \\
        CRF         & $84.4$    & $89.3$    & $79.2$    & $0.59$    &
        CRF         & $93.2$    & $85.1$    & $65.6$    & $0.14$    \\
        MLP-CRF     & $\mathbf{87.7}$    & $89.6$    & $80.1$    & $0.62$    &
        LSTM-CRF    & $93.9$    & $86.3$    & $68.1$    & $\mathbf{0.11}$    \\
        \midrule
        MLP-CEN     & $84.4$    & $\mathbf{96.2}$    & $\mathbf{83.3}$    & $\mathbf{0.52}$    &
        LSTM-CEN    & $\mathbf{94.8}$    & $\mathbf{87.5}$    & $\mathbf{70.1}$    & $\mathbf{0.09}$    \\
        \bottomrule
    \end{tabular}
    \label{tab:performance-survival}
    \vspace{-3ex}
\end{table*}

\underline{\texttt{SUPPORT2}}:
The data had 9105 patient records (7105 training, 1000 validation, 1000 test) and 73 variables.
We selected 50 variables for both $\Cv$ and $\Xv$ features (\ie, the context and the variables of interest were identical).
Categorical features (such as \texttt{race} or \texttt{sex}) were one-hot encoded.
The values of all features were non-negative, and we filled the missing values with -1 to preserve the information about missingness.
For CRF-based predictors, we capped the survival timeline at 3 years and converted it into 156 discrete 7-day intervals.

\underline{\texttt{PhysioNet}}:
The data had 4000 patient records, each represented by a 48-hour irregularly sampled 37-dimensional time-series of different measurements taken during the patient's stay at the ICU.
We resampled and mean-aggregated the time-series at 30 min frequency.
This resulted in a large number of missing values that we filled with 0.
The resampled time-series were used as the context, $\Cv$.
For the attributes, $\Xv$, we took the values of the last available measurement for each variable in the series.
For CRF-based predictors, we capped the survival timeline at 60 days and converted into 60 discrete intervals.

\paragraph{Models.}
For baselines, we use the classical Aalen and Cox models\footnote{Implementation based on \url{https://github.com/CamDavidsonPilon/lifelines}.} and the CRF from \citep{lin2011learning}.
All the baselines used $\Xv$ as their inputs.
Next, we combine CRFs with neural encoders in two ways:
\begin{itemize}[noitemsep,topsep=2pt,parsep=2pt,leftmargin=2em]
    \item[(i)] We apply CRFs to the outputs from the neural encoders (the models denoted MLP-CRF and LSTM-CRF).\footnote{Similar models have been very successful in the natural language applications~\citep{collobert2011natural}.}
    Note that parameters of such CRF layer assign weights to the latent features and are not interpretable in terms of the attributes of interest.
    \item[(ii)] We use {\CENs} with CRF-based explanations, that process the context variables, $\Cv$, using the same neural networks as in (i) and output the sequence of parameters $\thetav^{1:m}$ for CRFs, while the latter act on the attributes, $\Xv$, to make structured predictions.
\end{itemize}
More details on the architectures and training are given in Appendix~\ref{app:architectures}.

\begin{figure}[t]
\centering
\includegraphics[width=\textwidth]{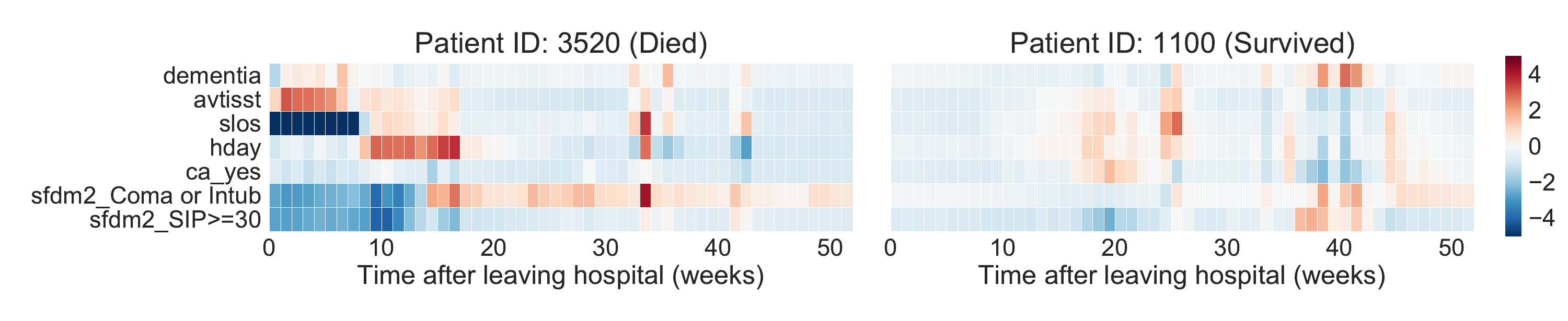}%
\caption{%
Weights of the CEN-generated CRF explanations for two patients from SUPPORT2 dataset for a set of the most influential features:
\texttt{dementia} (comorbidity), \texttt{avtisst} (avg. TISS, days 3-25), \texttt{slos} (days from study entry to discharge), \texttt{hday} (day in hospital at study admit), \texttt{ca\_yes} (the patient had cancer), \texttt{sfdm2\_Coma or Intub} (intubated or in coma at month 2), \texttt{sfdm2\_SIP} (sickness impact profile score at month 2).
Higher weight values correspond to higher contributions to the risk of death after a given time.}
\label{fig:support2-heatmaps}
\end{figure}

\paragraph{Metrics.}
Following \citet{lin2011learning}, we use two metrics specific to survival analysis:
\begin{itemize}[noitemsep,topsep=2pt,parsep=2pt,leftmargin=2em]
    \item[(a)] Accuracy of correctly predicting survival of a patient at times that correspond to 25\%, 50\%, and 75\% population-level temporal quantiles (\ie, the time points such that the corresponding \% of the population in the data were discharged from the study due to censorship or death).
    \item[(b)] The relative absolute error (RAE) between the predicted and actual time of death for non-censored patients.
\end{itemize}

\begin{wrapfigure}[14]{l}{0.37\textwidth}
\centering
\vspace{-2.5ex}
\includegraphics[width=0.37\textwidth]{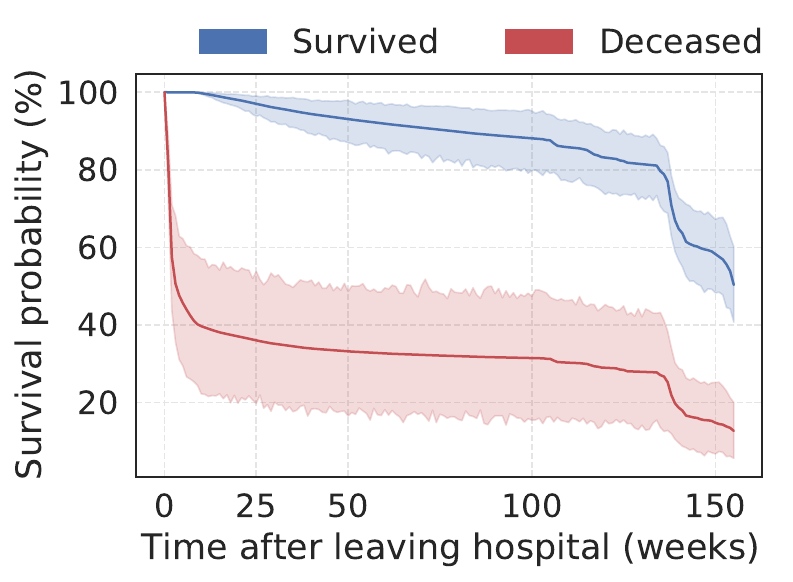}%
\caption{%
CEN-predicted survival curves for 100 random test patients from SUPPORT2.
Color indicates death within 1 year after leaving the hospital.
Shaded regions are 99\% CI.}
\label{fig:support2-lifelines}
\end{wrapfigure}

\paragraph{Performance.}
The results for all models are given in Table~\ref{tab:performance-survival}.
Our implementation of the CRF baseline slightly improves upon the performance reported by~\citet{lin2011learning}.
MLP-CRF and LSTM-CRF improve upon plain CRFs but, as we noted, can no longer be interpreted in terms of the original variables.
{\CENs} outperform or closely match neural CRF models on all metrics while providing interpretable explanations for the predicted risk for each patient at each point in time.

\paragraph{Qualitative analysis.}
To inspect predictions of {\CENs} qualitatively, for any given patient, we can visualize the weights assigned by the corresponding explanation to the respective attributes.
Figure~\ref{fig:support2-heatmaps} shows weights of the explanations for a subset of the most influential features for two patients from SUPPORT2 dataset who were predicted as survivor/non-survivor.
These temporal charts help us (a) to better understand which features the model selects as the most influential at each point in time, and (b) to identify potential inconsistencies in the model or the data---for example, using a chart as in Figure~\ref{fig:support2-heatmaps} we identified and excluded a feature (\texttt{hospdead}) from SUPPORT2 data, which initially was included but leaked information about the outcome as it directly indicated in-hospital death.
Finally, explanations also allow us to better understand patient-specific temporal dynamics of the contributing factors to the survival rates predicted by the model (Figure~\ref{fig:support2-lifelines}).

\section{Conclusion}
\label{sec:conclusion}

In this paper, we have introduced contextual explanation networks (CENs)---a class of models that learn to predict by generating and leveraging intermediate context-specific explanations.
We have formally defined {\CENs} as a class of probabilistic models, considered a number of special cases (\eg, the mixture-of-experts model), and derived learning and inference algorithms within the encoder-decoder framework for simple and sequentially-structured outputs.
We have shown that there are certain conditions when post-hoc explanations are erroneous and misleading.
Such cases are hard to detect unless explanation is a part of the prediction process itself, as in {\CEN}.
Finally, learning to predict and to explain jointly turned out to have a number of benefits, including strong regularization, consistency, and ability to generate explanations with no computational overhead, as shown in our case studies.

We would like to point out a few limitations of our approach and potential ways of addressing those in the future work.
Firstly, while each prediction made by {\CEN} comes with an explanation, the process of conditioning on the context is still uninterpretable.
Ideas similar to context selection~\citep{liu2017context} or rationale generation~\citep{lei2016rationalizing} may help improve interpretability of the conditioning.
Secondly, the space of explanations considered in this work assumes the same graphical structure and parameterization for all explanations and uses a simple sparse dictionary constraint.
This might be limiting, and one could imagine using a more hierarchically structured space of explanations instead, bringing to bear amortized inference techniques~\citep{rudolph2017structured}.
Nonetheless, we believe that the proposed class of models is useful not only for improving prediction capabilities, but also for model diagnostics, pattern discovery, and general data analysis, especially when machine learning is used for decision support in high-stakes applications.

\section{Acknowledgements}
We thank Willie Neiswanger and Mrinmaya Sachan for many useful comments on an early draft of the paper, and Ahmed Hefny, Shashank J. Reddy, Bryon Aragam, and Ruslan Salakhutdinov for helpful discussions.
This work was supported by NIH R01GM114311.
M.A. was supported in part by the CMLH Fellowship.

\clearpage

\appendix
\setcounter{equation}{0}
\renewcommand{\theequation}{\Alph{section}.\arabic{equation}}

\section{Proofs}
\label{app:proofs}

\subsection{Proof of Proposition \ref{prop:CEN-conditional-independence}}
\label{proof:CEN-conditional-independence}

Assume that $\prob{\Yv \mid \Xv, \thetav}$ factorizes as $\prod_{\av \in \Vc_{\Yv}} \prob{\Yv_{\av} \mid \Yv_{\mathrm{MB}(\av)}, \Xv, \thetav_{\av}}$, where $\av$ denotes subsets of the $\Yv$ variables and $\mathrm{MB}(\av)$ stands for the corresponding Markov blankets.
Using the definition of {\CEN} given in \eqref{eq:CEN-general}, we have:
\begin{equation}
    \begin{split}
        \prob{\Yv \mid \Xv, \Cv}
        & = \int \prob{\Yv \mid \Xv, \thetav} \prob{\thetav \mid \Cv} d\thetav \\
        & = \int \prod_{\av \in \Vc_{\Yv}} \prob{\Yv_{\av} \mid \Yv_{\mathrm{MB}(\av)}, \Xv, \thetav_{\av}} \prod_{j}\prob{\theta_j \mid \Cv} d\thetav \\
        & = \prod_{\av \in \Vc_{\Yv}} \left[\int \prob{\Yv_{\av} \mid \Yv_{\mathrm{MB}(\av)}, \Xv, \thetav_{\av}} \prod_{j \in \av}\prob{\theta_j \mid \Cv} d\thetav_{\av}\right] \\
        & = \prod_{\av \in \Vc_{\Yv}} \prob{\Yv_{\av} \mid \Yv_{\mathrm{MB}(\av)}, \Xv, \Cv}
    \end{split}
\end{equation}

\subsection{Proof of Proposition \ref{prop:explanation-contribution-bound}}
\label{proof:explanation-contribution-bound}

To derive the lower bound on the contribution of explanations in terms of expected accuracy, we first need to bound the probability of the error when only $\thetav$ are used for prediction:
\begin{equation*}
    \Pb_e := \prob{\hat \Yv(\thetav) \neq \Yv} = \ep[\thetav \sim \prob{\thetav}]{\prob{\hat \Yv \neq \Yv \mid \thetav}},
\end{equation*}
which we bound using the Fano's inequality~\citep[Ch. 2.11,][]{cover2012elements}:
\begin{equation}
    \label{eq:fano-inequality}
    \Hc\bb{\Pb_e} + \Pb_e \log\bb{|\Yc| - 1} \geq \Hc\bb{\Yc \mid \thetav}
\end{equation}
Since the error ($\hat \Yv(\thetav) \neq \Yv$) is a binary random variable, then $\Hc\bb{\Pb_e} \leq 1$.
After weakening the inequality and using $\Hc\bb{\Yc \mid \thetav} \geq \delta$ from the proposition statement, we get:
\begin{equation}
    \label{eq:expected-predictive-error-noexp-bound}
    \ep[\thetav \sim \prob{\thetav}]{\prob{\hat \Yv \neq \Yv \mid \thetav}} \geq \frac{\Hc\bb{\Yc \mid \thetav} - 1}{\log|\Yc|} \geq \frac{\delta - 1}{\log|\Yc|}
\end{equation}
The claimed lower bound \eqref{eq:cen-exp-contribution-bound} follows after we combine \eqref{eq:expected-predictive-error-noexp-bound} and the assumed bound on the accuracy of the model in terms of $\varepsilon$ given in \eqref{eq:cen-predictive-acc-bound}.

\subsection{Proof of Theorem \ref{thm:LIME-CEN}}
\label{proof:LIME-CEN}

To prove the theorem, consider the case when $f$ is defined by a {\CEN}, instead of $\xv$ we have $(\cv, \xv)$, and the class of approximations, $G$, coincides with the class of explanations, and hence can be represented by $\thetav$.
In this setting, we can pose the same problem as:
\begin{equation}
    \label{eq:lime-cen}
    \hat \thetav = \argmin_{\thetav} \Lc(f, \thetav, \pi_{\cv, \xv}) + \Omega(\thetav)
\end{equation}
Suppose that {\CEN} produces $\thetav^\star$ explanation for the context $\cv$ using a deterministic encoder, $\phi$.
The question is whether and under which conditions $\hat \thetav$ can recover $\thetav^\star$.
Theorem~\ref{thm:LIME-CEN} answers the question in affirmative and provides a concentration result for the case when hypotheses are linear.
Here, we prove Theorem~\ref{thm:LIME-CEN} for a little more general class of log-linear explanations: $\logit{\prob{Y = 1 \mid \xv, \theta}} = \av(\xv)^\top \thetav$, where $\av$ is a $C$-Lipschitz vector-valued function whose values have a zero-mean distribution when $(\xv, \cv)$ are sampled from $\pi_{\xv,\cv}$\footnote{In case of logistic regression, $\av(\xv) = [1, x_1, \dots, x_d]^\top$.}.
For simplicity of the analysis, we consider binary classification and omit the regularization term, $\Omega(g)$.
We define the loss function, $\Lc(f, \thetav, \pi_{\xv, \cv})$, as:
\begin{equation}
    \Lc = \frac{1}{K} \sum_{k=1}^K \left(\logit{\prob{Y = 1 \mid \xv_k - \xv, \cv_k}} - \logit{\prob{Y = 1 \mid \xv_k - \xv, \thetav}}\right)^2
\end{equation}
where $(\xv_k, \cv_k) \sim \pi_{\xv, \cv}$ and $\pi_{\xv, \cv} := \pi_{\xv} \pi_{\cv}$ is a distribution concentrated around $(\xv, \cv)$.
Without loss of generality, we also drop the bias terms in the linear models and assume that $\av(\xv_k - \xv)$ are centered.

\vspace{1ex}
\begin{proof}
The optimization problem \eqref{eq:lime-cen} reduces to the least squares linear regression:
\begin{equation}
    \label{eq:lime-log-linear}
    \hat\thetav = \argmin_{\thetav} \frac{1}{K} \sum_{k=1}^K \left(\logit{\prob{Y = 1 \mid \xv_k - \xv, \cv_k}} - \av(\xv_k - \xv)^\top \thetav\right)^2
\end{equation}
We consider deterministic encoding, $\prob{\thetav \mid \cv} := \delta(\thetav, \phiv(\cv))$, and hence we have:
\begin{equation}
    \begin{split}
        \logit{\prob{Y = 1 \mid \xv_k - \xv, \cv_k}}
        & = \logit{\prob{Y = 1 \mid \xv_k - \xv, \thetav = \phiv(\cv_k)}} \\
        & = \av(\xv_k - \xv)^\top \phiv(\cv_k)
    \end{split}
\end{equation}
To simplify the notation, we denote $\av_k := \av(\xv_k - \xv)$, $\phiv_k := \phiv(\cv_k)$, and $\phiv := \phiv(\cv)$.
The solution of \eqref{eq:lime-log-linear} now can be written in a closed form:
\begin{equation}
    \hat\thetav = \left[\frac{1}{K} \sum_{k=1}^K \av_k \av_k^\top\right]^{+} \left[\frac{1}{K} \sum_{k=1}^K \av_k \av_k^\top \phiv_k\right]
\end{equation}
Note that $\hat\thetav$ is a random variable since $(\xv_k, \cv_k)$ are randomly generated from $\pi_{\xv, \cv}$.
To further simplify the notation, denote $M := \frac{1}{K} \sum_{k=1}^K \av_k \av_k^\top$.
To get a concentration bound on $\|\hat\thetav - \thetav^\star\|$, we will use the continuity of $\phi(\cdot)$ and $\av(\cdot)$, concentration properties of $\pi_{\xv, \cv}$ around $(\xv, \cv)$, and some elementary results from random matrix theory.
To be more concrete, since we assumed that $\pi_{\xv,\cv}$ factorizes, we further let $\pi_\xv$ and $\pi_\cv$ concentrate such that $\prob[\pi_{\xv}]{\|\xv^\prime - \xv\| > t} < \varepsilon_{\xv}(t)$ and $\prob[\pi_{\cv}]{\|\cv^\prime - \cv\| > t} < \varepsilon_{\cv}(t)$, respectively, where $\varepsilon_{\xv}(t)$ and $\varepsilon_{\cv}(t)$ both go to 0 as $t \rightarrow \infty$, potentially at different rates.

First, we have the following bound from the convexity of the norm:
\begin{eqnarray}
    \prob{\|\hat\thetav - \thetav^\star\| > t}
    &=& \prob{\left\|\frac{1}{K} \sum_{k=1}^K\left[M^{+} \av_k \av_k^\top (\phiv_k - \phiv)\right]\right\| > t} \\
    &\leq& \prob{\frac{1}{K} \sum_{k=1}^K \left\|M^{+} \av_k \av_k^\top (\phiv_k - \phiv)\right\| > t}
\end{eqnarray}
By making use of the inequality $\|Ax\| \leq \|A\| \|x\|$, where $\|A\|$ denotes the spectral norm of the matrix $A$, the $L$-Lipschitz property of $\phi(\cv)$, the $C$-Lipschitz property of $\av(\xv)$, and the concentration of $\xv_k$ around $\xv$, we have
\begin{eqnarray}
    \prob{\|\hat\thetav - \thetav^\star\| > t}
    &\leq& \prob{L \frac{1}{K} \sum_{k=1}^K \left\|M^{+} \av_k \av_k^\top\right\| \|\cv_k - \cv\| > t} \\
    &\leq& \prob{C L \left\|M^{+}\right\| \frac{1}{K} \sum_{k=1}^K \left\|\av_k \av_k^\top\right\| \|\cv_k - \cv\| > t} \\
    &\leq& \prob{\frac{C L}{\lambda_\mathrm{min}(M)} \frac{1}{K} \sum_{k=1}^K \|\xv_k - \xv\| \|\cv_k - \cv\| > t} \\
    &\leq& \prob{\frac{C L \tau^2}{\lambda_\mathrm{min}(M)} > t} + \prob{\|\xv_k - \xv\| \|\cv_k - \cv\| > \tau^2} \\
    &\leq& \prob{\lambda_\mathrm{min}\left(M / (C \tau)^2\right) < \frac{L}{C^2 t}} + \varepsilon_\xv(\tau) + \varepsilon_\cv(\tau)
\end{eqnarray}
Note that we used the fact that the spectral norm of a rank-1 matrix, $\av(\xv_k) \av(\xv_k)^\top$, is simply the norm of $\av(\xv_k)$, and the spectral norm of the pseudo-inverse of a matrix is equal to the inverse of the least non-zero singular value of the original matrix: $\|M^+\| \leq \lambda_\mathrm{max}(M^+) = \lambda^{-1}_\mathrm{min}(M)$.

Finally, we need a concentration bound on $\lambda_\mathrm{min}\left(M / (C \tau)^2\right)$ to complete the proof.
Note that $\frac{M}{C^2\tau^2} = \frac{1}{K} \sum_{k=1}^K \left(\frac{\av_k}{C\tau}\right) \left(\frac{\av_k}{C\tau}\right)^\top$, where the norm of $\left(\frac{\av_k}{C\tau}\right)$ is bounded by 1.
If we denote $\mu_\mathrm{min}(C\tau)$ the minimal eigenvalue of $\cov{\frac{\av_k}{C\tau}}$, we can write the matrix Chernoff inequality~\citep{tropp2012user} as follows:
\begin{equation*}
    \prob{\lambda_\mathrm{min}\left(M / (C \tau)^2\right) < \alpha} \leq d \exp\left\{-K D(\alpha \| \mu_\mathrm{min}(C\tau))\right\}, \quad \alpha \in [0, \mu_\mathrm{min}(C\tau)]
\end{equation*}
where $d$ is the dimension of $\av_k$, $\alpha := \frac{L}{C^2 t}$, and $D(a\|b)$ denotes the binary information divergence:
\begin{equation*}
    D(a\|b) = a\log\left(\frac{a}{b}\right) + (1 - a)\log\left(\frac{1 - a}{1 - b}\right).
\end{equation*}
The final concentration bound has the following form:
\begin{equation}
    \prob{\|\hat\thetav - \thetav^\star\| > t} \leq d \exp\left\{-K D\left(\frac{L}{C^2 t} \| \mu_\mathrm{min}(C\tau)\right)\right\} + \varepsilon_\xv(\tau) + \varepsilon_\cv(\tau)
\end{equation}
We see that as $\tau \rightarrow \infty$ and $t \rightarrow \infty$ all terms on the right hand side vanish, and hence $\hat\thetav$ concentrates around $\thetav^\star$.
Note that as long as $\mu_\mathrm{min}(C\tau)$ is far from 0, the first term can be made negligibly small by sampling more points around $(\xv, \cv)$.
Finally, we set $\tau \equiv t$ and denote the right hand side by $\delta_{K,L,C}(t)$ that goes to 0 as $t \rightarrow \infty$ to recover the statement of the original theorem.
\end{proof}
\begin{remark}
    We have shown that $\hat\thetav$ concentrates around $\thetav^\star$ under mild conditions.
    With more assumptions on the sampling distribution, $\pi_{\xv, \cv}$, (e.g., sub-gaussian) one could derive precise convergence rates.
    Note that we are in total control of any assumptions we put on $\pi_{\xv, \cv}$ since precisely that distribution is used for sampling.
    This is a major difference between the local approximation setup here and the setup of linear regression with random design; in the latter case, we have no control over the distribution of the design matrix, and any assumptions we make could potentially be unrealistic.
\end{remark}
\begin{remark}
    Note that concentration analysis of a more general case when the loss $\Lc$ is a general convex function and $\Omega(g)$ is a decomposable regularizer could be done by using results from the M-estimation theory~\citep{negahban2009unified}, but would be much more involved and unnecessary for our purposes.
\end{remark}

\section{Experimental Details}
\label{app:experiments}

This section provides details on the experimental setups including architectures, training procedures, etc.
Additionally, we provide and discuss qualitative results for {\CENs} on the MNIST and IMDB datasets.

\subsection{Additional Details on the Datasets and Experiment Setups}
\label{app:setup-details}

\paragraph{MNIST.}
We used the classical split of the dataset into 50k training, 10k validation, and 10k testing points.
All models were trained for 100 epochs using the AMSGrad optimizer~\citep{reddi2019convergence} with the learning rate of $10^{-3}$.
No data augmentation was used in any of our experiments.
HOG representations were computed using $3 \times 3$ blocks.

\paragraph{CIFAR10.}
For this set of experiments, we followed the setup given by \citet{cifar10blog}, reimplemented in Keras~\citep{chollet2015keras} with TensorFlow~\citep{abadi2016tensorflow} backend.
The input images were global contrast normalized (a.k.a. GCN whitened) while the rescaled image representations were simply standardized.
Again, HOG representations were computed using $3 \times 3$ blocks.
No data augmentation was used in our experiments.

\paragraph{IMDB.}
We considered the labeled part of the data only (50,000 reviews total).
The data were split into 20,000 train, 5,000 validation, and 25,000 test points.
The vocabulary was limited to 20,000 most frequent words (and 5,000 most frequent words when constructing BoW representations).
All models were trained with the AMSGrad optimizer~\citep{} with $10^{-2}$ learning rate.
The models were initialized randomly; no pre-training or any other unsupervised/semi-supervised technique was used.

\paragraph{Satellite.}
As described in the main text, we used a pre-trained {\VGG} network\footnote{The model was taken form \url{https://github.com/nealjean/predicting-poverty}.} to extract features from the satellite imagery.
Further, we added one fully connected layer network with 128 hidden units used as the context encoder.
For the VCEN model, we used dictionary-based encoding with Dirichlet prior and logistic normal distribution as the output of the inference network.
For the decoder, we used an MLP of the same architecture as the encoder network.
All models were trained with Adam optimizer with 0.05 learning rate.
The results were obtained by 5-fold cross-validation.

\paragraph{Medical data.}
We have used minimal pre-processing of both SUPPORT2 and PhysioNet datasets limited to standardization and missing-value filling.
We found that denoting missing values with negative entries ($-1$) often led a slightly improved performance compared to any other NA-filling techniques.
PhysioNet time series data was irregularly sampled across the time, so we had to resample temporal sequences at regular intervals of 30 minutes (consequently, this has created quite a few missing values for some of the measurements).
All models were trained using Adam optimizer with $10^{-2}$ learning rate.

\begin{figure}[ht]
\begin{subfigure}[b]{0.3308\linewidth}
    \includegraphics[width=\textwidth]{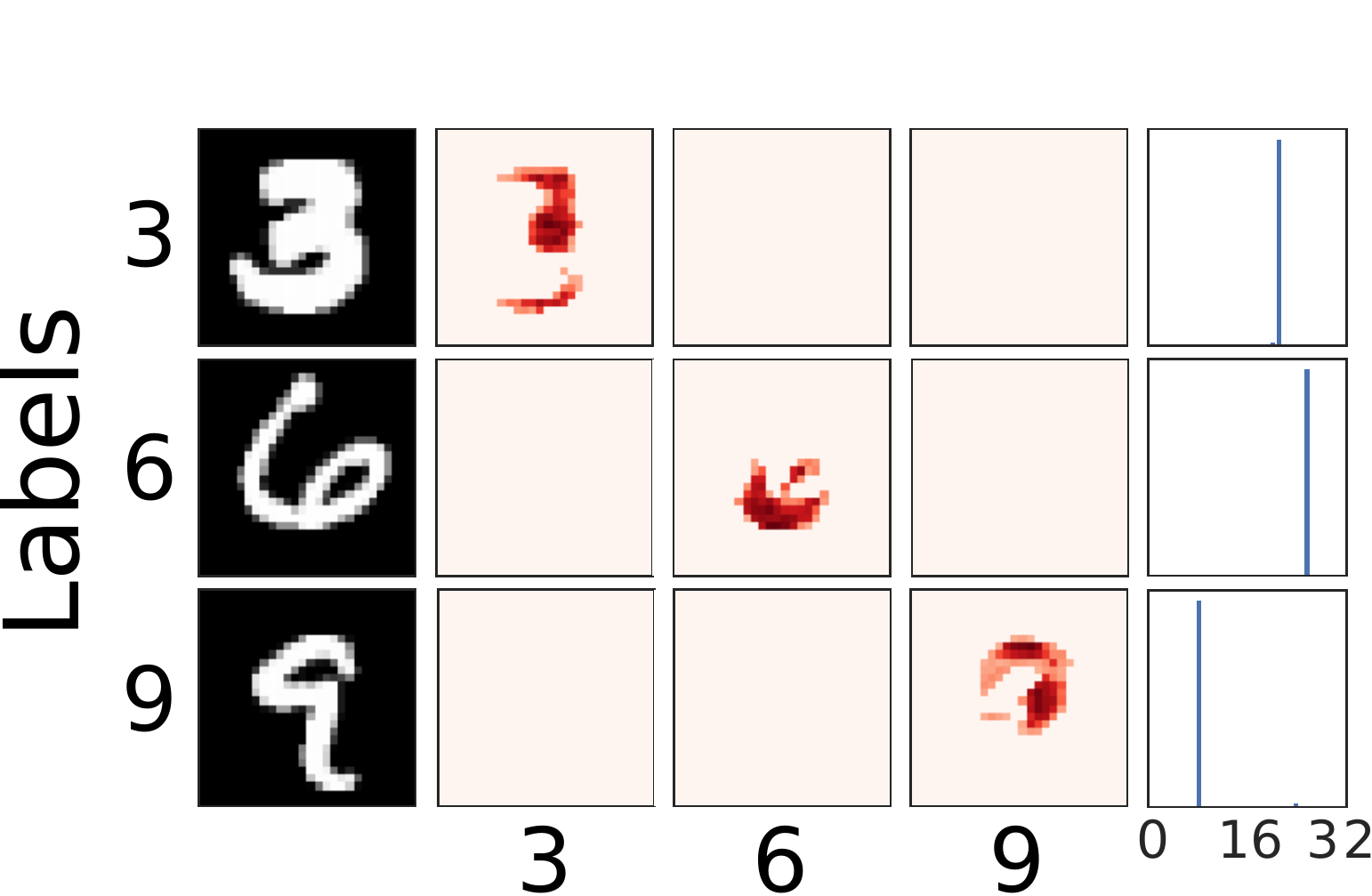}%
    \caption{Correct}\label{fig:mnist-qualitative-correct}
\end{subfigure}
~
\begin{subfigure}[b]{0.3\linewidth}
    \includegraphics[width=\textwidth]{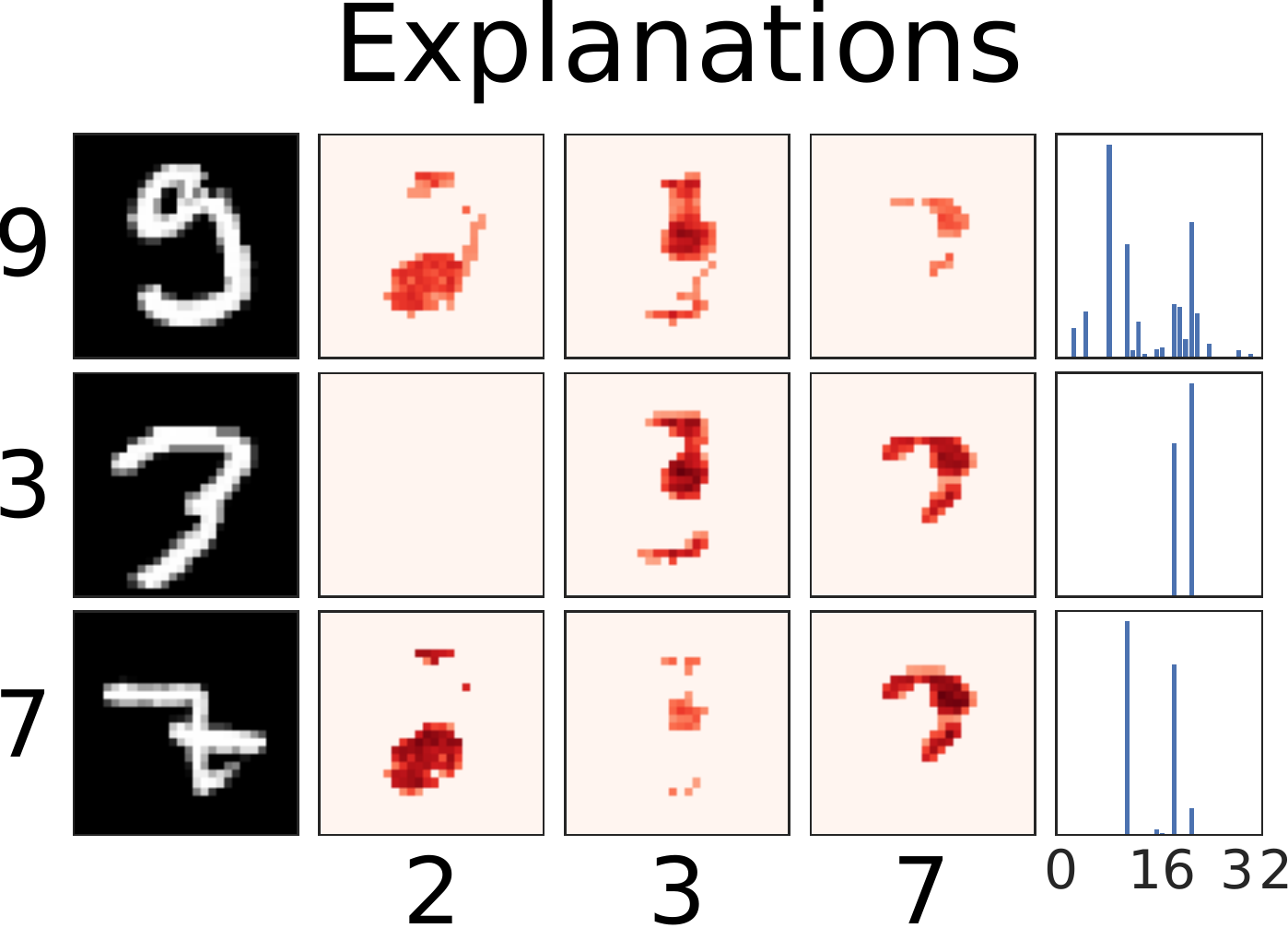}%
    \caption{Misclassified}\label{fig:mnist-qualitative-incorrect}
\end{subfigure}
~
\begin{subfigure}[b]{0.3265\linewidth}
    \includegraphics[width=\textwidth]{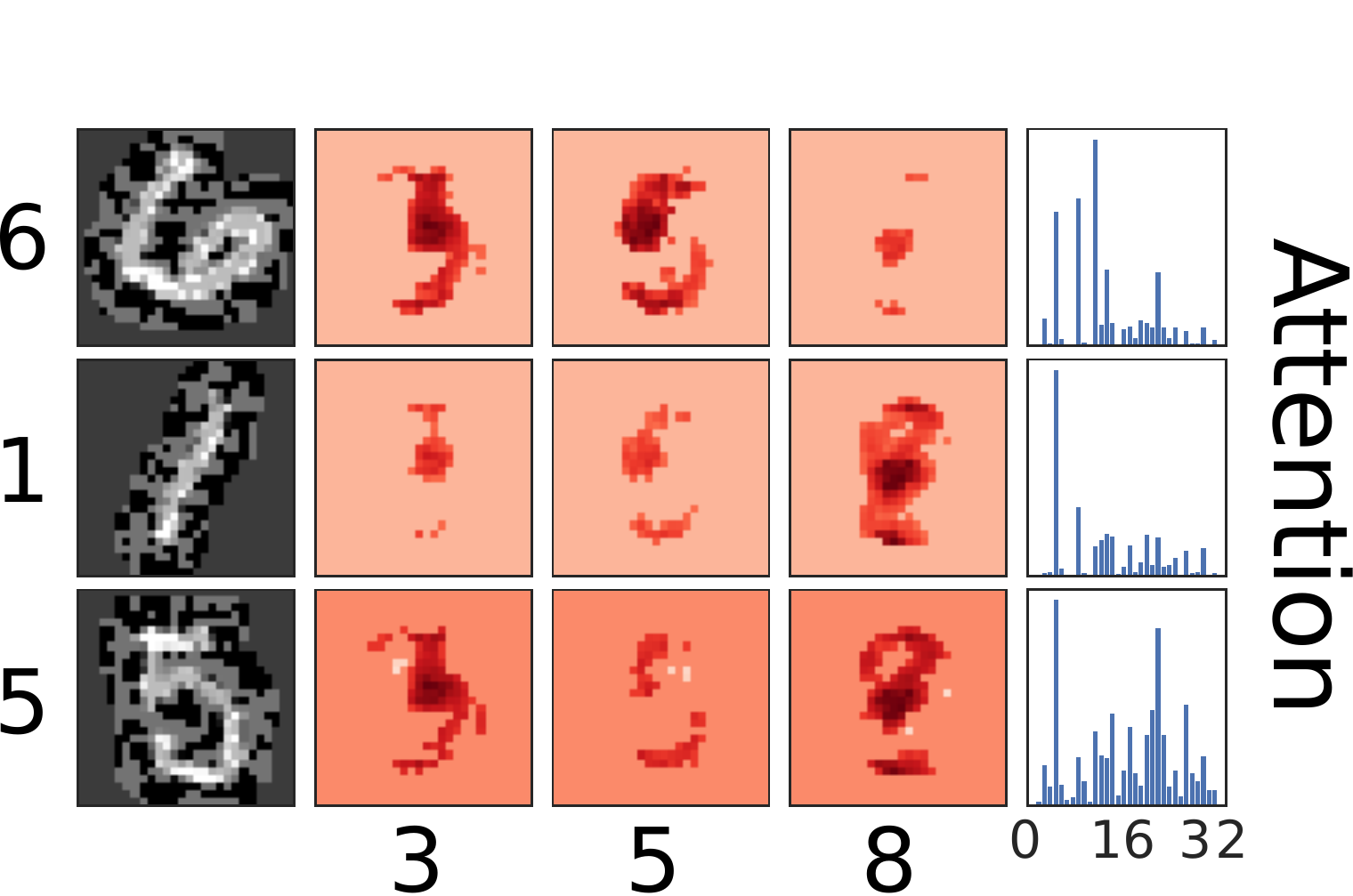}%
    \caption{Adversarial}\label{fig:mnist-qualitative-adversarial}
\end{subfigure}
\begin{subfigure}[b]{0.55\linewidth}
    \includegraphics[width=\textwidth]{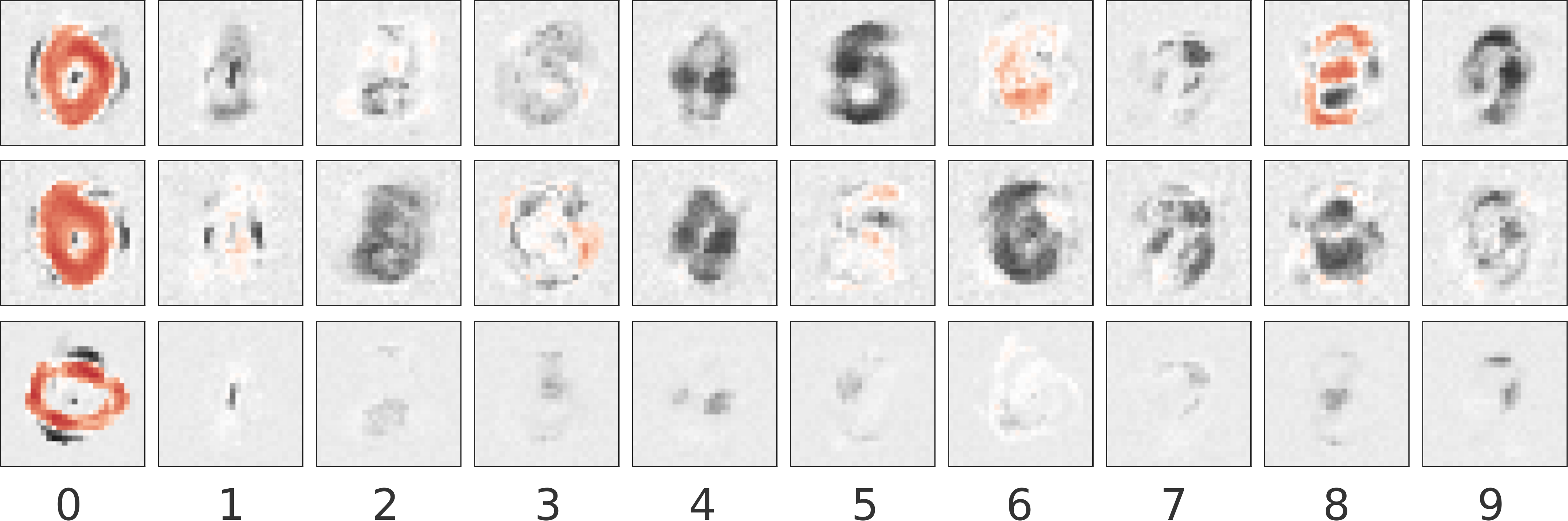}%
    \caption{Selected elements of the explanation dictionary}\label{fig:mnist-zero-explanations}
\end{subfigure}%
\begin{subfigure}[b]{0.45\linewidth}
    \includegraphics[width=\textwidth]{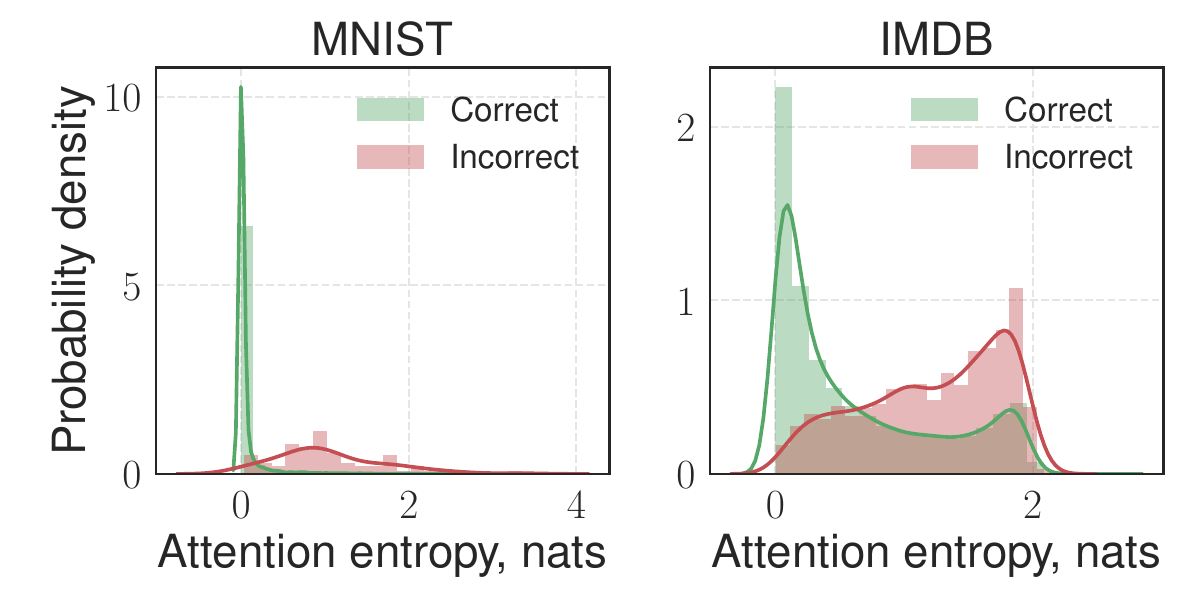}%
    \caption{Attention}\label{fig:attention-entropy}
\end{subfigure}
\caption{%
Explanations generated by CEN for the 3 top classes and the corresponding attention vectors for (a) correctly classified, (b) misclassified, and (c) adversarially constructed images.
Adversarial examples were generated using the fast gradient sign method (FGSM)~\citep{papernot2016practical}.
(d) Elements from the learned 32-element dictionary that correspond to different writing styles of 0 digits.
(e) Histogram of the attention entropy for correctly and incorrectly classified test instances for {\CEN}-\texttt{pxl} on MNIST and {\CEN}-\texttt{tpc} on IMDB.
}\label{fig:mnist-plots}
\end{figure}

\subsection{More on Qualitative Analysis}
\label{app:qualitative-analysis}

\subsubsection{MNIST}

Figures~\ref{fig:mnist-qualitative-correct}, \ref{fig:mnist-qualitative-incorrect}, and \ref{fig:mnist-qualitative-adversarial} visualize explanations for predictions made by {\CEN}-\texttt{pxl} on MNIST.
The figures correspond to 3 cases where {\CEN} (a) made a correct prediction, (b) made a mistake, and (c) was applied to an adversarial example (and made a mistake).
Each chart consists of the following columns: true labels, input images, explanations for the top 3 classes (as given by the activation of the final softmax layer), and attention vectors used to select explanations from the global dictionary.
A small subset of explanations from the dictionary is visualized in Figure~\ref{fig:mnist-zero-explanations} (the full dictionary is given in Figure~\ref{fig:mnist-dict}), where each image is a weight vector used to construct the pre-activation for a particular class.
Note that different elements of the dictionary capture different patterns in the data (in Figure~\ref{fig:mnist-zero-explanations}, different styles of writing the 0 digit) which {\CEN} actually uses for prediction.

Also note that confident correct predictions (Figures~\ref{fig:mnist-qualitative-correct}) are made by selecting a single explanation from the dictionary using a sharp attention vector.
However, when the model makes a mistake, its attention is often dispersed (Figures~\ref{fig:mnist-qualitative-incorrect} and \ref{fig:mnist-qualitative-adversarial}), i.e., there is uncertainty in which pattern it tries to use for prediction.
Figure~\ref{fig:attention-entropy} further quantifies this phenomenon by plotting histogram of the attention entropy for all test examples which were correctly and incorrectly classified.
While {\CENs} are certainly not adversarial-proof, high entropy of the attention vectors is indicative of ambiguous or out-of-distribution examples which is helpful for model diagnostics.

\subsubsection{IMDB}

Similar to MNIST, we train {\CEN}-\texttt{tpc} with linear explanations in terms of topics on the IMDB dataset.
Then, we generate explanations for each test example and visualize histograms of the weights assigned by the explanations to 6 selected topics in Figure~\ref{fig:imdb-dict-hist}.
The 3 topics in the top row are acting- and plot-related (and intuitively have positive, negative, or neutral connotation), while the 3 topics in the bottom are related to particular genre of the movies.

Note that acting-related topics turn out to be bi-modal, i.e., contributing either positively, negatively, or neutrally to the sentiment prediction in different contexts.
As expected intuitively, {\CEN} assigns highly negative weight to the topic related to ``bad acting/plot'' and highly positive weight to ``great story/performance'' in most of the contexts (and treats those neutrally conditional on some of the reviews).
Interestingly, genre-related topics almost always have a negligible contribution to the sentiment (i.e., get almost 0 weights assigned by explanations) which indicates that the learned model does not have any particular bias towards or against a given genre.
Importantly, inspecting summary statistics of the explanations generated by {\CEN} allows us to explore the biases that the model picks up from the data and actively uses for prediction\footnote{%
If we wish to enforce or eliminate certain patterns from explanations (e.g., to ensure fairness), we may impose additional constraints on the dictionary.
However, this is beyond the scope of this work.
}.

Figure~\ref{fig:imdb-dict-tpc} visualizes the full dictionary of size 16 learned by {\CEN}-\texttt{tpc}.
Each column corresponds to a dictionary atom that represents a typical explanation pattern that {\CEN} attends to before making a prediction.
By inspecting the dictionary, we can find interesting patterns.
For instance, atoms 5 and 11 assign inverse weights to topics \texttt{[kid, child, disney, family]} and \texttt{[sexual, violence, nudity, sex]}.
Depending on the context of the review, {\CEN} may use one of these patterns to predict the sentiment.
Note that these two topics are negatively correlated across all dictionary elements, which again is quite intuitive.

\subsubsection{Satellite}

We visualize the two explanations, M1 and M2, learned by {\CEN}-\texttt{att} on the Satellite dataset in full in
Figures~\ref{fig:satellite-models-full} and provide additional correlation plots between the selected explanation and values of each survey variable in Figure~\ref{fig:satellite-models-feature-corr}.

\begin{figure}[t!]
    \centering
    \includegraphics[width=0.71\textwidth]{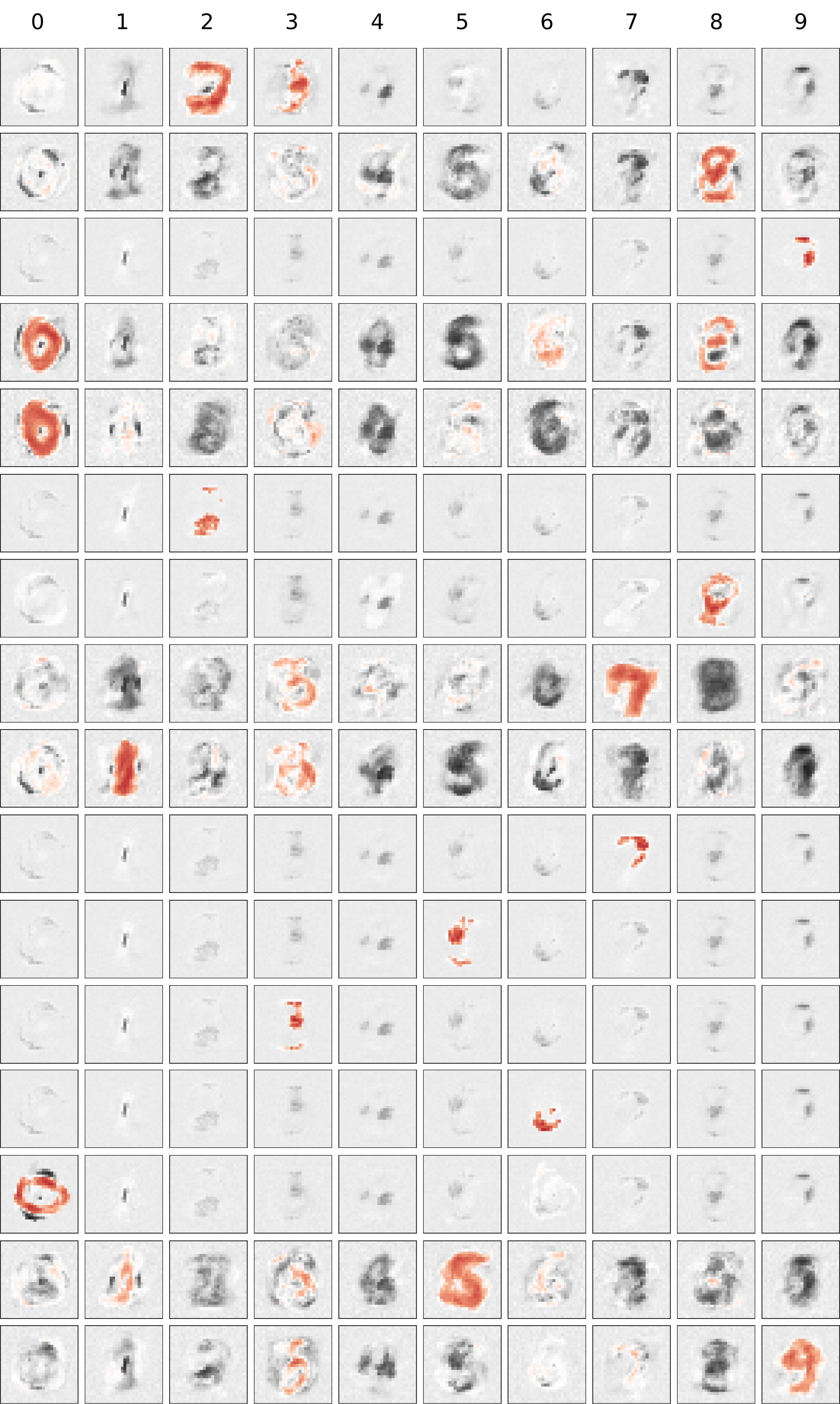}
    \caption{%
    Visualization of the model dictionary learned by CEN on MNIST.
    Each row corresponds to a dictionary element, and each column corresponds to the weights of the model voting for each class of digits.
    Images visualize the weights of the models.
    Red corresponds to high positive values, dark gray to high negative values, and white to values that are close to 0.}
    \label{fig:mnist-dict}
\end{figure}

\begin{figure}[t!]
\centering
\includegraphics[width=0.97\textwidth]{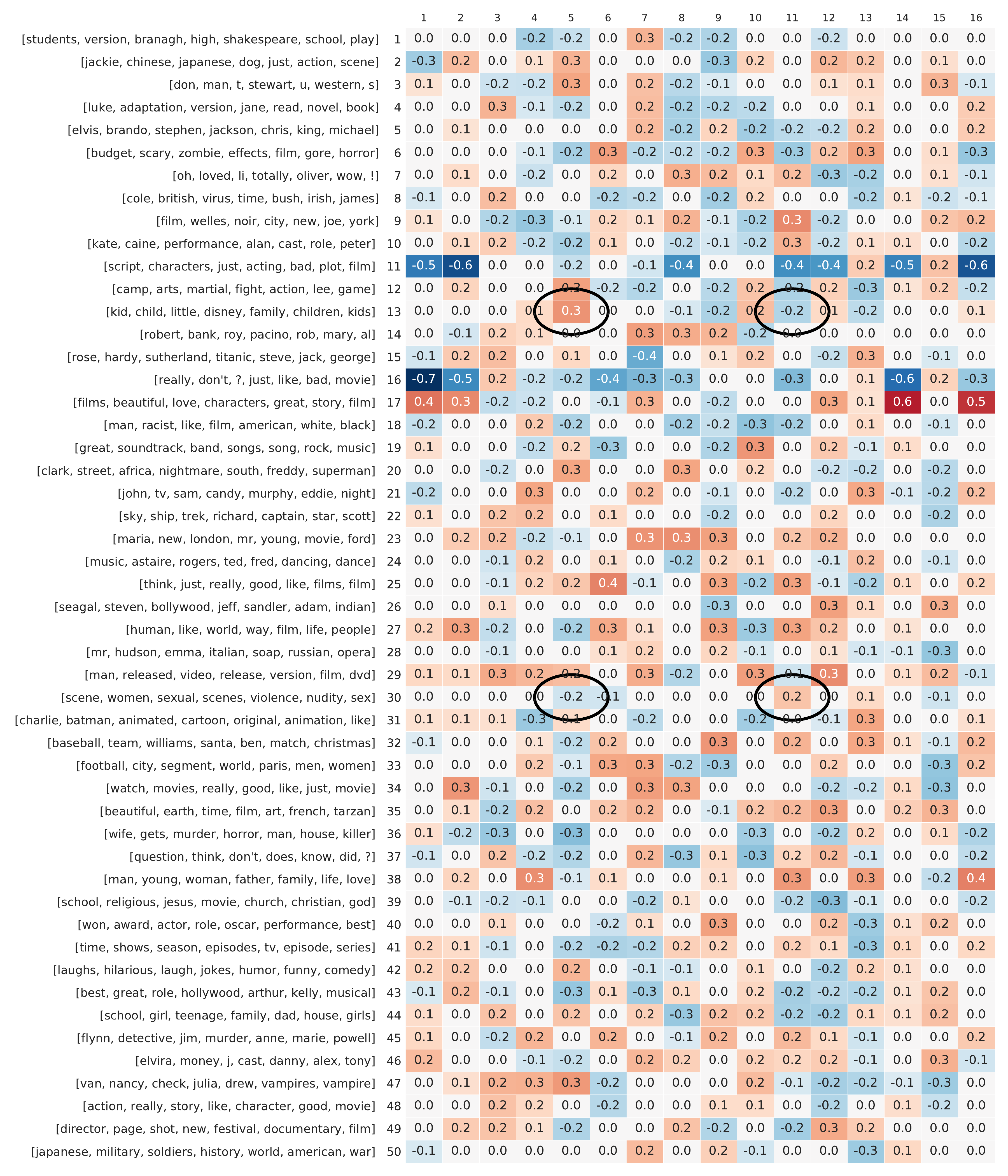}%
\caption{%
The full dictionary learned by CEN-\texttt{tpc} model: rows correspond to topics and columns correspond to dictionary atoms.
Very small values were thresholded for visualization clarity.
Different atoms capture different prediction patterns;
for example, atom 5 assigns a highly positive weight to the \texttt{[kid, child, disney, family]} topic and down-weighs \texttt{[sexual, violence, nudity, sex]}, while atom 11 acts in an opposite manner.
Given the context of the review, CEN combines just a few atoms to make a prediction.}
\label{fig:imdb-dict-tpc}
\end{figure}

\begin{figure}[th!]
\begin{subfigure}[b]{\textwidth}
    \centering
    \includegraphics[width=\textwidth]{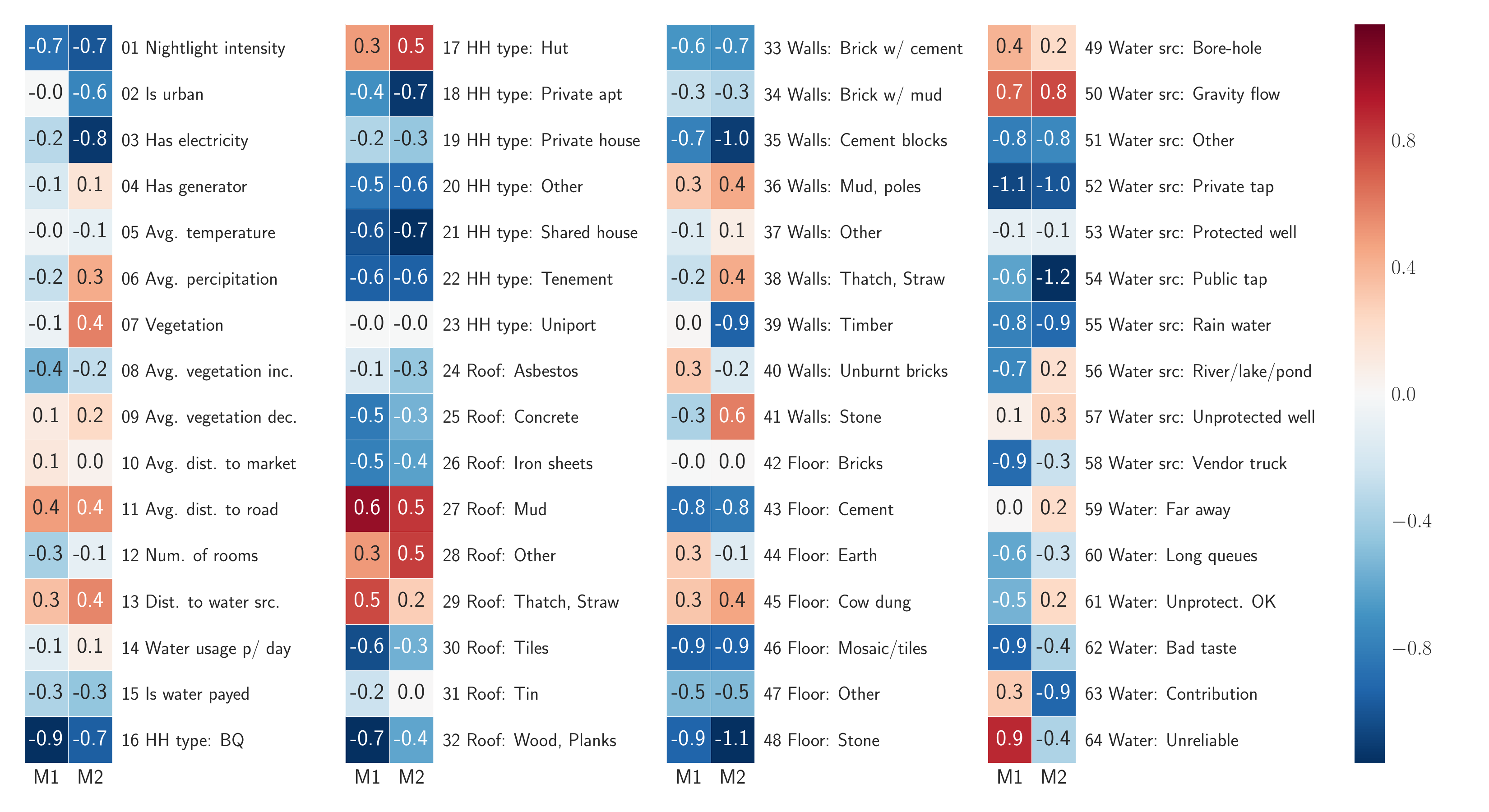}
    \vspace{-4ex}\caption{%
        Full visualization of models M1 and M2 learned by CEN on Satellite data.
    }\label{fig:satellite-models-full}
\end{subfigure}
\begin{subfigure}[b]{\textwidth}
    \centering
    \includegraphics[width=\textwidth]{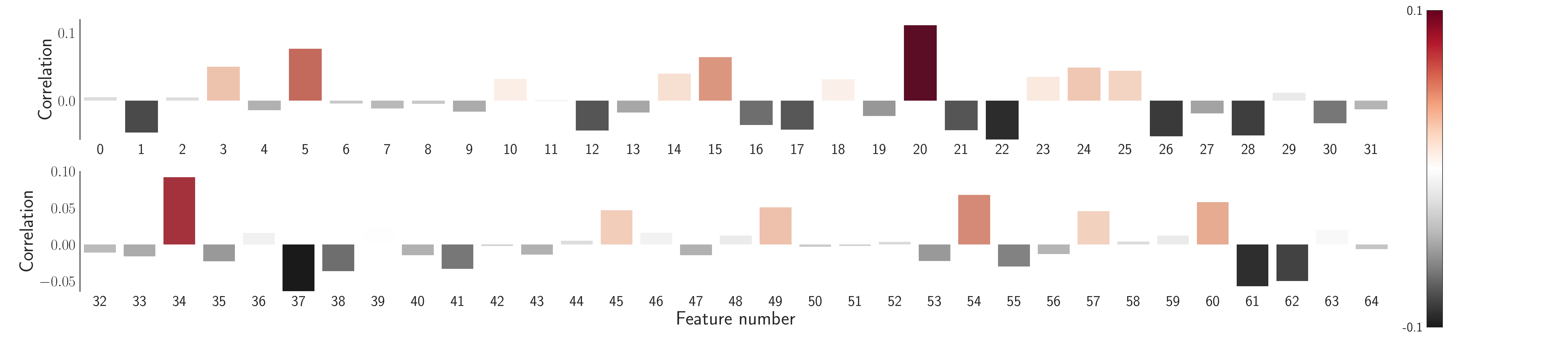}
    \vspace{-4ex}\caption{%
        Correlation between the selected explanation and the value of a particular survey variable.
    }\label{fig:satellite-models-feature-corr}
\end{subfigure}
\caption{%
Additional visualizations for CENs trained on the Satellite data.}
\label{fig:satellite-appendix}
\end{figure}

\subsection{Model Architectures}
\label{app:architectures}

Architectures of the model used in our experiments are summarized in Tables~\ref{tab:architectures-mnist-imdb},~\ref{tab:architectures-cifar10-satellite},~\ref{tab:architectures-survival}.

\clearpage

\begin{table}[t!]
\caption{\small%
Top-performing architectures used in our experiments on MNIST and IMDB datasets.}
\label{tab:architectures-mnist-imdb}
\smallskip
\centering
\begin{subtable}[t]{.52\textwidth}
\caption{\small MNIST}
\scriptsize
\def\arraystretch{1.2}
\setlength\tabcolsep{2pt}
\begin{tabular}[t]{@{}c|l|r@{}}
    \toprule
    \multicolumn{3}{c}{\textbf{Convolutional Encoder}}          \\
    \midrule
    \multirow{15}{*}{\rotatebox{90}{Convolutional Block}}
    &   layer                           & Conv2D                \\
    &   \# filters                      & 32                    \\
    &   kernel size                     & $3 \times 3$          \\
    &   strides                         & $1 \times 1$          \\
    &   padding                         & valid                 \\
    &   activation                      & ReLU                  \\
    \cmidrule{2-3}
    &   layer                           & Conv2D                \\
    &   \# filters                      & 32                    \\
    &   kernel size                     & $3 \times 3$          \\
    &   strides                         & $1 \times 1$          \\
    &   padding                         & valid                 \\
    &   activation                      & ReLU                  \\
    \cmidrule{2-3}
    &   layer                           & MaxPoo2D              \\
    &   pooling size                    & $2 \times 2$          \\
    &   dropout                         & 0.25                  \\
    \midrule
    &   layer                           & Dense                 \\
    &   units                           & 128                   \\
    &   dropout                         & 0.50                  \\
    \midrule
    \multicolumn{2}{l}{\# blocks}       & 1                     \\
    \multicolumn{2}{l}{\# params}       & 1.2M                  \\
    \bottomrule
\end{tabular}%
\hspace{2pt}%
\setlength\tabcolsep{2pt}
\begin{tabular}[t]{@{}l|r@{}}
    \toprule
    \multicolumn{2}{c}{\textbf{Contextual Explanations}}        \\
    \midrule
    model                               & Logistic regr.        \\
    features                            & HOG (3, 3)            \\
    \# features                         & 729                   \\
    standardized                        & Yes                   \\
    dictionary                          & 256                   \\
    $l_1$ penalty                       & $5 \cdot 10^{-5}$     \\
    $l_2$ penalty                       & $1 \cdot 10^{-6}$     \\
    \midrule
    model                               & Logistic reg.         \\
    features                            & Pixels (20, 20)       \\
    \# features                         & 400                   \\
    standardized                        & Yes                   \\
    dictionary                          & 64                    \\
    $l_1$ penalty                       & $5 \cdot 10^{-5}$     \\
    $l_2$ penalty                       & $1 \cdot 10^{-6}$     \\
    \midrule[.75pt]
    \multicolumn{2}{c}{\textbf{Contextual VAE}}                 \\
    \midrule
    prior                               & $\mathrm{Dir}(0.2)$   \\
    sampler                             & LogisticNormal        \\
    \bottomrule
\end{tabular}
\end{subtable}%
\begin{subtable}[t]{.49\textwidth}
\caption{\small IMDB}
\scriptsize
\def\arraystretch{1.2}
\setlength\tabcolsep{2pt}
\begin{tabular}[t]{@{}l|r@{}}
    \toprule
    \multicolumn{2}{c}{\textbf{Squential Encoder}}              \\
    \midrule
    layer                           & Embedding                 \\
    vocabulary                      & 20k                       \\
    dimension                       & 1024                      \\
    \midrule
    layer                           & LSTM                      \\
    bidirectional                   & Yes                       \\
    units                           & 256                       \\
    max length                      & 200                       \\
    dropout                         & 0.25                      \\
    rec. dropout                    & 0.25                      \\
    \midrule
    layer                           & MaxPool1D                 \\
    \midrule
    \# params                       & 23.1M                     \\
    \bottomrule
\end{tabular}%
\hspace{2pt}%
\setlength\tabcolsep{2pt}
\begin{tabular}[t]{@{}l|r@{}}
    \toprule
    \multicolumn{2}{c}{\textbf{Contextual Explanations}}        \\
    \midrule
    model                               & Logistic reg.         \\
    features                            & BoW                   \\
    \# features                         & 20k                   \\
    Dictionary                          & 32                    \\
    $l_1$ penalty                       & $5 \cdot 10^{-5}$     \\
    $l_2$ penalty                       & $1 \cdot 10^{-6}$     \\
    \midrule
    model                               & Logistic reg.         \\
    features                            & Topics                \\
    \# features                         & 50                    \\
    Dictionary                          & 16                    \\
    $l_1$ penalty                       & $1 \cdot 10^{-6}$     \\
    $l_2$ penalty                       & $1 \cdot 10^{-8}$     \\
    \midrule[.75pt]
    \multicolumn{2}{c}{\textbf{Contextual VAE}}                 \\
    \midrule
    Prior                               & $\mathrm{Dir}(0.1)$   \\
    Sampler                             & LogisticNormal        \\
    \bottomrule
\end{tabular}
\end{subtable}
\end{table}

\begin{table}[t!]
\caption{\small%
Top-performing architectures used in our experiments on CIFAR10 and Satellite datasets.
\VGG architecture for CIFAR10 was taken from \url{https://github.com/szagoruyko/cifar.torch} but implemented in Keras with TensorFlow backend.
Weights of the pre-trained \VGGF model for the Satellite experiments were taken from \url{https://github.com/nealjean/predicting-poverty}.}
\label{tab:architectures-cifar10-satellite}
\smallskip
\begin{subtable}[t]{.52\textwidth}
    \caption{\small CIFAR10}
    \centering
    \scriptsize
    \def\arraystretch{1.2}
    \setlength\tabcolsep{2pt}
    \begin{tabular}[t]{@{}c|l|r@{}}
        \toprule
        \multicolumn{3}{c}{\textbf{Convolutional Encoder}}              \\
        \midrule
        \multirow{3}{*}{\rotatebox{90}{VGG-16}}
        &   model                           & \VGG                      \\
        &   pretrained                      & No                        \\
        &   fixed weights                   & No                        \\
        \midrule
        \multirow{6}{*}{\rotatebox{90}{MLP}}
        &   layer                           & Dense                     \\
        &   pretrained                      & No                        \\
        &   fixed weights                   & No                        \\
        &   units                           & 16                        \\
        &   dropout                         & 0.25                      \\
        &   activation                      & ReLU                      \\
        \midrule
        \multicolumn{2}{l}{\# params}       & 20.0M                     \\
        \bottomrule
    \end{tabular}%
    \hspace{1pt}%
    \setlength\tabcolsep{2pt}
    \begin{tabular}[t]{@{}l|r@{}}
        \toprule
        \multicolumn{2}{c}{\textbf{Contextual Explanations}}            \\
        \midrule
        model                         & Logistic reg.                   \\
        features                      & HOG (3, 3)                      \\
        \# features                   & 1024                            \\
        dictionary                    & 16                              \\
        $l_1$ penalty                 & $1 \cdot 10^{-5}$               \\
        $l_2$ penalty                 & $1 \cdot 10^{-6}$               \\
        \midrule[.75pt]
        \multicolumn{2}{c}{\textbf{Contextual VAE}}                     \\
        \midrule
        prior                               & $\mathrm{Dir}(0.2)$       \\
        sampler                             & LogisticNormal            \\
        \bottomrule
    \end{tabular}
\end{subtable}%
\begin{subtable}[t]{.52\textwidth}
    \caption{\small Satellite}
    \centering
    \scriptsize
    \def\arraystretch{1.2}
    \setlength\tabcolsep{2pt}
    \begin{tabular}[t]{@{}c|l|r@{}}
        \toprule
        \multicolumn{3}{c}{\textbf{Convolutional Encoder}}              \\
        \midrule
        \multirow{3}{*}{\rotatebox{90}{VGG-F}}
        &   model                           & \VGGF                     \\
        &   pretrained                      & Yes                       \\
        &   fixed weights                   & Yes                       \\
        \midrule
        \multirow{6}{*}{\rotatebox{90}{MLP}}
        &   layer                           & Dense                     \\
        &   pretrained                      & No                        \\
        &   fixed weights                   & No                        \\
        &   units                           & 128                       \\
        &   dropout                         & 0.25                      \\
        &   activation                      & ReLU                      \\
        \midrule
        \multicolumn{2}{l}{\# trainable params}    &    0.5M            \\
        \bottomrule
    \end{tabular}%
    \hspace{1pt}%
    \setlength\tabcolsep{2pt}
    \begin{tabular}[t]{@{}l|r@{}}
        \toprule
        \multicolumn{2}{c}{\textbf{Contextual Explanations}}            \\
        \midrule
        model                               & Logistic reg.             \\
        features                            & Survey                    \\
        \# features                         & 64                        \\
        dictionary                          & 16                        \\
        $l_1$ penalty                       & $1 \cdot 10^{-3}$         \\
        $l_2$ penalty                       & $1 \cdot 10^{-4}$         \\
        \# params                           &                           \\
        \midrule[.75pt]
        \multicolumn{2}{c}{\textbf{Contextual VAE}}                     \\
        \midrule
        prior                               & $\mathrm{Dir}(0.2)$       \\
        sampler                             & LogisticNormal            \\
        \bottomrule
    \end{tabular}
\end{subtable}
\end{table}

\begin{table}[t]
\caption{\small%
Top-performing architectures used in our experiments on SUPPORT2 and PhysioNet.}
\label{tab:architectures-survival}
\smallskip
\begin{subtable}[t]{.52\textwidth}
    \caption{\small SUPPORT2}
    \centering
    \scriptsize
    \def\arraystretch{1.2}
    \setlength\tabcolsep{2pt}
    \begin{tabular}[t]{@{}c|l|r@{}}
        \toprule
        \multicolumn{3}{c}{\textbf{MLP Encoder}}                        \\
        \midrule
        \multirow{6}{*}{\rotatebox{90}{MLP}}
        &   layer                           & Dense                     \\
        &   pretrained                      & No                        \\
        &   fixed weights                   & No                        \\
        &   units                           & 64                        \\
        &   dropout                         & 0.50                      \\
        &   activation                      & ReLU                      \\
        \bottomrule
    \end{tabular}
    \hspace{2pt}
    \setlength\tabcolsep{2pt}
    \begin{tabular}[t]{@{}l|r@{}}
        \toprule
        \multicolumn{2}{c}{\textbf{Contextual Explanations}}            \\
        \midrule
        model                         & Linear CRF                      \\
        features                      & Measurements                    \\
        \# features                   & 50                              \\
        dictionary                    & 16                              \\
        $l_1$ penalty                 & $1 \cdot 10^{-3}$               \\
        $l_2$ penalty                 & $1 \cdot 10^{-4}$               \\
        \bottomrule
    \end{tabular}
\end{subtable}%
\begin{subtable}[t]{.52\textwidth}
    \caption{\small PhysioNet Challenge 2012}
    \centering
    \scriptsize
    \def\arraystretch{1.2}
    \setlength\tabcolsep{2pt}
    \begin{tabular}[t]{@{}c|l|r@{}}
        \toprule
        \multicolumn{3}{c}{\textbf{Sequential Encoder}}  \\
        \midrule
        \multirow{6}{*}{\rotatebox{90}{LSTM}}
        & layer                           & LSTM         \\
        & bidirectional                   & No           \\
        & units                           & 32           \\
        & max length                      & 150          \\
        & dropout                         & 0.25         \\
        & rec. dropout                    & 0.25         \\
        \bottomrule
    \end{tabular}
    \hspace{2pt}
    \setlength\tabcolsep{2pt}
    \begin{tabular}[t]{@{}l|r@{}}
        \toprule
        \multicolumn{2}{c}{\textbf{Contextual Explanations}}        \\
        \midrule
        model                         & Linear CRF                  \\
        features                      & Statistics                  \\
        \# features                   & 111                         \\
        dictionary                    & 16                          \\
        $l_1$ penalty                 & $1 \cdot 10^{-3}$           \\
        $l_2$ penalty                 & $1 \cdot 10^{-4}$           \\
        \bottomrule
    \end{tabular}
\end{subtable}
\end{table}

\clearpage
\bibliography{references}

\end{document}